\documentclass[11pt]{article}
\usepackage[numbers, compress, sort]{natbib}
\usepackage[margin=1.1in]{geometry}

\usepackage{float}
\usepackage[utf8]{inputenc} 
\usepackage[T1]{fontenc}    
\usepackage{lmodern}
\usepackage{hyperref}       
\usepackage{url}            
\usepackage{booktabs}       
\usepackage{amsfonts}       
\usepackage{nicefrac}       
\usepackage{microtype}      
\usepackage[normalem]{ulem}
\usepackage{thm-restate}
\newcommand{\restatehack}[1]{}   
\usepackage{array}
\newcolumntype{R}{>{$}r<{$}} 
\newcolumntype{C}{>{$}c<{$}}

\usepackage{mathtools, amsmath, amssymb, graphicx, verbatim, amsthm}
\usepackage{color}
\definecolor{darkblue}{rgb}{0.0,0.0,0.2}
\definecolor{darkgreen}{rgb}{0.0,0.3,0.0}
\hypersetup{colorlinks,breaklinks,
            linkcolor=darkblue,urlcolor=darkblue,
            anchorcolor=darkblue,citecolor=darkblue}
\usepackage{wrapfig}
\usepackage[font=small]{caption}
\usepackage{subcaption}
\usepackage[colorinlistoftodos,textsize=tiny]{todonotes} 
\usepackage{bbm}
\usepackage{xspace}

\usepackage{tikz,pgfplots,tikz-3dplot}
\usetikzlibrary{calc}
\usetikzlibrary{arrows.meta,shapes.misc,positioning,fit,calc,intersections,through,decorations.markings}
\usetikzlibrary{backgrounds}
\tdplotsetmaincoords{55}{135}        

\tikzstyle{cell}=[dashed,thick]
\tikzstyle{simplex}=[thick]
\newcommand{\tikzfigscale}{2.5}

\usetikzlibrary{arrows}
\usetikzlibrary{shapes}

\usetikzlibrary{calc}
\newcommand{\Comments}{1}
\newcommand{\mynote}[2]{\ifnum\Comments=1\textcolor{#1}{#2}\fi}
\newcommand{\mytodo}[2]{\ifnum\Comments=1
  \todo[linecolor=#1!80!black,backgroundcolor=#1,bordercolor=#1!80!black]{#2}\fi}

\ifnum\Comments=1               
\fi

\newcommand{\reals}{\mathbb{R}}

\newcommand{\dom}{\mathrm{dom}}

\newcommand{\prop}[1]{\mathrm{prop}[#1]}

\newcommand{\affhull}{\mathrm{affhull}}

\newcommand{\cell}{\mathrm{cell}}

\newcommand{\BEP}{L_{\text{\tiny BEP}}}
\newcommand{\LWW}{L_{\text{\tiny WW}}}
\newcommand{\ellabstain}{\ell_{\text{abs}}}
\newcommand{\ellOP}{\ell_{\text{\tiny OP}}}

\newcommand{\mode}{\mathrm{mode}}

\newcommand{\simplex}{\Delta_\Y}

\newcommand{\Li}[1]{L^{(#1)}}

\newcommand{\D}{\mathcal{D}}
\newcommand{\E}{\mathbb{E}}
\newcommand{\F}{\mathcal{F}}
\renewcommand{\H}{\mathcal{H}}

\renewcommand{\P}{\mathcal{P}}
\newcommand{\R}{\mathcal{R}}
\newcommand{\Sc}{\mathcal{S}}
\newcommand{\T}{\mathcal{T}}
\newcommand{\U}{\mathcal{U}}
\newcommand{\V}{\mathcal{V}}
\newcommand{\X}{\mathcal{X}}
\newcommand{\Y}{\mathcal{Y}}

\newcommand{\risk}[1]{\underline{#1}}
\newcommand{\inprod}[2]{\langle #1, #2 \rangle}
\newcommand{\relint}{\mathrm{relint}}
\newcommand{\inter}{\mathrm{inter}}

\newcommand{\toto}{\rightrightarrows}

\newcommand{\trim}{\mathrm{trim}}
\newcommand{\red}{\mathrm{red}}
\newcommand{\trimred}{\mathrm{trim}}
\newcommand{\trimcover}{\mathrm{trim}}

\newcommand{\hyp}{\mathrm{hypo}}

\newcommand{\ones}{\mathbbm{1}}
\DeclarePairedDelimiter\ceil{\lceil}{\rceil}

\newcommand{\regret}[3]{R_{#1}(#2,#3)}

\newcommand{\Ind}[1]{\ones\{#1\}}

\newcommand{\hinge}{L_{\mathrm{hinge}}}
\newcommand{\ellzo}{\ell_{\text{0-1}}}
\newcommand{\ellabs}{\ell_{\text{abs}}^f}
\newcommand{\elltopk}{\ell^{\text{top-$k$}}}

\DeclareMathOperator*{\argmax}{arg\,max}
\DeclareMathOperator*{\argmin}{arg\,min}
\DeclareMathOperator*{\sgn}{sgn}

\newtheorem*{theorem*}{Theorem}
\newtheorem{theorem}{Theorem}
\newtheorem{lemma}{Lemma}

\newtheorem{proposition}{Proposition}
\newtheorem{corollary}{Corollary}

\newtheorem{claim}{Claim}

\newtheorem{definition}{Definition}
\newtheorem{construction}{Construction}

\newtheorem{assumption}{Assumption}

\title{An Embedding Framework for the Design and Analysis of Consistent Polyhedral Surrogates}
\author{
 Jessie Finocchiaro \\
 \texttt{jefi8453@colorado.edu}
 \and
 Rafael Frongillo\\
 \texttt{raf@colorado.edu}
 \and
 Bo Waggoner\\
 \texttt{bwag@colorado.edu}
}
\date{
  University of Colorado Boulder
  \\[15pt]
  \today
}
\begin{document}

\maketitle

\begin{abstract}
  We formalize and study the natural approach of designing convex surrogate loss functions via embeddings, for problems such as classification, ranking, or structured prediction. 
  In this approach, one embeds each of the finitely many predictions (e.g.\ rankings) as a point in $\reals^d$, assigns the original loss values to these points, and ``convexifies'' the loss in some way to obtain a surrogate.
  We establish a strong connection between this approach and polyhedral (piecewise-linear convex) surrogate losses:
  every discrete loss is embedded by some polyhedral loss, and every polyhedral loss embeds some discrete loss.
  Moreover, an embedding gives rise to a consistent link function as well as linear surrogate regret bounds.
  Our results are constructive, as we illustrate with several examples.
  In particular, our framework gives succinct proofs of consistency or inconsistency for various polyhedral surrogates in the literature, and for inconsistent surrogates, it further reveals the discrete losses for which these surrogates are consistent.
  We go on to show additional structure of embeddings, such as the equivalence of embedding and matching Bayes risks, and the equivalence of various notions of non-redudancy.
  Using these results, we establish that indirect elicitation, a necessary condition for consistency, is also sufficient when working with polyhedral surrogates.
\end{abstract}

\section{Introduction}\label{sec:intro}

In supervised learning, one tries to learn a hypothesis which fits labeled data as judged by a target loss function.
Minimizing the target loss directly is typically computationally intractable for discrete prediction tasks like classification, ranking, and structured prediction.
Instead, one typically minimizes a surrogate loss which is convex and therefore efficiently minimized.
After learning a surrogate hypothesis, a link function then translates back to the target problem.
To ensure the surrogate and link properly correspond to the target problem, the surrogate must be \emph{consistent}, meaning that minimizing the surrogate loss over enough data will also minimize the target loss.

While a growing body of work seeks to design and analyze consistent convex surrogates for particular target loss functions, to date much of this work has been ad-hoc.
We lack general tools to systematically construct consistent convex surrogates, much less an understanding of the full class of consistent surrogates.
For example, in multiclass and top-$k$ classification, several proposed surrogates were proposed and adopted but later proved to be inconsistent~\cite{yang2018consistency,crammer2001algorithmic,rifkin2004defense}.
This state of affairs is even more dire for structured prediction, where in addition to convexity and consistency, one often requires a low \emph{prediction dimension} (the dimension of the surrogate prediction space) as the label set can grow exponentially large.
Clever constructions like the binary-encoded predictions (BEP) surrogate for multiclass classification with an abstain option~\cite{ramaswamy2018consistent}, which achieves logarithmic prediction dimension, are the exception rather than the rule.
In all of these settings, we lack a unifying framework that moves from a given target problem to a convex consistent surrogate and link function.

To address this shortcoming, we introduce a new framework motivated by a particularly natural approach for finding convex surrogates, wherein one ``embeds'' a discrete loss.
Specifically, we say a convex surrogate $L$ embeds a discrete target loss $\ell$ if there is an injective embedding from the discrete reports (predictions) to $\reals^d$ such that (i) the original loss values are recovered, and (ii) a report is $\ell$-optimal if and only if the embedded report is $L$-optimal.
(See \S~\ref{sec:embedding}.)
Common examples of this general construction include hinge loss as a surrogate for 0-1 loss and the BEP surrogate mentioned above~\citep{ramaswamy2018consistent}.

Using this framework, we give several constructive results to design new consistent surrogates, as well as a suite of tools to analyze existing surrogates.
As a first step, in \S~\ref{sec:poly-loss-embed}, we show that such an embedding scheme is intimately related to the class of \emph{polyhedral} loss functions, i.e., those that are piecewise-linear and convex.
\begin{restatable}{theorem}{embedpolyinformal}\label{thm:embed-poly-main}
  Every discrete loss $\ell$ is embedded by some polyhedral loss $L$, and every polyhedral loss $L$ embeds some discrete loss $\ell$.
\end{restatable}
\noindent
Crucially, we go on in \S~\ref{sec:calibration} to show that an embedding gives rise to a calibrated link function, and is therefore consistent with respect to the target loss.
\begin{restatable}{theorem}{linkinformal}\label{thm:link-main}
  Given any polyhedral loss $L$, let $\ell$ be a discrete loss it embeds. There exists a link function $\psi$ such that $(L,\psi)$ is calibrated with respect to $\ell$.
\end{restatable}
\noindent
Beyond consistency, we show in \S~\ref{subsec:regret-bounds} that any polyhedral surrogate achieves a linear surrogate regret bound, which allows one to translate generalization bounds from the surrogate to the target.
Our results are constructive: given a discrete target loss, we show how to define a surrogate and construct a calibrated link, and given a polyhedral surrogate, we show how to find a discrete loss that it embeds.

We demonstrate the constructiveness of our framework in \S~\ref{sec:applications} with several applications, many of which are subsequent to our work.
In addition to constructing new surrogates, we illustrate the power of our framework to analyze previously proposed polyhedral surrogates.
For example, while we know that the inconsistent top-$k$ polyhedral surrogates mentioned above do not solve top-$k$ classification, our framework illuminates the problem they \emph{do} solve, and what restrictions on the underlying distribution would render them top-$k$ consistent (\S~\ref{sec:top-k}).

Underpinning our results are several observations, outlined in \S~\ref{sec:min-rep-sets}, which formalize the idea that polyhedral losses ``behave like'' discrete losses.
In particular, any polyhedral loss $L$ has a finite \emph{representative set} $\Sc$ of reports, such that for all distributions, some report in $\Sc$ is $L$-optimal.
We show that $L$ embeds the discrete loss $\ell = L|_\Sc$ given by restricting $L$ to just the reports in $\Sc$.
To go from a discrete loss to a polyhedral surrogate, we prove that the conditions of an embedding are equivalent to matching Bayes risks (Proposition~\ref{prop:embed-bayes-risks}), and use the fact that discrete losses and polyhedral losses both have polyhedral Bayes risks.

Finally, we also provide several observations beyond what is needed to prove our main results, which we view as conceptual contributions (\S~\ref{sec:min-rep-sets},~\ref{sec:poly-ie-consistency}).
Using tools from property elicitation, we show an equivalence between minimum representative sets (those of minimum cardinality) and ``non-redundancy'', wherein no report is dominated by another.
We further show that, while a minimum representative set is not always unique, the loss values associated with it are unique, giving rise to a natural ``trim'' operation on losses.
The paper concludes with an interesting observation (Theorem~\ref{thm:poly-ie-implies-consistent}): while indirect property elicitation is generally a strictly weaker condition than consistency, the two are equivalent when restricting to the class of polyhedral surrogates.

Taken together, we view our contributions as both conceptual and practical.
We uncover the remarkable structure of polyhedral surrogates, deepening our understanding of the relationship between surrogate and discrete target losses.
This structure leads to a powerful new framework to design and analyze surrogate losses.
As we illustrate with several examples, this framework has already been applied to solve open questions by designing new surrogates, to uncover the behavior of existing surrogates, and to construct link functions in complex structured problems.
We conclude with several directions for future work.

\section{Setting}
\label{sec:setting}

For discrete prediction problems like classification, the given (discrete) loss is often computationally intractable to optimize directly.
Therefore, many machine learning algorithms instead minimize a surrogate loss function with better optimization qualities, such as convexity.
To ensure that this surrogate loss successfully addresses the original target problem, one needs to establish statistical consistency, a minimal requirement that is a prerequisite for generalization bounds.
Consistency roughly says that, in the limit as one obtains more and more data, the learned hypothesis approaches the best possible.
Consistency depends crucially on the choice of link function that maps surrogate reports (predictions) to target reports.

In this section, we introduce notation and concepts related to consistency that we use throughout.
Consistency is often a difficult condition to work with directly, but in finite prediction settings, it is equivalent the simpler notion of \emph{calibration} (Definition~\ref{def:calibrated}) which depends solely on the conditional distribution over the labels~\citep{bartlett2006convexity,tewari2007consistency,ramaswamy2016convex}.
Even simpler than calibration indirect elicitation, a weaker condition only requiring that optimal surrogate reports link to optimal target reports.
Finally, we introduce a new precise notion of embedding, a special case of indirect elicitation, which forms the backbone of our approach.

\subsection{Notation and Losses}
\label{sec:notation-losses}

Let $\Y$ be a finite label space, and throughout let $n=|\Y|$.
Define $\reals^\Y_+$ to be the nonnegative orthant in $\reals^\Y$, i.e., $\reals^\Y_+ = \{x \in \reals^\Y : \forall y\in\Y\; x_y \geq 0 \}$.
Let $\simplex = \{p\in\reals^{\Y}_+ : \|p\|_1 = 1\}$ be the set of probability distributions on $\Y$, represented as vectors.
We will primarily focus on conditional distributions $p\in\simplex$ over labels, abstracting away the feature space $\X$; see \S~\ref{subsec:calibration-links} for a discussion of the joint distribution over $\X\times\Y$.

A generic loss function, denoted $L:\R\to\reals^\Y_+$, maps a report (prediction) $r$ from a set $\R$ to the vector of loss values $L(r) = (L(r)_y)_{y\in\Y}$ for each possible outcome $y\in\Y$.
We write the corresponding expected loss when $Y \sim p$ as $\inprod{p}{L(r)}$.
The \emph{Bayes risk} of a loss $L:\R\to\reals^\Y_+$ is the function $\risk{L}:\simplex\to\reals_+$ given by $\risk{L}(p) := \inf_{r\in\R} \inprod{p}{L(r)}$.
When restricting the domain of a loss $L$ from $\R$ to $\R' \subseteq \R$, we write $L|_{\R'}$.

We assume that the target prediction problem is given in the form of a \emph{target loss} $\ell:\R\to\reals^\Y_+$ for some report set $\R$.
A \emph{discrete (target) loss} is such an $\ell$ where $\R$ is a finite set.
Surrogate losses will take $\R = \reals^d$ and be written $L:\reals^d\to\reals^\Y_+$, typically with reports written $u\in\reals^d$.

For example, 0-1 loss is a discrete loss with $\R = \Y = \{-1,1\}$
given by $\ellzo(r)_y = \Ind{r \neq y}$, with Bayes risk $\risk{\ellzo}(p) = 1-\max_{y\in\Y} p_y$.
Two widely-used surrogates for $\ellzo$ are hinge loss $\hinge(u)_y = (1-yu)_+$, where $(x)_+ = \max(x,0)$, and logistic loss $L(u)_y = \log(1+\exp(-yu))$ for $u\in\reals$.
See Figure~\ref{fig:bayes-risks-01} for a visualization of the Bayes risks of 0-1, hinge, and logistic losses, respectively.

Many of the surrogate losses we consider will be \emph{polyhedral}, meaning piecewise linear and convex; we briefly recall the relevant definitions.
In $\reals^d$, a \emph{polyhedral set} or \emph{polyhedron} is the intersection of a finite number of closed halfspaces.
A \emph{polytope} is a bounded polyhedral set.
A convex function $f:\reals^d\to\reals$ is \emph{polyhedral} if its epigraph is polyhedral, or equivalently, if it can be written as a pointwise maximum of a finite set of affine functions~\citep{rockafellar1997convex}.
\begin{definition}[Polyhedral loss]
  A loss $L: \reals^d \to \reals^{\Y}_+$ is \emph{polyhedral} if $L(u)_y$ is a polyhedral function of $u$ for each $y\in\Y$.
\end{definition}
\noindent
In the example above, hinge loss is polyhedral, whereas logistic loss is not.

\subsection{Property Elicitation}
\label{sec:property-elicitation}

We will frequently appeal to concepts and results from property elicitation~\citep{savage1971elicitation,osband1985information-eliciting,lambert2008eliciting,gneiting2011making,steinwart2014elicitation,frongillo2015vector-valued,lambert2018elicitation}.
Here one studies \emph{properties}, maps from label distributions to reports, and asks when a property characterizes the reports that exactly minimize a loss.
In our case, this map will at times be set-valued, meaning a single distribution could yield multiple optimal reports.
For example, when $p=(1/2,1/2)$, both $r=1$ and $r=-1$ optimize 0-1 loss.
We will use double arrow notation to denote a (non-empty) set-valued map, so that $\Gamma: \simplex \toto \R$ is shorthand for $\Gamma: \simplex \to 2^{\R} \setminus \{\emptyset\}$, where $2^\R$ denotes the power set of $\R$.

\begin{definition}[Property, level set]\label{def:property}
  A \emph{property} is a function $\Gamma:\simplex\toto\R$.
  The \emph{level set} of $\Gamma$ for report $r$ is the set $\Gamma_r := \{p \in \simplex \mid r \in \Gamma(p)\}$.
  If $\R$ is finite, we call $\Gamma$ a \emph{finite property}.
\end{definition}

Intuitively, $\Gamma(p)$ is the set of reports which should optimize expected loss for a given distribution $p$, and $\Gamma_r$ is the set of distributions for which the report $r$ should be optimal.
For example, the \emph{mode} is the 
property $\mode(p) = \argmax_{y\in\Y} p_y$, and captures the set of optimal reports for 0-1 loss: for each distribution over the labels, one should report the most likely label.
In this case we say 0-1 loss (directly) \emph{elicits} the mode, as we formalize below.

\begin{definition}[Directly Elicits]
  \label{def:elicits}
  A loss $L:\R\to\reals^\Y_+$, \emph{(directly) elicits} a property $\Gamma:\simplex \toto \R$ if
  \begin{equation}
    \forall p\in\simplex,\;\;\;\Gamma(p) = \argmin_{r \in \R} \inprod{p}{L(r)}~.
  \end{equation}
  If $L$ elicits a property, that property is unique and we denote it $\prop{L}$.
\end{definition}
Since we have defined a property $\Gamma$ to be nonempty, if the infimum of expected loss $\inprod{p}{L(\cdot)}$ is not attained for some $p \in \simplex$, then $L$ does not elicit a property.
We say that a loss $L$ is \emph{minimizable} if the infimum of $\inprod{p}{L(\cdot)}$ is attained for all $p \in \simplex$.
For example, hinge loss is minimizable, whereas logistic loss is not (take $p=(0,1)$ or $(1,0)$).

We will typically denote general properties and losses with $\Gamma$ and $L$, respectively.
For surrogate losses and properties, recall that we typically consider the report set $\reals^d$.
For discrete target losses and properties, we will take $\R$ to be a finite set, and use lowercase notation $\gamma$ and $\ell$, respectively.

\subsection{Calibration and Links}
\label{subsec:calibration-links}

To assess whether a surrogate and link function align with the original loss, we turn to the common condition of \emph{calibration}.
Roughly, a surrogate and link are calibrated if the best possible expected loss achieved by linking to an incorrect report is strictly suboptimal.

\begin{definition}
  \label{def:calibrated}
  Let discrete loss $\ell:\R\to\reals^\Y_+$, proposed surrogate $L:\reals^d\to\reals^\Y_+$, and link function $\psi:\reals^d\to\R$ be given.
  We say $(L,\psi)$ is \emph{calibrated} with respect to $\ell$ if
for all $p \in \simplex$,
  \begin{equation}
    \label{eq:calibrated}
  \inf_{u \in \reals^d : \psi(u) \not\in \prop{\ell}(p)} \inprod{p}{L(u)} > \inf_{u \in \reals^d} \inprod{p}{L(u)}~.
  \end{equation}
  If $(L, \psi)$ is calibrated with respect to $\ell$, we call $\psi$ a \emph{calibrated link.}
\end{definition}

It is well-known in finite-outcome settings that calibration is equivalent to \emph{consistency}, in the following sense (cf.~\citep{bartlett2006convexity,zhang2004statistical,agarwal2015consistent}).
Suppose we have the feature space $\X$ and label space $\Y$.
We say a surrogate and link pair $(L,\psi)$ is consistent with respect to $\ell$ if, for all data distributions $D \in \Delta(\X \times \Y)$, and all sequences of surrogate hypotheses $H_1,H_2,\ldots$ whose $L$-loss limits to the optimal surrogate loss $L^*$ (in expectation over $D$), the $\ell$-loss of the sequence $\psi\circ H_1,\psi \circ H_2, \ldots$ limits to the optimal target loss $\ell^*$.
As Definition~\ref{def:calibrated} does not involve the feature space $\X$, we will drop it for the remainder of the paper.

Like the use of a surrogate and link pair in the calibration definition, one can also extend the earlier definition of property elicitation to \emph{indirect (property) elicitation}, in which one applies a link to an elicited property to obtain a related property of interest.
\begin{definition}\label{def:indirect-elic}
	Let minimizable loss $L : \reals^d \to \reals^\Y_+$ and link $\psi : \reals^d \to \R$ be given.
  The pair $(L, \psi)$ indirectly elicits a property $\gamma : \simplex \toto \R$ if for all $u \in \reals^d$, we have $\Gamma_u \subseteq \gamma_{\psi(u)}$, where $\Gamma = \prop{L}$.
	Moreover, we say $L$ indirectly elicits $\gamma$ if such a $\psi$ exists, i.e., if for all $u\in\reals^d$ there exists $r\in\R$ such that $\Gamma_u \subseteq \gamma_r$.
  \footnote{The elicitation literature often refers to this latter condition as one property ``refining'' another~\citep{frongillo2015elicitation}.}
\end{definition}

Indirect elicitation is a weaker condition than calibration~\citep{steinwart2008support,agarwal2015consistent,finocchiaro2021unifying}; we briefly sketch the argument.
Suppose $L$ is minimizable and $(L,\psi)$ is calibrated with respect to $\ell$, and set $\Gamma=\prop{L}$ and $\gamma=\prop{\ell}$.
Let $u\in\reals^d$ and $p\in\Gamma_u$.
From eq.~\eqref{eq:calibrated}, if $\psi(u) \notin \gamma(p)$, then we would have $u\notin\Gamma(p)$, a contradiction to $p\in\Gamma_u$.
Thus, $\psi(u) \in \gamma(p)$, so $p\in\gamma_{\psi(u)}$.
In fact, indirect elicitation is strictly weaker; take hinge loss with the link $\psi(u) = -1$ for $u < 1$ and $\psi(u) = 1$ for $u\geq 1$.
\citet{agarwal2015consistent} were the first to formally connect property elicitation to calibration, though their results generally do not apply to discrete prediction tasks.

\subsection{Embedding}
\label{sec:embedding}

We now formalize the sense in which a convex surrogate can \emph{embed} a target loss $\ell$.
Here one maps each report (prediction) of $\ell$ to a point in $\reals^d$, then constructs a convex loss on $\reals^d$ that agrees with $\ell$ at these points.
This approach captures several consistent surrogates in the literature (e.g.,~\citep{ramaswamy2015hierarchical,ramaswamy2016convex,lapin2015top,wang2020weston}; see \S~\ref{sec:applications}).

An important subtlety is that it is not always necessary to map \emph{all} target reports to $\reals^d$.
It is often convenient to allow $\ell$ to have reports that are ``redundant'' in some sense. (We explore redundancy further in \S~\ref{sec:min-rep-sets}; see also \citet{wang2020weston}.)
Because of this redundancy, we will only require an embedding map to be defined on a \emph{representative set}: a set of reports $\Sc$ such that, for all label distributions, at least one report $r\in\Sc$ minimizes expected loss.
\begin{definition}[Representative set]
  Let $\Gamma:\simplex\toto\R$.
  We say $\Sc \subseteq \R$ is \emph{representative for $\Gamma$} if we have $\Gamma(p) \cap \Sc \neq \emptyset$ for all $p\in \simplex$.
  We further say $\Sc$ is a \emph{minimum representative set} if it has the smallest cardinality among all representative sets.
  Given a minimizable loss $L:\R\to\reals^\Y_+$, we say $\Sc$ is a (minimum) representative set for $L$ if it is a (minimum) representative set for $\prop L$.
\end{definition}

\noindent
\citet{wang2020weston} first study the notion of minimum representative sets under the name \emph{embedding cardinality}.

We now define an embedding, which is a special case of indirect property elicitation.
(This fact is non-trivial; see Lemma~\ref{lem:embedding-implies-indirect-elic}.)
In addition to matching loss values, as described above, we require the original reports to be $\ell$-optimal exactly when the corresponding embedded points are $L$-optimal.
As we discuss following Proposition~\ref{prop:representative-embeds-restriction}, this latter condition can be more simply stated: the representative set for the target must embed into a representative set for the surrogate.
\begin{definition}[Embedding]\label{def:loss-embed}
  A minimizable loss $L:\reals^d\to\reals^\Y_+$ \emph{embeds} a loss $\ell:\R\to\reals^\Y_+$ if there exists a representative set $\Sc$ for $\ell$ and an injective embedding $\varphi:\Sc\to\reals^d$ such that
  (i) for all $r\in\Sc$ we have $L(\varphi(r)) = \ell(r)$, and (ii) for all $p\in\simplex,r\in\Sc$ we have
  \begin{equation}\label{eq:embed-loss}
    r \in \prop{\ell}(p) \iff \varphi(r) \in \prop{L}(p)~.
  \end{equation}
  If $\Sc$ is a minimal representative set, we say $L$ \emph{tightly embeds} $\ell$.
\end{definition}

To illustrate the idea of embedding, let us closely examine hinge loss as a surrogate for 0-1 loss in binary classification.
Recall that we have $\R = \Y = \{-1, +1\}$, with $\hinge(u)_y = (1 - uy)_+$ and $\ellzo(r)_y := \Ind{r\neq y}$, typically with link function $\psi(u) = \sgn(u)$.
We will see that hinge loss embeds (2 times) 0-1 loss, via the identity embedding $\varphi(r) = r$.
For condition (i), it is straightforward to check that $\hinge(\varphi(r))_y = \hinge(r)_y = 2 \Ind{r \neq y} = 2\ellzo(r)_y$ for all $r,y\in\{-1,1\}$.
For condition (ii), let us compute the property each loss elicits, i.e., the set of optimal reports for each $p\in\simplex$:
\[
\prop{\ellzo}(p) = \begin{cases}
1 & p_1 > 1/2 \\
\{-1,1\} & p_1 = 1/2\\
-1 & p_1 < 1/2
\end{cases}
\qquad
\prop{L_{hinge}}(p) = \begin{cases}
[1,\infty) & p_1 = 1\\
1 & p_1 \in (1/2,1) \\
[-1,1] & p_1 = 1/2\\
-1& p_1 \in (0, 1/2)\\
(-\infty, -1]& p_1 = 0
\end{cases}~.
\]
In particular, we see that $-1 \in \prop{\ellzo}(p) \iff p_1 \in [0, 1/2] \iff -1 \in \prop{\hinge}(p)$, and $1 \in \prop{\ellzo}(p) \iff p_1 \in [1/2,1] \iff 1 \in \prop{\hinge}(p)$.
With both conditions of Definition~\ref{def:loss-embed} satisfied, we can conclude that $\hinge$ embeds $2\ellzo$ with the representative set $\Sc = \{-1,1\}$.
By results in \S~\ref{subsec:match-BR}, one could also show that $\hinge$ embeds $2\ellzo$ by the fact that their Bayes risks match (Figure~\ref{fig:bayes-risks-01}).

\begin{figure}
	\begin{minipage}{0.3\linewidth}
	\centering
	\includegraphics[width=0.95\linewidth]{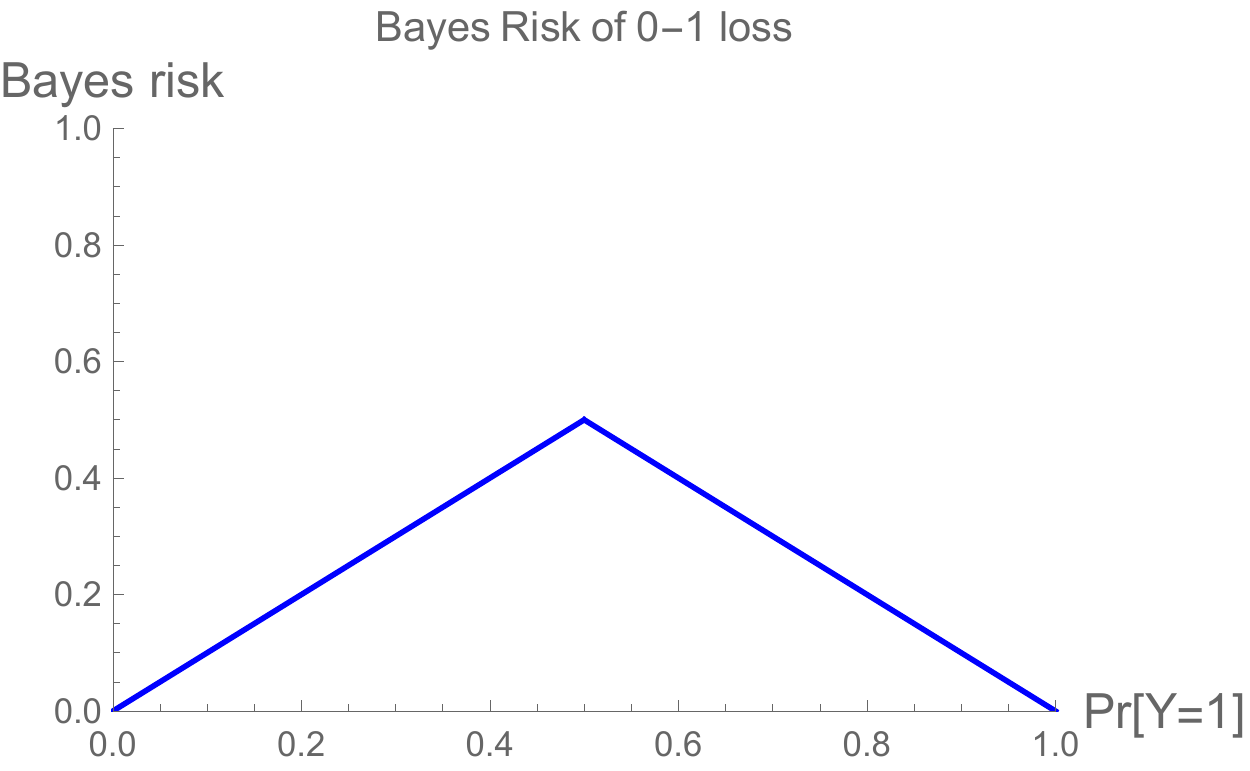}
	\end{minipage}
	\hfill
	\begin{minipage}{0.3\linewidth}
	\centering		\includegraphics[width=0.95\linewidth]{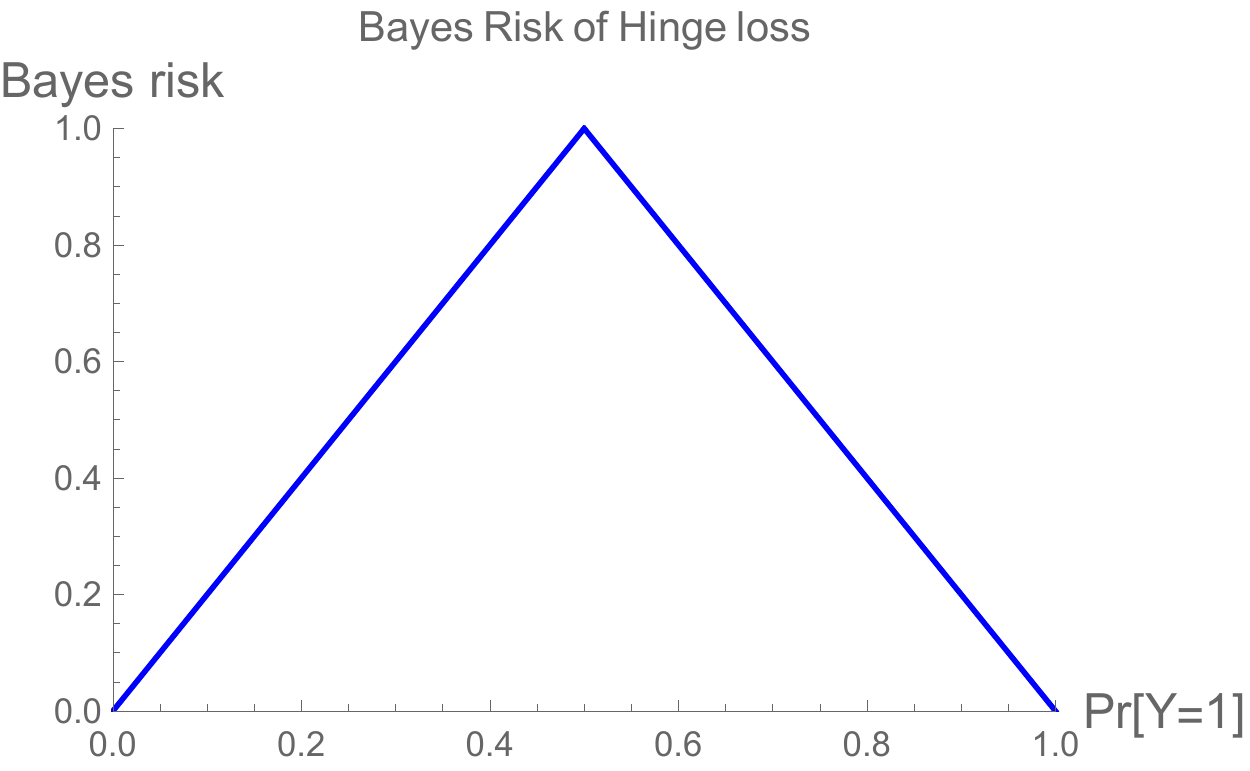}
	\end{minipage}
	\hfill
	\begin{minipage}{0.3\linewidth}
	\centering
	\includegraphics[width=0.95\linewidth]{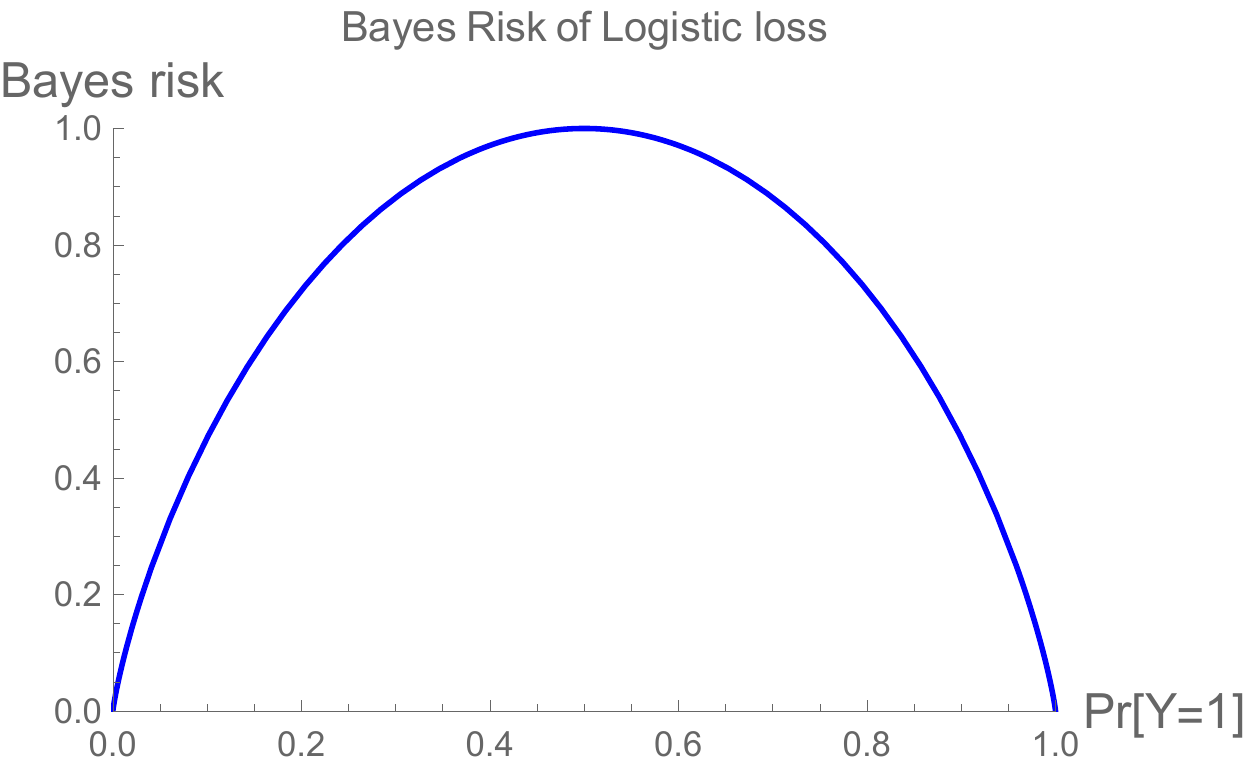}
\end{minipage}
\caption{Bayes risks $\risk L : p \mapsto \inf_u \inprod{p}{L(u)}$ of 0-1, hinge, and logistic losses, respectively, plotted as a function of $p_1 = \Pr[Y=1]$.
	Observe that the Bayes risks of 0-1 and hinge loss are both piecewise linear and concave, while the Bayes risk of logistic loss is not piecewise linear.  Proposition~\ref{prop:embed-bayes-risks} states that embedding is equivalent to matching Bayes risks, confirming that hinge loss (M) embeds 0-1 loss (L), but logistic loss (R) does not.}
\label{fig:bayes-risks-01}
\end{figure}

In this particular example, it is known $(\hinge,\sgn)$ is calibrated with respect to 0-1 loss.
Beyond this simple case, however, it is not clear whether an embedding will always yield a calibrated link.
Indeed, while it is intuitively clear that embedded points should link back to their original reports, via $\psi(\varphi(r)) = r$, how to map the remaining values is far from obvious.
Using the strong connection between embeddings and polyhedral surrogates in \S~\ref{sec:poly-loss-embed}, we give a construction to map the remaining values in \S~\ref{sec:calibration}, showing that embeddings from polyhedral surrogates always yield calibration.

While our notion of embedding is sufficient for calibration (and therefore consistency), it is worth noting that an embedding is not \emph{necessary} for these conditions.  
For example, while logistic loss does not embed 0-1 loss, logistic loss and the sign link are still consistent for 0-1 loss.
When working with polyhedral surrogates, however, embeddings are necessary for calibration in a strong sense, as we discuss in \S~\ref{sec:poly-ie-consistency}: if a polyhedral surrogate $L$ has a calibrated link to some target $\ell$, then $L$ must embed some discrete target $\hat\ell$ which can then be linked to $\ell$.

\section{Embeddings and Polyhedral Losses}
\label{sec:poly-loss-embed}

In this section, we establish a tight relationship between the technique of embedding and the use of polyhedral (piecewise-linear convex) surrogate losses, culminating in Theorem~\ref{thm:embed-poly-main}.
We defer the question of when such surrogates are consistent to \S~\ref{sec:calibration}. 

A first observation is that if a loss $L$ elicits a property $\Gamma$, then $L$ restricted to some representative set $\Sc$, denoted $L|_\Sc$, elicits $\Gamma$ restricted to $\Sc$.
As a consequence, restricting to representative sets preserves the Bayes risk.
We will use these observations throughout.
\begin{lemma}\label{lem:loss-restrict}
  Let $L:\R\to\reals^\Y_+$ elicit $\Gamma$, and let $\Sc\subseteq\R$ be representative for $L$.
  Then $L|_\Sc$ elicits $\gamma:\simplex\toto\Sc$ defined by $\gamma(p) = \Gamma(p)\cap \Sc$.
  Moreover, $\risk{L}=\risk{L|_\Sc}$.
\end{lemma}
\begin{proof}
  Let $p\in\simplex$ be fixed throughout.
  First let $r \in \gamma(p) = \Gamma(p) \cap \Sc$.
  Then $r \in \Gamma(p) = \argmin_{u\in\R} \inprod{p}{L(u)}$, so as $r\in\Sc$ we have in particular $r \in \argmin_{u\in\Sc} \inprod{p}{L(u)}$.
  For the other direction, suppose $r \in \argmin_{u\in\Sc} \inprod{p}{L(u)}$.
  As $\Sc$ is representative for $L$, we must have some $s \in \Gamma(p) \cap \Sc$.
  On the one hand, $s\in\Gamma(p) = \argmin_{u\in\R} \inprod{p}{L(u)}$.
  On the other, as $s \in \Sc$, we certainly have $s \in \argmin_{u\in\Sc} \inprod{p}{L(u)}$.
  But now we must have $\inprod{p}{L(r)} = \inprod{p}{L(s)}$, and thus $r \in \argmin_{u\in\R} \inprod{p}{L(u)} = \Gamma(p)$ as well.
  We now see $r \in \Gamma(p) \cap \Sc$.
  Finally, the equality of the Bayes risks $\min_{u\in\R} \inprod{p}{L(u)} = \min_{u\in\Sc} \inprod{p}{L(u)}$ follows immediately by the above, as $\emptyset \neq \Gamma(p)\cap\Sc \subseteq \Gamma(p)$ for all $p\in\simplex$.
\end{proof}

Lemma~\ref{lem:loss-restrict} leads to the following useful tool for finding embeddings: if a surrogate has a finite representative set, it embeds its restriction to the representative set.
\begin{proposition}\label{prop:representative-embeds-restriction}
  Let a minimizable surrogate loss $L:\reals^d \to \reals^\Y_+$ be given.
  If $L$ has a finite representative set $\Sc \subseteq \reals^d$, then $L$ embeds the discrete loss $L|_\Sc$.
\end{proposition}
\begin{proof}
  Let $\Gamma = \prop{L}$ and $\gamma = \prop{L|_\Sc}$.
  Define $\varphi : \Sc \to \Sc$ to be the identity embedding.
  Condition (i) of an embedding is trivially satisfied, as $L|_\Sc(u) = L(u)$ for all $u\in\Sc$.
  Now let $u\in\Sc$.
  From Lemma~\ref{lem:loss-restrict}, for all $p\in\simplex$ we have $u \in \gamma(p) \iff u \in \Gamma(p) \cap \Sc \iff u \in \Gamma(p)$.
  We conclude condition (ii) of an embedding.
\end{proof}

Proposition~\ref{prop:representative-embeds-restriction} reveals an equivalent definition of an embedding which can be more convenient.
Given a representative set $\Sc$ for $\ell$, an injection $\varphi:\Sc\to\reals^d$ is an embedding if: (i) the loss values match as in Definition~\ref{def:loss-embed}(i), and (ii) $\varphi(\Sc)$ is representative for $L$.
This new definition is clearly implied by Definition~\ref{def:loss-embed}; for the converse, Proposition~\ref{prop:representative-embeds-restriction} states that $L$ embeds $L|_{\varphi(\Sc)}$, which by (i) is the same loss as $\ell$ up to relabeling via $\varphi$.

With Proposition~\ref{prop:representative-embeds-restriction} in hand, we now shift our focus to \emph{polyhedral} (piecewise-linear and convex) surrogates.
While polyhedral surrogates do not directly elicit finite properties, as their report sets are infinite, they do elicit properties with a finite range, meaning the set of possible optimal report sets is finite.

\begin{restatable}{lemma}{polyhedralrangegamma}
	\label{lem:polyhedral-range-gamma}
	Let $L:\reals^d\to\reals_+^\Y$ be a polyhedral loss.
  Then $L$ is minimizable and elicits a property $\Gamma := \prop{L}$.
	Moreover, the range of $\Gamma$, given by $\Gamma(\simplex) = \{\Gamma(p) \subseteq \reals^d : p\in\simplex\}$, is a finite set of closed polyhedra.
\end{restatable}
\begin{proof}[Proof sketch]
	See \S~\ref{app:power-diagrams} for the full proof.
	As $L$ is bounded from below, $L$ is minimizable from~\citet[Corollary 19.3.1]{rockafellar1997convex} .
	With $\Y$ finite, there are only finitely many supporting sets over $\simplex$.
	For $p \in \simplex$, the power diagram induced by projecting the epigraph of expected loss onto $\reals^d$ is the same for any $p$ of the same support (Lemma~\ref{lem:polyhedral-pd-same}).
	Moreover, we have $\Gamma(p)$ being exactly one of the faces of the projected epigraph since the hyperplane $u \mapsto (u, \inprod{p}{L(u)})$ supports the epigraph of the expected loss at exactly the property value; moreover, since the loss is polyhedral the supporting hyperplane must support on a face of the epigraph.
	Since this epigraph has finitely many faces (as it is polyhedral), the range of $\Gamma$ is then (a subset) of elements of a finitely generated (finite supports) set of finite elements (finite faces).
	Moreover, each element of $\Gamma(\simplex)$ is a closed polyhedron since it corresponds exactly to a closed face of a polyhedral set.
\end{proof}

From Lemma~\ref{lem:polyhedral-range-gamma}, one can simply select a point from each of the finitely many optimal sets to obtain a finite representative set.
Plugging this finite representative set into Proposition~\ref{prop:representative-embeds-restriction} then yields an embedding.

\begin{theorem}\label{thm:poly-embeds-discrete}
  Every polyhedral loss $L$ embeds a discrete loss.
\end{theorem}
\begin{proof}
  Let $L:\reals^d\to\reals_+^\Y$ be a polyhedral loss, and $\Gamma = \prop{L}$.
  By Lemma~\ref{lem:polyhedral-range-gamma}, $\Gamma(\simplex)$ is finite set. 
  For each $U\in \Gamma(\simplex)$, select $u_U \in U$, and let $\Sc = \{u_U : U \in\Gamma(\simplex)\}$, which is again finite.
  For any $p\in\simplex$ then, let $U = \Gamma(p)$.
  We have $U \in \Gamma(\simplex)$ by definition, and thus some $u_U \in \Sc$; in particular, $u_U \in U = \Gamma(p)$.
  We conclude that $\Sc$ is representative for $L$.
  Proposition~\ref{prop:representative-embeds-restriction} now states that $L$ embeds $L|_\Sc$.
\end{proof}

We now turn to the reverse direction: which discrete losses are embedded by some polyhedral loss?
Perhaps surprisingly, we show in Theorem~\ref{thm:discrete-loss-poly-embeddable} that \emph{every} discrete loss is embeddable by a polyhedral surrogate.
Combining this result with Theorem~\ref{thm:poly-embeds-discrete} establishes Theorem~\ref{thm:embed-poly-main}.
Further combining with Theorem~\ref{thm:link-main}, proved in the following section, this construction gives a calibrated polyhedral surrogate for every discrete target loss.

In the proof, we apply a result we will prove in \S~\ref{sec:min-rep-sets}: a minimizable surrogate embeds a discrete loss if and only if their Bayes risks match (Proposition~\ref{prop:embed-bayes-risks}).

\begin{theorem}\label{thm:discrete-loss-poly-embeddable}
  Every discrete loss $\ell:\R \to \reals^\Y_+$ is embedded by a polyhedral loss.
\end{theorem}
\begin{proof}
  Let $n = |\Y|$, and let $C:\reals^n \to \reals$ be given by $(-\risk{\ell})^*$, the convex conjugate of $-\risk{\ell}$.
  From standard results in convex analysis, $C$ is polyhedral as $-\risk{\ell}$ is, and $C$ is finite on all of $\reals^\Y$ as the domain of $-\risk{\ell}$ is bounded~\cite[Corollary 13.3.1]{rockafellar1997convex}.
  Note that $-\risk{\ell}$ is a closed convex function, as the infimum of affine functions, and thus $(-\risk{\ell})^{**} = -\risk{\ell}$.
  Define $L:\reals^n\to\reals^\Y$ by $L(u) = C(u)\ones - u$, where $\ones\in\reals^\Y$ is the all-ones vector.
  As $C$ is polyhedral, so is $L$.
  We first show that $L$ embeds $\ell$, and then establish that the range of $L$ is in fact $\reals^\Y_+$, as desired.

  We compute Bayes risks and apply Proposition~\ref{prop:embed-bayes-risks} to see that $L$ embeds $\ell$.
  Observe that $\risk{\ell}$ is polyhedral as $\ell$ is discrete.
  For any $p\in\simplex$, we have
  \begin{align*}
    \risk{L}(p)
    &= \inf_{u\in\reals^n} \inprod{p}{C(u)\ones - u}\\
    &= \inf_{u\in\reals^n} C(u) - \inprod{p}{u}\\
    &= -\sup_{u\in\reals^n} \inprod{p}{u} - C(u)\\
    &= -C^*(p) = - (-\risk{\ell}(p))^{**} = \risk{\ell}(p)~.
  \end{align*}
  It remains to show $L(u)_y \geq 0$ for all $u\in\reals^n$, $y\in\Y$.
  Letting $\delta_y\in\simplex$ be the point distribution on outcome $y\in\Y$, we have for all $u\in\reals^n$, $L(u)_y \geq \inf_{u'\in\reals^n} L(u')_y = \risk{L}(\delta_y) = \risk{\ell}(\delta_y) \geq 0$, where the final inequality follows from the nonnegativity of $\ell$.
\end{proof}

The proof of Theorem~\ref{thm:discrete-loss-poly-embeddable} uses a construction via convex conjugate duality similar to many constructions in the literature. 
For example, the min-max objective in the literature on adversarial prediction~\citep{asif2015adversarial,farnia2016minimax,fathony2016adversarial,fathony2018consistent} is a special case of this construction when one unfolds the definition of the convex conjugate of $-\risk \ell$.
\citet{reid2012convexity} construct a canonical link function for proper losses with differentiable Bayes risks; the link maps a report $p\in\simplex$ to the gradient of the Bayes risk at $p$, which uses the same duality as above.
\citet[Proposition 3]{duchi2018multiclass} give essentially the same construction as ours, but only comment on the calibration of surrogates under such constructions for multiclass classification tasks given by strictly concave losses, excluding polyhedral surrogates.
Finally, a similar construction also appears in the design of prediction markets \cite{abernethy2013efficient} and in connections between proper losses and mechanism design \cite{frongillo2014general}.

\section{Consistency and Linear Regret Transfer via Separated Links}
\label{sec:calibration}

We have seen that every polyhedral loss embeds some discrete loss.
The embedding itself tells us how to link the embedded points back to the discrete reports: link $\varphi(r)$ back to $r$.
But it is not clear how to extend this to yield a full link function $\psi: \reals^d \to \R$, and whether such a $\psi$ can lead to consistency.
In this section, we prove Theorem~\ref{thm:link-main}, restated below, via a construction to generate calibrated links for \emph{any} polyhedral surrogate.
Recalling that calibration is equivalent to consistency for discrete targets, this result implies that an embedding always yields consistency.

The key idea behind our construction is the notion of \emph{separation}, a condition which is equivalent to calibration for discrete prediction problems.
Roughly, given a surrogate $L$ and discrete target $\ell$, a link is $\epsilon$-separated if the distance between any $L$-optimal point in $\reals^d$, and any point that links to an $\ell$-suboptimal report, is at least $\epsilon$.
We also show how this characterization also leads to linear \emph{regret transfer} or \emph{surrogate regret} bounds.

\subsection{Separation}

Recall that for indirect elicitation, any point $u \in \Gamma(p)$ must link to a report $\psi(u) \in \gamma(p)$. 
(In terms of losses, $u$ minimizing expected $L$-loss implies that $\psi(u)$ minimizes expected $\ell$-loss, with respect to $p$.)
The idea of separation is that points in the neighborhood of $u$ must also link to to a report in $\gamma(p)$.
Furthermore, there must be a uniform lower bound $\epsilon$ on the size of any such neighborhood.

\begin{definition}[Separated Link]\label{def:sep-link}
  Let properties $\Gamma:\simplex\toto\reals^d$ and $\gamma:\simplex\toto\R$ be given.
  We say a link $\psi:\reals^d\to\R$
  is \emph{$\epsilon$-separated with respect to $\Gamma$ and $\gamma$} if for all $u\in\reals^d$, $p\in\simplex$ with $\psi(u)\notin\gamma(p)$, we have $d_\infty(u,\Gamma(p)) \geq \epsilon$, where $d_\infty(u,A) \doteq \inf_{a\in A} \|u-a\|_\infty$.
  \footnote{\citet{frongillo2021surrogate} define $\epsilon$-separation with a strict inequality $d_\infty(u,\Gamma(p)) > \epsilon$; we adopt a weak inequality as it is more natural in applications, such as hinge loss for binary classification.}
  Similarly, we say $\psi$ is $\epsilon$-separated with respect to $L$ and $\ell$ if it is $\epsilon$-separated with respect to $\prop{L}$ and $\prop{\ell}$.
\end{definition}

\begin{restatable}{theorem}{calibratedseparated} \label{thm:calibrated-separated}
  Let polyhedral surrogate $L:\reals^d \to \reals^\Y_+$, discrete loss $\ell:\R\to\reals^\Y_+$, and link $\psi:\reals^d\to\R$ be given.
  Then $(L,\psi)$ is calibrated with respect to $\ell$ if and only if
  $\psi$ is $\epsilon$-separated with respect to $L$ and $\ell$ for some
  $\epsilon>0$.
\end{restatable}
Intuitively, calibration of a polyhedral surrogate and separated link follows from two facts.
First, Lemma~\ref{lem:polyhedral-range-gamma} states that a polyhedral surrogate has only a finite number of ``optimal report sets'' $\{\Gamma(p) : p \in \simplex\}$.
Second, for a given $p$, the expected surrogate loss for a suboptimal point scales with the distance from the optimal set $\Gamma(p)$ at some minimum linear rate $\alpha > 0$; this rate is related to \emph{Hoffman constants}.
Combining with the first fact gives a universal minimum constant $\alpha$ for all $p$.
Now bringing in $\epsilon$-separation, any surrogate report linking to a suboptimal target report has expected surrogate loss at least $\alpha \cdot \epsilon > 0$.
On the other hand, if a link is not separated, then the same two facts imply that, for a sequence of surrogate reports that get arbitrarily close to an optimal report set while linking to a suboptimal target report, this sequence has expected loss approaching optimal, violating calibration.
See \S~\ref{sec:equiv-sep-calib} for the proof.

\subsection{Consistency}

\restatehack{
  \begin{theorem}
    \label{thm:calibrated-separated}
    \label{thm:thickened-separated}
  \end{theorem}}

To prove Theorem \ref{thm:link-main}, that embedding implies calibration, we now show how to construct a calibrated link from an embedding.
In light of Theorem~\ref{thm:calibrated-separated}, it now suffices to show that for any polyhedral $L$ embedding some $\ell$, there exists a \emph{separated} link $\psi$ with respect to $L$ and $\ell$.
Construction~\ref{const:eps-thick-link}, discussed next, ``thickens'' a given embedding to produce a link.
Theorem \ref{thm:thickened-separated} states that, for a small enough choice of $\epsilon$, that link is separated.
See \S~\ref{app:sep-link-exists} for the proof.
\begin{restatable}{theorem}{thickenedseparated} \label{thm:thickened-separated}
  Let polyhedral surrogate $L:\reals^d \to \reals^\Y_+$ embed the discrete loss $\ell:\R\to\reals^\Y_+$.
  Then there exists $\epsilon_0 > 0$ such that, for all $0 < \epsilon \leq \epsilon_0$, Construction~\ref{const:eps-thick-link} for $L,\ell,\epsilon,\|\cdot\|$ produces a nonempty set of links, all of which are $\epsilon$-separated with respect to $L$ and $\ell$.
\end{restatable}

To set the stage for Construction \ref{const:eps-thick-link}, we sketch the two main steps in proving Theorem \ref{thm:thickened-separated}: (i) showing that one can define a link $\psi$ on all possible optimal points of $L$; (ii) ``thickening'' $\psi$ so that it is separated.

For (i), given the embedding $\varphi: \Sc\to\reals^d$, begin by linking each embedding point back to its original report, so that $\psi(\varphi(r)) = r$.
Now we wish to determine $\psi(u)$ for non-embedding points $u\in\reals^d$ which optimize $L$ for some distribution.
Let $\Gamma = \prop L$ and define $\U = \{\Gamma(p) \mid p\in\simplex\}$ to be the ``range'' of $\Gamma$, i.e., all possible optimal sets.
For each $U\in\U$, we can define $R_U = \{ r\in\Sc \mid \varphi(r) \in U\}$ to be the set of target reports which, by definition of embedding, are $\ell$-optimal when $U$ is $L$-optimal.
We would like to restrict $\psi(U) \subseteq R_U$, so that optimal surrogate reports are mapped to optimal target reports.
The challenge is that we could have a point $u\in U\cap U'$ for two optimal sets $U,U'\in\U$, and \emph{a priori}, it could be that $R_U \cap R_{U'} = \emptyset$.
The first step of the proof is to show that this scenario cannot arise.
\begin{equation}
  \label{eq:intersection-property-embedding}
  \forall \U'\subseteq\U, \; \cap_{U\in\U'} U \neq \emptyset \implies \cap_{U\in\U'} R_U \neq \emptyset~.
\end{equation}
\noindent
Thus, for any $u\in\cup\U$, we have a nonempty set of valid choices for $\psi(u)$.
(Eq.~\eqref{eq:intersection-property-embedding} may appear similar to indirect elicitation; in fact the two conditions are equivalent, as we discuss in \S~\ref{sec:poly-ie-consistency}.)

For (ii), we show that this link can be ``thickened'' by some positive $\epsilon$, as described next.
Let $U\in\U$.
By the above, $\psi$ is already correct on $U$.
Now, we ``thicken'' $U$ to obtain $U_{\epsilon} = \{u : \|u - U\| \leq \epsilon\}$.
Then we require that all points in $U_{\epsilon}$ are linked to some element of $R_U$.
For $\epsilon > 0$, this condition directly implies separation.

It is not clear that this linking is possible, however, because a point $u$ may be in the intersection of several thickened sets $U_{\epsilon}, U'_{\epsilon}$, etc., corresponding to $\Gamma(p), \Gamma(p')$, etc.
Therefore, we need to take each possible collection $U,U'$, etc., show that their intersection (if nonempty) contains a legal choice for the link, and then thicken their intersection in an analogous way.

To do so, given $u \in U_\epsilon \cap U'_\epsilon \cap \dots$, we define a \emph{link envelope} $\Psi(u)$ which encodes the remaining legal choices for $\psi(u)$ after imposing the requirements for each such set $U_\epsilon,U'_\epsilon$, etc.
The key claim is that, for small enough $\epsilon > 0$, $\Psi(u)$ is nonempty: at least one permitted value for $\psi(u)$ remains.
This claim follows from a geometric result (Lemma~\ref{lem:thick-intersect}) that, for all small enough $\epsilon$, a subset of thickenings $U_{\epsilon}$ intersect if and only if the $U$ sets themselves intersect.
When they do intersect, eq.~\eqref{eq:intersection-property-embedding} implies that there exists a permitted choice of link for the intersection of the thickenings.
It is crucial that, by Lemma~\ref{lem:polyhedral-range-gamma}, polyhedral surrogates only have finitely many sets of the form $U = \Gamma(p)$.
Together, these observations yield a single uniform smallest $\epsilon$ such that the key claim is true for all $u \in \reals^d$.

Given the above proof sketch, the following construction is relatively straightforward.
We initialize the link using the embedding points and optimal report sets, then adjust $\Psi$ to narrow down to only legal choices; we then pick from $\psi(u)$ from $\Psi(u)$ arbitrarily.
Theorem \ref{thm:thickened-separated} implies that, for all small enough $\epsilon$, the resulting link $\psi$ is well-defined at all points, and $\epsilon$-separated.
\begin{construction}[$\epsilon$-thickened link] \label{const:eps-thick-link}
  Let $L:\reals^d\to\reals^\Y_+$, $\ell:\R\to\reals^\Y_+$, $\epsilon > 0$, and a norm $\|\cdot\|$ be given, such that $L$ is polyhedral and embeds $\ell$ via the embedding $\varphi: \Sc \to \reals^d$ for a representative set $\Sc\subseteq\R$.
  Define $\Gamma = \prop L$ and $\U = \{\Gamma(p) \mid p \in \simplex\}$.
  For all $U \in \U$, define $R_U = \{r \in \Sc \mid \varphi(r) \in U\}$.
  The \emph{$\epsilon$-thickened link} $\psi$ is constructed as follows.
  First, initialize the \emph{link envelope} $\Psi: \reals^d \to 2^{\Sc}$ by setting $\Psi(u) = \Sc$ for all $u$.
  Then for each $U \in \U$, for all points $u$ such that $\inf_{u^* \in U} \|u^*-u\| < \epsilon$, update $\Psi(u) = \Psi(u) \cap R_U$.
  If we have $\Psi(u)\neq\emptyset$ for all $u\in\reals^d$, then the construction \emph{produces a link} $\psi \in \Psi$ pointwise, breaking ties arbitrarily.
\end{construction}
In \S~\ref{sec:poly-ie-consistency} we generalize this construction beyond embeddings, where we only require that $L$ indirectly elicits $\prop\ell$.
There we will see that, perhaps surprisingly, this construction recovers every possible calibrated link.

Applying Construction \ref{const:eps-thick-link} also enables one to verify the consistency of a given proposed link $\psi^*$.
For a given $\epsilon$ and norm $\|\cdot\|$, suppose one follows the routine of Construction~\ref{const:eps-thick-link} until the last step in which values for the link $\psi$ are selected.
Instead, we can simply test whether the proposed link values are contained in the valid choices, i.e., if $\psi^*(u) \in \Psi(u)$ for all $u\in\reals^d$.
If so, then the proposed link $\psi^*$ is calibrated.
See \S~\ref{sec:top-k} for an illustration of this test.
On the other hand, if $\psi^*$ cannot be produced from Construction \ref{const:eps-thick-link}, then it cannot be a calibrated link.
This impossibility can be shown, for example, if there exists a point $u$ where $\Psi(u)$ is empty for all $\epsilon > 0$.

Construction~\ref{const:eps-thick-link} is not necessarily computationally efficient as the number of labels $n$ grows.
In practice this potential inefficiency is not typically a concern, as the family of losses typically has some closed form expression in terms of $n$, and thus the construction can proceed at the symbolic level.
We illustrate this formulaic approach in \S~\ref{sec:abstain}.

\subsection{Surrogate regret bounds}\label{subsec:regret-bounds}
Recall that the approach of surrogate risk minimization is to learn a hypothesis $h$ that minimizes expected surrogate loss, then output hypothesis $\psi \circ h$, which hopefully minimizes expected target loss.
We would like surrogates where a bound on surrogate loss of $h$ immediately implies a bound on target loss of $\psi \circ h$.
One can formalize this problem in terms of \emph{regret}, as follows.
Fix a data distribution $\D$.
The \emph{surrogate regret} $R_L$ of $h$ and the \emph{target regret} $R_{\ell}$ of the implied hypothesis $\psi \circ h$, are given by
\begin{align*}
  R_L(h;\D) &= \E_{(X,Y)\sim\D} L(h(X))_Y - \inf_{h':\X\to\reals^d} \E_{(X,Y)\sim\D} L(h'(X))_Y~,
  \\
  R_\ell(\psi\circ h;\D) &= \E_{(X,Y)\sim\D} \ell(\psi(h(X)))_Y - \inf_{g':\X\to\R} \E_{(X,Y)\sim\D} \ell(g'(X))_Y~,
\end{align*}
where the infimum is taken over all measurable functions.
The infimum represents the risk of the Bayes optimal hypothesis, so regret can be viewed as \emph{excess risk} under the assumption that the learner's hypothesis class contains the Bayes optimal.

Consistency means that if the surrogate regret of $h$ converges to zero, then the target regret of $\psi \circ h$ does as well; in other words, $R_L(h;\D) \to 0$ implies $R_{\ell}(\psi \circ h;\D) \to 0$.
Consistency is therefore a minimal requirement; in general, we are also interested in the \emph{rate} at which the target regret diminishes, as a function of the number of data points $n$.
A \emph{regret transfer bound}, also called a \emph{surrogate regret bound}, gives a guarantee on the relationship between the rates of convergence of $R_L$ and $R_{\ell}$.
For example, a surrogate with a fast rate of convegence $R_L \to 0$ as $n \to \infty$ is not very useful if we nevertheless have a slow rate $R_{\ell} \to 0$.

We show that, for any polyhedral surrogate, the transfer is linear: if surrogate regret diminishes at a rate of $O(f(n))$, then the target rate is also $O(f(n))$.
In particular, fast convergence in surrogate regret implies fast convergence in target regret.
\begin{theorem}
  \label{thm:linear-regret-bound}
  Let $(L,\psi)$ be consistent for a discrete loss $\ell$, and $L$ polyhedral.
  Then there exists $c > 0$ such that, for all hypotheses $h$ and data distributions $\D$, we have $R_{\ell}(\psi \circ h;\D) \leq c \cdot R_L(h;\D)$.
\end{theorem}
In the proof (\S~\ref{app:regret-bounds}), we further show that the constant $c$ can be decomposed in terms of three constants, which depend on $L$, $\psi$, and $\ell$, respectively.
Specifically, we may write $c = C_\ell H_L / \epsilon_\psi$, where $H_L$ is the Hoffman constant for $L$, $\epsilon_\psi$ the separation of $\psi$, and $C_\ell$ the maximum loss gap of $\ell$.
This expression gives the intuition that larger link separation is generally better for performance.
In some cases, this bound can be tightened, as we discuss in \S~\ref{sec:regret-tighter-bounds}.
See \citet{frongillo2021surrogate} for a quadratic lower bound on the rate transfer for sufficiently non-polyhedral surrogates.

\section{Application to Specific Surrogates}\label{sec:applications}

Our results give a framework to construct consistent polyhedral surrogates and link functions for any discrete target loss, as well as to verify consistency or inconsistency for specific surrogate and link pairs.
Below, we illustrate the power of this framework with specific examples from the literature.
To warm up, we study the abstain surrogate given by~\citet{ramaswamy2018consistent}, which is an embedding, and show how to rederive their link function and surrogate regret bounds (\S~\ref{sec:abstain}). 
We then give three examples of subsequent works that use our framework, in the context of structured binary classification (\S~\ref{sec:lovasz-hinge}), multiclass classification (\S~\ref{sec:winge}), and top-$k$ classification (\S~\ref{sec:top-k}).
In each of these latter three examples, our framework illuminates the behavior of inconsistent surrogates by revealing the discrete losses they embed, i.e., the true targets for which they are consistent.
In structured binary classification and top-$k$ classification, our framework also gives new consistent surrogates and link functions which appear challenging to derive otherwise.

\subsection{Applying the embedding framework}
\label{sec:apply-embedd-fram}

When using our framework to study the consistency or inconsistency of an existing surrogate $L:\reals^d \to \reals^\Y_+$, often the first step is determining the loss it embeds.
To do so, we suggest the following general approach.
First, for each $y\in\Y$, divide $\reals^d$ into a finite number of polyhedral regions on which $L(\cdot)_y$ is an affine function.
Second, identify the vertices of these polyhedral regions.
\footnote{In some cases, these regions do not have vertices, such as the top-$k$ surrogates in \S~\ref{sec:top-k} which are invariant in the all-ones direction; here one can restrict to a subspace, or otherwise select among equivalent reports.}
Third, conclude that the union of these vertices, $\Sc\subset\reals^d$, is a finite representative set for $L$.
Now $L$ embeds $L|_\Sc$ from Proposition~\ref{prop:representative-embeds-restriction}.
From here one can further remove redundant reports until arriving at a tight embedding if desired, or re-label embedded reports to a more intuitive form; call this resulting embedded loss $\hat \ell$.

Once the embedded discrete loss $\hat\ell$ is known, the behavior of the surrogate $L$ becomes more clear.
In particular, we learn what problem $L$ is actually solving, as captured by $\hat\ell$.
If this problem $\hat\ell$ is not the desired target problem $\ell$, we can still derive restrictions on label distributions (e.g., $\P \subseteq \simplex$) for which $L$ is $\P$-calibrated for $\ell$.
Any level set of the embedded property $\hat\gamma = \prop{\hat\ell}$ which spans multiple level sets of the target property $\gamma=\prop{\ell}$ will lead to inconsistency for $\ell$ (Figure~\ref{fig:abstain-intuition-inconsistent}).
To obtain consistency with respect to a desired target, therefore, it suffices to restrict to the union of level sets of $\hat\gamma$ which are each fully contained in some level set of $\gamma$.

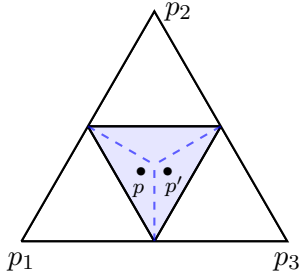
\begin{figure}
	\begin{minipage}{0.28\linewidth}

\begin{tikzpicture} [scale=\tikzfigscale, thick, tdplot_main_coords]
\coordinate (orig) at (0,0,0);

\coordinate[label=below:${p_1}$] (h) at (1,0,0);
\coordinate[label=below:${p_3}$] (m) at (0,1,0);
\coordinate[label=right:${p_2}$] (l) at (0,0,1);

\coordinate (onetwo) at (1/2, 0, 1/2);
\coordinate (onethree) at (1/2, 1/2, 0);
\coordinate (twothree) at (0, 1/2, 1/2);
\coordinate (uniform) at (1/3, 1/3, 1/3);

\draw[simplex] (h) -- (m) -- (l) -- (h);

\begin{scope}
\clip (h) -- (m) -- (l) -- (h);

\draw (h) -- (onetwo) -- (onethree) -- cycle;
\draw (m) -- (onethree) -- (twothree) -- cycle;
\draw (l) -- (onetwo) -- (twothree) -- cycle;

\draw[fill = blue, fill opacity = 0.1] (onetwo) -- (twothree) -- (onethree) -- cycle;

\draw[dashed, color=blue, opacity = 0.6] (onetwo) -- (uniform);
\draw[dashed, color=blue, opacity = 0.6] (twothree) -- (uniform);
\draw[dashed, color=blue, opacity = 0.6](onethree) -- (uniform);

\node(p) at (2/5, 3/10, 3/10){\textbullet};
\node at (0.45, 0.33, .22){\scriptsize $p$};
\node(pprime) at (3/10, 2/5, 3/10){\textbullet} node[anchor=north west]{\scriptsize $p'$};

\end{scope}

\end{tikzpicture}
 \end{minipage}
\begin{minipage}{0.72\linewidth}
	\caption{Using an embedding to show inconsistency.
    Let $L$ be a surrogate embedding $\ell$, and let $\ell$ be a desired target; here $L=\BEP$, $\hat\ell = \ellabstain$ (\S~\ref{sec:abstain}) and $\ell$ is 0-1 loss for multiclass classification.
    Let $\Gamma = \prop{L}$, $\hat\gamma = \prop{\hat\ell}$, and $\gamma = \prop{\ell}$ be the properties elicited by these losses; here $\gamma = \mode$, as 0-1 loss elicits the mode.
    The level sets of $\hat\gamma$ are given in solid black lines, and those of $\gamma$ in dashed blue lines.
    To exhibit non-calibration, take distributions $p, p'$ in the relative interior of the blue cell $\hat \gamma_{\hat r}$ (here $\hat r=\bot$ for $\ellabstain$) but in different cells of $\gamma$.
    These distributions will satisfy $\hat\gamma(p) = \hat\gamma(p')$, and thus $\Gamma(p) = \Gamma(p')$ by definition of embedding, but $\gamma(p) \cap \gamma(p') = \emptyset$.
    Taking $u\in\Gamma(p)=\Gamma(p')$, it is impossible to define $\psi(u)$ to satisfy calibration, as $\psi(u)$ cannot be in $\gamma(p)$ and $\gamma(p')$ simultaneously.
    In particular, even though $u$ is $L$-optimal for both $p$ and $p'$, $\psi(u)$ will be $\ell$-suboptimal for at least one, violating calibration.
    We may impose restrictions on the conditional label distributions to remove the blue cell, however, and thus satisfy calibration.
    For this example, the $\BEP$ is classification-calibrated if one restricts to the set of distributions where at least one label $y\in\Y$ has probability $p_y \geq \tfrac 1 2$.
	}
	\label{fig:abstain-intuition-inconsistent}
	\end{minipage}
\end{figure}

With an embedding in hand, Construction~\ref{const:eps-thick-link} provides a calibrated link function from $L$ to $\hat\ell$.
This construction is especially beneficial in cases where the most intuitive link functions are not calibrated, and no known calibrated link is known; see \S~\ref{sec:lovasz-hinge} for a somewhat intricate example.
Surrogate regret bounds then follow from Theorem~\ref{thm:linear-regret-bound}, as we illustrate in \S~\ref{sec:abstain}.
In particular, our results imply the existence of linear regret transfer bounds bounds for several applications where no such bounds were known (\S~\ref{sec:lovasz-hinge},~\ref{sec:winge},~\ref{sec:top-k}).

Finally, our link construction can even be useful in cases where the search for consistent surrogates has been restricted to those accommodating a particular canonical link function $\psi$.
For example, one typically uses the sign link for binary classification, and the argmax link (the $k$ largest coordinates) for top-$k$ classification (\S~\ref{sec:top-k}).
As we show in Proposition~\ref{prop:link-converse}, Construction~\ref{const:eps-thick-link} fully characterizes the set of possible calibrated link functions for a polyhedral embedding via the link envelope $\Psi$, so $\psi$ is calibrated if and only if it is contained in $\Psi$ for some $\epsilon>0$.
We demonstrate this approach for top-$k$ classification in \S~\ref{sec:top-k}.
More generally, however, while such canonical link functions may be intuitive for a given problem, our results suggest that researchers should consider setting them aside and instead let Construction~\ref{const:eps-thick-link} determine the link.

\subsection{Consistency of abstain surrogate and link construction}
\label{sec:abstain}

Several authors consider a variant of multiclass classification, with the addition of an \emph{abstain} option~\citep{bartlett2008classification,ramaswamy2018consistent,madras2018predict,elyaniv2010foundations,cortes2016learning}.
\citet{ramaswamy2018consistent} study the loss $\ellabstain: \Y \cup \{\bot\} \to \reals^\Y_+$ defined by $\ellabstain(r)_y = 0$ if $r=y$, $1/2$ if $r = \bot$, and 1 otherwise.
The report $\bot$ corresponds to ``abstaining'' to predict, in exchange for a constant loss regardless of outcome $y$. 
\citeauthor{ramaswamy2018consistent} give the polyhedral \emph{binary encoded predictions (BEP)} surrogate $\BEP$, and the link $\psi^\infty$ which they show is calibrated for $\ellabstain$.
Letting $d = \ceil{\log_2 |\Y|}$, their surrogate $\BEP : \reals^d \to \reals^\Y_+$ is given by
\begin{equation}\label{eq:abstain-surrogate}
\BEP(u)_y = \max_{j \in [d]} \left(1 - \varphi(y)_j u_j\right)_+~,
\end{equation}
where $\varphi:\Y\to\{-1,1\}^d$ is an injection.
\footnote{To translate our notation to that of \citet{ramaswamy2018consistent}, take $B = -\varphi$.}
Observe that $\BEP$ is exactly hinge loss when $|\Y|=2$ and thus $d=1$. 
The authors show that the link $\psi^{\infty}$ is calibrated, where
\begin{equation}\label{eq:abstain-link}
  \psi^{\infty}(u) = \begin{cases}
	\bot & \min_{i \in [d]} |u_i| \leq 1/2\\ 
	\varphi^{-1}(\sgn(u)) &\text{otherwise}
  \end{cases}~,
\end{equation}
and they go on to establish linear surrogate regret bounds for $(\BEP,\psi^{\infty})$.

Using our framework, one can show that $\BEP$ embeds (2 times) $\ellabstain$, with the embedding given by $\varphi$ above where we define $\varphi(\bot) = 0 \in \reals^d$.
(Following the general procedure outlined above, the regions where $\BEP$ is affine all have vertices in the set $\{-1,1\}^d \cup \{0\}$, meaning it is representative, and $\BEP$ restricted to that set is precisely $2\ellabstain \circ \varphi^{-1}$.)

As an illustration, one can use the fact that $\BEP$ embeds $\ellabstain$ to verify that $\BEP$ is inconsistent for multiclass classification, i.e., with respect to 0-1 loss.
In particular, since the abstain report $\bot$ is $\ellabstain$-optimal whenever $\max_{y\in\Y} p_y \leq 1/2$, by the definition of embedding, the origin $0\in\reals^d$ is $\BEP$-optimal for the same distributions.
Recalling that 0-1 loss elicits the mode, one can now find two distributions with different modes but for which $0$ is $\BEP$-optimal, violating calibration (Figure~\ref{fig:abstain-intuition-inconsistent}).

Moreover, as we illustrate in Figure~\ref{fig:abstain-links}(L), the link $\psi^{\infty}$ proposed by \citeauthor{ramaswamy2018consistent} can be recovered from Construction~\ref{const:eps-thick-link} by choosing the norm $\|\cdot\|_\infty$ and $\epsilon=1/2$ (or smaller).
Hence, our framework could have simplified the process of finding $\psi^\infty$, and the corresponding proof of consistency.
It also could have simplified the derivation of surrogate regret bounds (\S~\ref{subsec:regret-bounds}); we show how to recover the tight bound of \citeauthor{ramaswamy2018consistent} for the BEP surrogate in \S~\ref{sec:regret-tighter-bounds}.

To illustrate these points further, consider the alternate link $\psi^1$
in Figure~\ref{fig:abstain-links}(R),
given by
\begin{equation}\label{eq:abstain-link-1}
  \psi^1(u) = \begin{cases}
	\bot & \|u\|_1 \leq 1\\
	\varphi^{-1}(\sgn(u)) &\text{otherwise}
  \end{cases}~.
\end{equation}
This link is the result of Construction~\ref{const:eps-thick-link} for norm $\|\cdot\|_1$ and the choice $\epsilon=1$, which proves calibration of $(\BEP,\psi^1)$ with respect to $\ellabstain$.
Aside from its simplicity, one possible advantage of $\psi^1$ is that it assigns $\bot$ to much less of the surrogate space $\reals^d$.

\begin{figure}
\begin{center}
\begin{minipage}{0.32\linewidth}
\includegraphics[width=\linewidth]{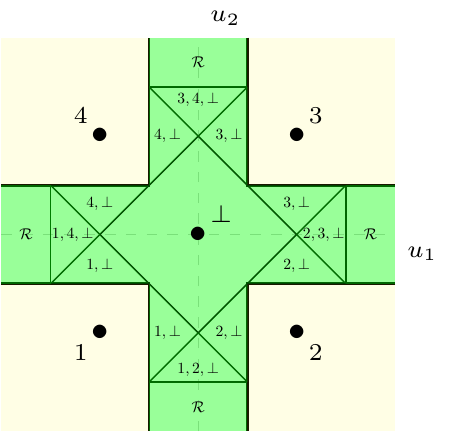}
\end{minipage}\hfill
\begin{minipage}{0.32\linewidth}
\includegraphics[width=\linewidth]{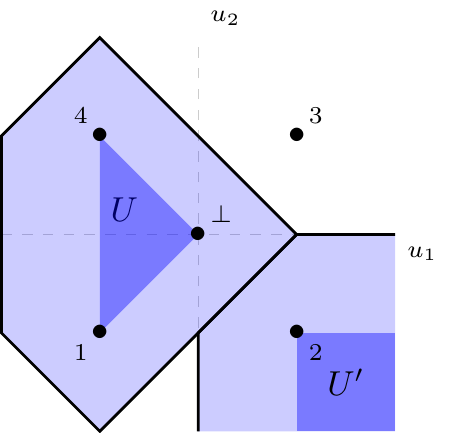}
\end{minipage}\hfill
\begin{minipage}{0.32\linewidth}
\includegraphics[width=\linewidth]{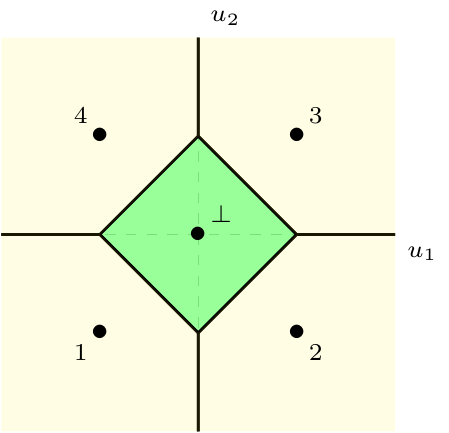}
\end{minipage}\hfill
\caption{Designing links for $\BEP$ with $d=2$ using Construction~\ref{const:eps-thick-link}. The embedding is shown in bold labeled by the corresponding reports. (L) The link envelope $\Psi$ resulting from Construction~\ref{const:eps-thick-link} using $\|\cdot\|_\infty$ and $\epsilon = 1/2$, and a possible link $\psi$ which matches eq.~\eqref{eq:abstain-link} from~\cite{ramaswamy2018consistent}.  (M) An illustration of the thickened sets for two sets $U, U' \in \U$, using $\|\cdot\|_1$ and $\epsilon = 1$. (R) The envelope $\Psi$ and link $\psi$ 
  using $\|\cdot\|_1$ and $\epsilon = 1$.}
\label{fig:abstain-links}
\end{center}
\end{figure}

\subsection{Lov\'asz hinge and the structured abstain problem}
\label{sec:lovasz-hinge}

\newcommand{\dis}{\mathrm{dis}}
\newcommand{\abs}{\mathrm{abs}}

Many structured prediction settings can be thought of as making multiple predictions at once, with a loss function that jointly measures error based on the relationship between these predictions~\cite{hazan2010direct, gao2011consistency, osokin2017structured}.
In the case of $k$ binary predictions, these settings are typically formalized by taking the predictions and outcomes to be $\R=\Y=\{-1,1\}^k$, with the $i$th coordinate giving the result for the $i$th binary prediction.
A natural family of losses are those which are functions of the misprediction or disagreement set $\dis(r,y) = \{i \in [k] \mid r_i \neq y_i\}$, meaning we may write $\ell^f(r)_y = f(\dis(r,y))$ for some set function $f:2^{[k]}\to\reals$.
For example, Hamming loss is given by $f(S) = |S|$.
In an effort to provide a general convex surrogate for these settings when $f$ is a submodular function, Yu and Blaschko~\cite{yu2018lovasz} introduce the \emph{Lov\'asz hinge} surrogate $L^f:\reals^k\to\reals^\Y_+$ which leverages the well-known convex Lov\'asz extension of submodular functions.
While the authors provide theoretical justification and experiments, they leave open whether the Lov\'asz hinge is actually consistent for $\ell^f$.

\citet{ourlovaszpaper} use our embedding framework to resolve the consistency of $L^f$, showing that it is inconsistent with respect to $\ell^f$ outside of the trivial case where $f$ is modular (in which case $\ell^f$ is a weighted Hamming loss).
Moreover, they show that $L^f$ embeds a variant $\ellabs$ of $\ell^f$ where one is allowed to abstain on a set of indices $A \subseteq [k]$, which they call the \emph{structured abstain problem}.
The inclusion of abstain options is natural when observing that the BEP surrogate $\BEP$, for multiclass classification with an abtain option (\S~\ref{sec:abstain}), is the special case of $L^f$ where $f(S) = \ones\{S \neq \emptyset\}$.

To derive the discrete loss $\ellabs$ that $L^f$ embeds, the authors follow an approach similar to \S~\ref{sec:apply-embedd-fram} to show that the set $\V = \{-1,0,1\}^k$ is representative for $L^f$, for any choice of $f$.
From Proposition~\ref{prop:representative-embeds-restriction}, they conclude that $L^f$ embeds $\ellabs := L^f|_{\V}$.
Letting $\abs(v) = \{i\in[k] \mid v_i = 0\}$ denote the ``abstain'' set, we may write $\ellabs : \V \to \reals^\Y_+$ as
\begin{equation}
	\ellabs(v)_y = f(\dis(v,y) \setminus \abs(v)) + f(\dis(v,y))~.
\end{equation}
(Observe that $\abs(v,y) \subseteq \dis(v,y)$, since $y\in\{-1,1\}^k$.)
By Theorem~\ref{thm:link-main}, then, there is a link function such that the Lov\'asz hinge is consistent with respect to the structured abstain loss $\ellabs$.

As \citeauthor{ourlovaszpaper} observe, actually determining a calibrated link function in this case is nontrivial.
Simple threshold links like for the BEP surrogate in \S~\ref{sec:abstain} are not always calibrated, thus casting doubt that a trial-and-error approach for finding the link would be successful.
Instead, the authors leverage our thickened link construction (Construction~\ref{const:eps-thick-link}) to derive two links $\psi^*$ and $\psi^\diamond$, which have somewhat intricate geometric structure (Figure~\ref{fig:lovasz-links}).
Perhaps surprisingly, by deriving the link envelope $\Psi$ which is contained in the envelopes for $L^f$ for all submodular and increasing $f$, they prove that $\psi^*(u) \subseteq \Psi(u)$ and $\psi^\diamond(u) \subseteq \Psi(u)$ for all $u \in \reals^d$.
Thus, both $(L^f, \psi^*)$ and $(L^f, \psi^\diamond)$ are simultaneously calibrated with respect to $\ellabs$ for all such $f$.

\begin{figure}
	\begin{center}
		\begin{minipage}{0.32\linewidth}
			\includegraphics[width=\linewidth]{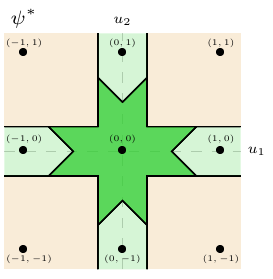}
		\end{minipage}\hfill
		\begin{minipage}{0.32\linewidth}
			\includegraphics[width=\linewidth]{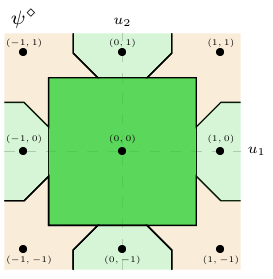}
		\end{minipage}\hfill
		\begin{minipage}{0.32\linewidth}
		\includegraphics[width=\linewidth]{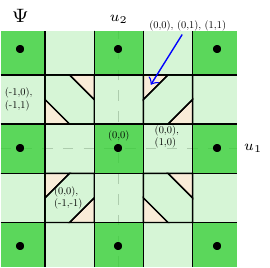}
		\end{minipage}\hfill
		\caption{Links $\psi^*$ and $\psi^\diamond$ such that $(L^f, \psi^*)$ and $\psi^\diamond$ are calibrated with respect to $\ellabs$ for all suitable $f$. Points in each region link to the embedding point contained in the region. Both are constructed via the link envelope $\Psi$ from Construction~\ref{const:eps-thick-link}, which yields possible choices for calibrated links.
			}
		\label{fig:lovasz-links}
	\end{center}
\end{figure}

\subsection{Embedding ordered partitions via Weston-Watkins hinge}
\label{sec:winge}
As the hinge loss is one of the most common surrogates for binary support vector machines (SVMs), original extensions to the multiclass setting included a one-vs-all reduction to the binary problem via hinge loss, generating ${n \choose 2}$ hyperplanes for $n$ labels.
Proposing a more efficient solution, \citet{weston1999support} give an alternate surrogate for multiclass SVM prediction, defined as follows for predictions $u \in \reals^n$,
\begin{equation}\label{eq:ww-hinge}
\LWW(u)_y = \sum_{i \in \Y : i \neq y} (1 - (u_y - u_i))_+~.
\end{equation}
This surrogate $\LWW$ was later shown to be inconsistent with respect to 0-1 loss~\citep{tewari2007consistency,liu2007fisher}.

\citet{wang2020weston} use our embedding framework to show that the Weston--Watkins hinge embeds the \emph{ordered partition} loss $\ellOP$, as defined below.
In turn, they recover the result of inconsistency with respect to 0-1 loss.
The report space for $\ellOP$ can be defined in terms of nested subsets of $[n] := \{1, \ldots, n\}$, as follows.
\footnote{To recover the partition of~\citet{wang2020weston}, one can define $S_i = T_i \setminus T_{i-1}$.}
\begin{align*}
\T = \{ (T_0,\ldots,T_s) \mid s \geq 1, \emptyset = T_0 \subsetneq T_1 \subsetneq \ldots \subsetneq T_s = [n]\}~.
\end{align*}
\noindent
The ordered partition target loss $\ellOP : \T \to \reals^\Y_+$ is then defined
\begin{align*}
\ellOP(T)_y &= \sum_{i=1}^{s} \left(|T_{i}| \cdot \Ind{y \not \in T_{i-1}} \right) -1~.~
\end{align*}
\noindent
The loss $\ellOP$ can be interpreted as a variation of 0-1 loss incorporating confidence: reports are a nested sequence of sets, and the penalty upon seeing label $y$ is the cardinality of the first set containing $y$, plus the cardinality of all earlier sets.

Upon showing that $\LWW$ embeds $\ellOP$, \citeauthor{wang2020weston} then characterize $\prop{\ellOP}$.
In the same manner as Figure~\ref{fig:abstain-intuition-inconsistent}, knowledge of the level sets of $\prop{\ellOP}$ clarifies which label distributions are the source of inconsistency for classification.
Removing these distributions gives a set $\P \subseteq \simplex$
such that $\LWW$ and the canonical link $\psi(u) : u \mapsto r \in \argmax_y \inprod{e_y}{u}$ are calibrated with respect to 0-1 loss on $\P \subseteq \simplex$ (i.e., such that eq.~\eqref{eq:calibrated} holds for all $p \in \P$).
See Figure~\ref{fig:ordered-partition} for an illustration.

\begin{figure}[t]
	\begin{minipage}{0.47\linewidth}
		\centering
		\includegraphics[width=0.9\linewidth]{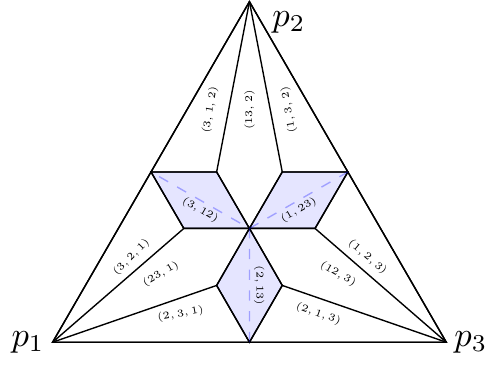}
	\end{minipage}
	\hfill
	\begin{minipage}{0.5\linewidth}
		\caption{Level sets of $\prop{\ellOP}$ (solid lines)
			juxtaposed against the level sets of $\mode$ (dashed lines).
      For the same reason as in Figure~\ref{fig:abstain-intuition-inconsistent}, the level sets of $\prop{\ellOP}$ whose relative interiors span multiple cells of the mode, colored blue, cannot be properly linked to the mode.
      These offending cells correspond to reports whose highest partition has more than one element, where in the white cells, the ``highest'' element of the partition is well-defined.}
		\label{fig:ordered-partition}
	\end{minipage}
\end{figure}

\subsection{Surrogates for top-$k$ classification}
\label{sec:top-k}
In settings like object recognition and information retrieval, it is natural to predict a set $S$ of labels.
In top-$k$ classification, one requires $|S|=k$, and given the true label $y$, the target loss is $\elltopk(S)_y = \ones\{y\notin S\}$~\citep{lapin2015top, lapin2016loss, lapin2018analysis,yang2018consistency,berrada2018smooth,rastegari2011scalable,reddi2019stochastic}.
In the literature on surrogates for top-$k$ classification, one goal has been to find a surrogate satisfying the following three desiderata: convexity, consistency, and piecewise linear (``hinge-like'') structure.
\citet{yang2018consistency} show that a number of previously proposed polyhedral losses, i.e., those which are convex and hinge-like, are inconsistent.
They further suggest that perhaps no surrogate could satisfy all three properties.

\citet{finocchiaro2022consistenttopk} apply the general approach outlined above to each of the polyhedral surrogates shown to be inconsistent by \citeauthor{yang2018consistency}, and determine the target problems they do solve, i.e., the discrete losses they embed.
Each of the examined surrogates embeds a discrete loss which can be viewed as a variant of the top-$k$ problem, allowing the algorithm to express varying levels of ``confidence'' on the top $k$ labels or report fewer than $k$ labels.
The label distributions for which these optimal reports differ from the optimal top-$k$ reports are shown in Table~\ref{tab:loss-slices} with $n=4$ and $k \in \{2,3\}$.
(Recall $n=|\Y|$, the number of labels.)

\begin{table*}
	\centering
	\begin{tabular}{ccccc}
		& $\prop{\Li{2}}$ & $\prop{\Li{3}}$ & $\prop{\Li{4}}$ & $\prop{L^k}$\\
		\hline \hline
		\rotatebox[origin=c]{90}{$k = 2$\hspace*{-2.5cm}} & \begin{tikzpicture} [scale=\tikzfigscale, thick, tdplot_main_coords]
\coordinate (orig) at (0,0,0);

\coordinate (uniform) at (1/4,1/4,1/4);
\coordinate[label=below:${p_1}$] (h) at (3/4,0,0);
\coordinate[label=below:${p_3}$] (m) at (0,3/4,0);
\coordinate[label=right:${p_2}$] (l) at (0,0,3/4);

\coordinate (one) at (3/4 - 1/8, 0, 1/8);
\coordinate (two) at (3/4 - 1/8, 1/8, 0);
\coordinate (three) at (1/8, 3/4 - 1/8, 0);
\coordinate (four) at (0, 3/4 - 1/8, 1/8);
\coordinate (five) at (0, 1/8, 3/4-1/8);
\coordinate (six) at (1/8, 0, 3/4-1/8);
\coordinate (seven) at (1/3, 0, 3/4 - 1/3);
\coordinate (eight) at (3/4 - 1/3, 0, 1/3);
\coordinate (nine) at (3/4 - 1/3, 1/3, 0);
\coordinate (ten) at (1/3, 3/4 - 1/3, 0);
\coordinate (eleven) at (0, 3/4 - 1/3, 1/3);
\coordinate (twelve) at (0, 1/3, 3/4 - 1/3);
\coordinate (thirteen) at (1/3, 3/4 - 2/3, 1/3);
\coordinate (fourteen) at (1/3, 1/3, 3/4 - 2/3);
\coordinate (fifteen) at (3/4 - 2/3, 1/3, 1/3);

\draw[simplex] (h) -- (m) -- (l) -- (h);

\begin{scope}
\clip (h) -- (m) -- (l) -- (h);

\draw (h) -- (one) -- (two) -- cycle;
\draw (l) -- (five) -- (six) -- cycle;
\draw (m) -- (three) -- (four) -- cycle;

\draw[fill = blue, fill opacity = 0.1] (one) -- (two) -- (nine)-- (fourteen) -- (thirteen) -- (eight) -- cycle;
\draw[fill = blue, fill opacity = 0.1] (three) -- (four) -- (eleven) -- (fifteen) -- (fourteen) -- (ten) -- cycle;
\draw[fill = blue, fill opacity = 0.1] (five) -- (six) -- (seven) -- (thirteen) -- (fifteen) -- (twelve) -- cycle;

\draw (nine) -- (ten) -- (fourteen) -- cycle;
\draw (seven) -- (eight) -- (thirteen) -- cycle;
\draw (eleven) -- (twelve) -- (fifteen) -- cycle;

\draw[fill = blue, fill opacity = 0.1] (thirteen) -- (fourteen) -- (fifteen) -- cycle;

\draw[dashed, color=blue, opacity = 0.6] (1/2, 1/4, 0) -- (uniform);
\draw[dashed, color=blue, opacity = 0.6] (1/2, 0, 1/4) -- (uniform);
\draw[dashed, color=blue, opacity = 0.6](1/4, 0, 1/2) -- (uniform);
\draw[dashed, color=blue, opacity = 0.6] (0, 1/4, 1/2) -- (uniform);
\draw[dashed, color=blue, opacity = 0.6](0, 1/2, 1/4) -- (uniform);
\draw[dashed, color=blue, opacity = 0.6] (1/4, 1/2, 0) -- (uniform);

\end{scope}

%
%
%
%

\end{tikzpicture} & \begin{tikzpicture} [scale=\tikzfigscale, thick, tdplot_main_coords]
\coordinate (orig) at (0,0,0);

\coordinate (uniform) at (1/4,1/4,1/4);
\coordinate[label=below:${p_1}$] (h) at (3/4,0,0);
\coordinate[label=below:${p_3}$] (m) at (0,3/4,0);
\coordinate[label=right:${p_2}$] (l) at (0,0,3/4);

\coordinate (one) at (3/4 - 1/8, 0, 1/8);
\coordinate (two) at (3/4 - 1/8, 1/8, 0);
\coordinate (three) at (1/8, 3/4 - 1/8, 0);
\coordinate (four) at (0, 3/4 - 1/8, 1/8);
\coordinate (five) at (0, 1/8, 3/4-1/8);
\coordinate (six) at (1/8, 0, 3/4-1/8);
\coordinate (seven) at (1/3, 0, 3/4 - 1/3);
\coordinate (eight) at (3/4 - 1/3, 0, 1/3);
\coordinate (nine) at (3/4 - 1/3, 1/3, 0);
\coordinate (ten) at (1/3, 3/4 - 1/3, 0);
\coordinate (eleven) at (0, 3/4 - 1/3, 1/3);
\coordinate (twelve) at (0, 1/3, 3/4 - 1/3);
\coordinate (thirteen) at (1/3, 3/4 - 2/3, 1/3);
\coordinate (fourteen) at (1/3, 1/3, 3/4 - 2/3);
\coordinate (fifteen) at (3/4 - 2/3, 1/3, 1/3);

\draw[simplex] (h) -- (m) -- (l) -- (h);

\begin{scope}
\clip (h) -- (m) -- (l) -- (h);

\draw (h) -- (one) -- (two) -- cycle;
\draw (l) -- (five) -- (six) -- cycle;
\draw (m) -- (three) -- (four) -- cycle;

\draw[fill = blue, fill opacity = 0.1] (one) -- (two) -- (nine)-- (fourteen) -- (thirteen) -- (eight) -- cycle;
\draw[fill = blue, fill opacity = 0.1] (three) -- (four) -- (eleven) -- (fifteen) -- (fourteen) -- (ten) -- cycle;
\draw[fill = blue, fill opacity = 0.1] (five) -- (six) -- (seven) -- (thirteen) -- (fifteen) -- (twelve) -- cycle;

\draw (nine) -- (ten) -- (fourteen) -- cycle;
\draw (seven) -- (eight) -- (thirteen) -- cycle;
\draw (eleven) -- (twelve) -- (fifteen) -- cycle;

\draw[fill = blue, fill opacity = 0.1] (thirteen) -- (fourteen) -- (fifteen) -- cycle;

\draw[dashed, color=blue, opacity = 0.6] (1/2, 1/4, 0) -- (uniform);
\draw[dashed, color=blue, opacity = 0.6] (1/2, 0, 1/4) -- (uniform);
\draw[dashed, color=blue, opacity = 0.6](1/4, 0, 1/2) -- (uniform);
\draw[dashed, color=blue, opacity = 0.6] (0, 1/4, 1/2) -- (uniform);
\draw[dashed, color=blue, opacity = 0.6](0, 1/2, 1/4) -- (uniform);
\draw[dashed, color=blue, opacity = 0.6] (1/4, 1/2, 0) -- (uniform);

\end{scope}

%
%
%
%

\end{tikzpicture} & 




\begin{tikzpicture} [scale=\tikzfigscale, thick, tdplot_main_coords]
\coordinate (orig) at (0,0,0);

\coordinate (uniform) at (1/4,1/4,1/4);
\coordinate[label=below:${p_1}$] (h) at (3/4,0,0);
\coordinate[label=below:${p_3}$] (m) at (0,3/4,0);
\coordinate[label=right:${p_2}$] (l) at (0,0,3/4);

\coordinate (one) at (1/2, 1/4, 0);
\coordinate (two) at (1/4, 1/2, 0);
\coordinate (three) at (0, 1/2, 1/4);
\coordinate (four) at (0, 1/4, 1/2);
\coordinate (five) at (1/4, 0, 1/2);
\coordinate (six) at (1/2, 0, 1/4);
\coordinate (seven) at (1/3, 1/3, 1/12);
\coordinate (eight) at (1/12, 1/3, 1/3);
\coordinate (nine) at (1/3, 1/12, 1/3);

\draw[simplex] (h) -- (m) -- (l) -- (h);

\begin{scope}
\clip (h) -- (m) -- (l) -- (h);

\draw (h) -- (one) -- (six) -- cycle;
\draw (l) -- (four) -- (five) -- cycle;
\draw (m) -- (three) -- (two) -- cycle;

\draw (one) -- (two) -- (seven) -- cycle;
\draw (three) -- (four) -- (eight) -- cycle;
\draw (five) -- (six) -- (nine) -- cycle;

\draw[fill=blue,fill opacity = 0.1] (one) -- (seven) -- (nine) -- (six) -- cycle;
\draw[fill=blue,fill opacity = 0.1] (two) -- (three) -- (eight) -- (seven) -- cycle;
\draw[fill=blue,fill opacity = 0.1] (four) -- (five) -- (nine) -- (eight) -- cycle;

\draw[fill=blue,fill opacity = 0.1] (seven) -- (eight) -- (nine) -- cycle;

\draw[dashed, color=blue, opacity = 0.6] (1/2, 1/4, 0) -- (uniform);
\draw[dashed, color=blue, opacity = 0.6] (1/2, 0, 1/4) -- (uniform);
\draw[dashed, color=blue, opacity = 0.6](1/4, 0, 1/2) -- (uniform);
\draw[dashed, color=blue, opacity = 0.6] (0, 1/4, 1/2) -- (uniform);
\draw[dashed, color=blue, opacity = 0.6](0, 1/2, 1/4) -- (uniform);
\draw[dashed, color=blue, opacity = 0.6] (1/4, 1/2, 0) -- (uniform);

\end{scope}

\node (13center) at (3/8 - 1/48, 3/8 - 1/48, 1/24) {};
\node (13label) at (3/8, 3/8, -1/5) {{\tiny $1,3$}};
\draw[-{Latex[length=1mm, width=1mm]}, red] (13label) -- (13center.center);

\node (emptycenter) at (.23 * .75, .38 * .75,.38 * .75) {};
\node (emptylabel) at (-1/8, 3/8, 3/8) {{\tiny $\emptyset$}};
\draw[-{Latex[length=1mm, width=1mm]}, red] (emptylabel) -- (emptycenter.center);

\node (2center) at (.2 * .75, .2 * .75,.6 * .75) {};
\node (2label) at (3/8, -1/8, 5/8) {{\tiny $2$}};
\draw[-{Latex[length=1mm, width=1mm]}, red] (2label) -- (2center.center);

\node (14center) at (.8 * .75, .1 * .75,.1 * .75) {};
\node (14label) at (5/8, -1/8, 3/8) {{\tiny $14$}};
\draw[-{Latex[length=1mm, width=1mm]}, red] (14label) -- (14center.center);

\end{tikzpicture}

 &

\begin{tikzpicture} [scale=\tikzfigscale, thick, tdplot_main_coords]
\coordinate (orig) at (0,0,0);

\coordinate (uniform) at (1/4,1/4,1/4);
\coordinate[label=below:${p_1}$] (h) at (3/4,0,0);
\coordinate[label=below:${p_3}$] (m) at (0,3/4,0);
\coordinate[label=right:${p_2}$] (l) at (0,0,3/4);

\draw[simplex] (h) -- (m) -- (l) -- (h);

\begin{scope}
\clip (h) -- (m) -- (l) -- (h);

\draw[opacity=0.9] (2/3, 1/3, 0) -- (uniform);
\draw[opacity=0.9] (2/3, 0, 1/3) -- (uniform);
\draw[opacity=0.9] (1/3, 0, 2/3) -- (uniform);
\draw[opacity=0.9] (0, 1/3, 2/3) -- (uniform);
\draw[opacity=0.9] (0, 2/3, 1/3) -- (uniform);
\draw[opacity=0.9] (1/3, 2/3, 0) -- (uniform);

\end{scope}

\node (14center) at (.66 * .75, 0.16 * .75, 0.16 * 0.75) {};
\node (14label) at (5/8, -1/4,2/8) {{\tiny $14$}};
\draw[-{Latex[length=1mm, width=1mm]}, red] (14label) -- (14center.center);

\node (23center) at (.1 * .75, 0.45 * .75, 0.45 * 0.75) {};
\node (23label) at (-1/8, 1/2,1/2) {{\tiny $23$}};
\draw[-{Latex[length=1mm, width=1mm]}, red] (23label) -- (23center.center);

\end{tikzpicture}
 \\ 
		\rotatebox[origin=c]{90}{$k = 3$\hspace*{-2.5cm}} & 




\begin{tikzpicture} [scale=\tikzfigscale, thick, tdplot_main_coords]
\coordinate (orig) at (0,0,0);

\coordinate (uniform) at (1/4,1/4,1/4);
\coordinate[label=below:${p_1}$] (h) at (3/4,0,0);
\coordinate[label=below:${p_3}$] (m) at (0,3/4,0);
\coordinate[label=right:${p_2}$] (l) at (0,0,3/4);

\coordinate (one) at (1/2, 1/20, 3/4 - 1/2 - 1/20);
\coordinate (two) at (1/2, 3/4 - 1/2 - 1/20,1/20);
\coordinate (three) at (3/4 - 1/2 - 1/20,1/2,1/20);
\coordinate (four) at (1/20,1/2, 3/4 - 1/2 - 1/20);
\coordinate (five) at (1/20,3/4 - 1/2 - 1/20, 1/2);
\coordinate (six) at (3/4 - 1/2 - 1/20,1/20, 1/2);
\coordinate (seven) at (1/2, 0, 1/4);
\coordinate (eight) at (1/2, 1/4, 0);
\coordinate (nine) at (1/4, 1/2, 0);
\coordinate (ten) at (0, 1/2, 1/4);
\coordinate (eleven) at (0, 1/4, 1/2);
\coordinate (twelve) at (1/4, 0, 1/2);
\coordinate (thirteen) at (1/3, 1/12, 1/3);
\coordinate (fourteen) at (1/3, 1/3, 1/12);
\coordinate (fifteen) at (1/12, 1/3, 1/3);
\coordinate (sixteen) at (1/6, 1/4, 1/3);
\coordinate (seventeen) at (1/4, 1/6, 1/3);
\coordinate (eighteen) at (1/3, 1/6, 1/4);
\coordinate (nineteen) at (1/3, 1/4, 1/6);
\coordinate (twenty) at (1/4, 1/3, 1/6);
\coordinate (twentyone) at (1/6, 1/3, 1/4);
\coordinate(twentytwo) at (1/4, 1/4, 1/4);

\draw[simplex] (h) -- (m) -- (l) -- (h);

\begin{scope}
\clip (h) -- (m) -- (l) -- (h);

\draw[fill=blue,fill opacity = 0.1] (h) -- (one) -- (two) -- cycle;
\draw[fill=blue,fill opacity = 0.1] (m) -- (three) -- (four) -- cycle;
\draw[fill=blue,fill opacity = 0.1] (l) -- (five) -- (six) -- cycle;

\draw (h) -- (one) -- (seven) -- cycle;
\draw (h) -- (two) -- (eight) -- cycle;
\draw (m) -- (three) -- (nine) -- cycle;
\draw (m) -- (four) -- (ten) -- cycle;
\draw (l) -- (five) -- (eleven) -- cycle;
\draw (l) -- (six) -- (twelve) -- cycle;

\draw (seven) -- (twelve) -- (thirteen) -- cycle;
\draw (eight) -- (nine) -- (fourteen) -- cycle;
\draw (ten) -- (eleven) -- (fifteen) -- cycle;

\draw (five) -- (eleven) -- (fifteen) -- (sixteen) -- cycle;
\draw (six) -- (twelve) -- (thirteen) -- (seventeen) -- cycle;
\draw (five) -- (six) -- (seventeen) -- (sixteen) -- cycle;
\draw (one) -- (seven) -- (thirteen) -- (eighteen) -- cycle;
\draw[fill=blue,fill opacity = 0.1] (one) -- (two) -- (nineteen) -- (eighteen) -- cycle;
\draw (two) -- (eight) -- (fourteen) -- (nineteen) -- cycle;
\draw (three) -- (nine) -- (fourteen) -- (twenty) -- cycle;
\draw[fill=blue,fill opacity = 0.1] (three) -- (four) -- (twentyone) -- (twenty) -- cycle;
\draw (four) -- (ten) -- (fifteen) -- (twentyone) -- cycle;
\draw[fill=blue,fill opacity = 0.1] (five) -- (six) -- (seventeen) -- (sixteen) -- cycle;

\draw (fourteen) -- (nineteen) -- (twentytwo) -- (twenty) -- cycle;
\draw[fill=blue,fill opacity = 0.1] (eighteen) -- (nineteen) -- (twentytwo) -- cycle;
\draw (thirteen) -- (eighteen) -- (twentytwo) -- (seventeen) -- cycle;
\draw[fill=blue,fill opacity = 0.1] (sixteen) -- (seventeen) -- (twentytwo) -- cycle;
\draw (fifteen) -- (sixteen) -- (twentytwo) -- (twentyone) -- cycle;
\draw[fill=blue,fill opacity = 0.1] (twenty) -- (twentyone) -- (twentytwo) -- cycle;

\draw[dashed, color=blue, opacity = 0.6] (h) -- (twentytwo);
\draw[dashed, color=blue, opacity = 0.6] (m) -- (twentytwo);
\draw[dashed, color=blue, opacity = 0.6] (l) -- (twentytwo);

\end{scope}

\end{tikzpicture}

 & 



\begin{tikzpicture} [scale=\tikzfigscale, thick, tdplot_main_coords]
\coordinate (orig) at (0,0,0);

\coordinate (uniform) at (1/4,1/4,1/4);
\coordinate[label=below:${p_1}$] (h) at (3/4,0,0);
\coordinate[label=below:${p_3}$] (m) at (0,3/4,0);
\coordinate[label=right:${p_2}$] (l) at (0,0,3/4);

\coordinate (one) at (3/4 - 1/8, 0, 1/8);
\coordinate (two) at (3/4 - 1/8, 1/8, 0);
\coordinate (three) at (1/8, 3/4 - 1/8, 0);
\coordinate (four) at (0, 3/4 - 1/8, 1/8);
\coordinate (five) at (0, 1/8, 3/4 - 1/8);
\coordinate (six) at (1/8, 0, 3/4 - 1/8);
\coordinate (seven) at (3/4 - 1/8, 1/40,1/8 - 1/40);
\coordinate (eight) at (3/4 - 1/8, 1/8 - 1/40, 1/40);
\coordinate (nine) at (1/8 - 1/40, 3/4 - 1/8, 1/40);
\coordinate (ten) at (1/40, 3/4 - 1/8, 1/8 - 1/40);
\coordinate (eleven) at (1/40, 1/8 - 1/40, 3/4 - 1/8);
\coordinate (twelve) at (1/8 - 1/40, 1/40, 3/4 - 1/8);
\coordinate (thirteen) at (3/4 - 1/3, 1/3, 0);
\coordinate (fourteen) at (1/3, 3/4 - 1/3, 0);
\coordinate (fifteen) at (0,3/4 - 1/3, 1/3);
\coordinate (sixteen) at (0,1/3, 3/4 - 1/3);
\coordinate (seventeen) at (1/3, 0, 3/4 - 1/3);
\coordinate (eighteen) at (3/4 - 1/3,0, 1/3);
\coordinate (nineteen) at (1/2,1/8, 1/8);
\coordinate (twenty) at (1/3, 1/3, 3/4 - 2/3);
\coordinate (twentyone) at (1/8, 1/2, 1/8);
\coordinate (twentytwo) at (3/4 - 2/3, 1/3, 1/3);
\coordinate (twentythree) at (1/8, 1/8, 1/2);
\coordinate (twentyfour) at (1/3, 3/4 - 2/3, 1/3);
\coordinate (twentyfive) at (1/3, 1/4, 1/2 - 1/3);
\coordinate (twentysix) at (1/4, 1/3, 1/2 - 1/3);
\coordinate (twentyseven) at (1/2 - 1/3, 1/3, 1/4);
\coordinate (twentyeight) at (1/2 - 1/3, 1/4, 1/3);
\coordinate (twentynine) at (1/4, 1/2-1/3, 1/3);
\coordinate (thirty) at (1/3, 1/2 - 1/3, 1/4);

\draw[simplex] (h) -- (m) -- (l) -- (h);

\begin{scope}
\clip (h) -- (m) -- (l) -- (h);

\draw (h) -- (one) -- (seven) -- cycle;
\draw (h) -- (two) -- (eight) -- cycle;
\draw (m) -- (three) -- (nine) -- cycle;
\draw (m) -- (four) -- (ten) -- cycle;
\draw (l) -- (five) -- (eleven) -- cycle;
\draw (l) -- (six) -- (twelve) -- cycle;

\draw[fill=blue,fill opacity = 0.1] (h) -- (seven) -- (nineteen) -- (eight) -- cycle;
\draw[fill=blue,fill opacity = 0.1] (m) -- (nine) -- (twentyone) -- (ten) -- cycle;
\draw[fill=blue,fill opacity = 0.1] (l) -- (eleven) -- (twentythree) -- (twelve) -- cycle;

\draw (one) -- (seven) -- (nineteen) -- (thirty) -- (twentyfour) -- (eighteen) -- cycle;
\draw (two) -- (eight) -- (nineteen) -- (twentyfive) -- (twenty) -- (thirteen); 
\draw (three) -- (nine) -- (twentyone) -- (twentysix) -- (twenty) -- (fourteen);
\draw (four) -- (ten) -- (twentyone) -- (twentyseven) -- (twentytwo) -- (fifteen) -- cycle;
\draw (five) -- (eleven) -- (twentythree) -- (twentyeight) -- (twentytwo) -- (sixteen) -- cycle;
\draw (six) -- (twelve) -- (twentythree) -- (twentynine) -- (twentyfour) -- (seventeen) -- cycle;

\draw[fill=blue,fill opacity = 0.1] (nineteen) -- (twentyfive) -- (uniform) -- (thirty);
\draw[fill=blue,fill opacity = 0.1] (twentyone) -- (twentyseven) -- (uniform) -- (twentysix);
\draw[fill=blue,fill opacity = 0.1] (twentythree) -- (twentyeight) -- (uniform) -- (twentynine) -- cycle;

\draw (twenty) -- (twentysix) -- (uniform) -- (twentyfive) -- cycle;
\draw (twentytwo) -- (twentyseven) -- (uniform) -- (twentyeight) -- cycle;
\draw (twentyfour) -- (twentynine) -- (uniform) -- (thirty) -- cycle;

\draw[dashed, color=blue, opacity = 0.6] (h) -- (uniform);
\draw[dashed, color=blue, opacity = 0.6] (m) -- (uniform);
\draw[dashed, color=blue, opacity = 0.6] (l) -- (uniform);



\end{scope}

%
%
%
%

\end{tikzpicture}

 & 



\begin{tikzpicture} [scale=\tikzfigscale, thick, tdplot_main_coords]
\coordinate (orig) at (0,0,0);

\coordinate (uniform) at (1/4,1/4,1/4);
\coordinate[label=below:${p_1}$] (h) at (3/4,0,0);
\coordinate[label=below:${p_3}$] (m) at (0,3/4,0);
\coordinate[label=right:${p_2}$] (l) at (0,0,3/4);

\draw[simplex] (h) -- (m) -- (l) -- (h);

\begin{scope}
\clip (h) -- (m) -- (l) -- (h);

\draw (h) -- (uniform) -- (m) -- cycle;
\draw (h) -- (uniform) -- (l) -- cycle;
\draw (m) -- (uniform) -- (l) -- cycle;

\draw[dashed, color=blue, opacity = 0.6] (h) -- (uniform) -- (m) -- (uniform) -- (l);

\end{scope}

\node (134center) at (.45 * .75, 0.45 * .75, 0.1 * 0.75) {};
\node (134label) at (3/8, 3/8, -1/5) {{\tiny $134$}};
\draw[-{Latex[length=1mm, width=1mm]}, red] (134label) -- (134center.center);

\node (234center) at (.1 * .75, 0.45 * .75, 0.45 * 0.75) {};
\node (234label) at (-1/4, 3/8, 3/8) {{\tiny $234$}};
\draw[-{Latex[length=1mm, width=1mm]}, red] (234label) -- (234center.center);

\node (124center) at (.45 * .75, 0.1 * .75, 0.45 * 0.75) {};
\node (124label) at (3/8, -1/4,3/8) {{\tiny $124$}};
\draw[-{Latex[length=1mm, width=1mm]}, red] (124label) -- (124center.center);

\end{tikzpicture}

 &

\begin{tikzpicture} [scale=\tikzfigscale, thick, tdplot_main_coords]
\coordinate (orig) at (0,0,0);

\coordinate (uniform) at (1/4,1/4,1/4);
\coordinate[label=below:${p_1}$] (h) at (3/4,0,0);
\coordinate[label=below:${p_3}$] (m) at (0,3/4,0);
\coordinate[label=right:${p_2}$] (l) at (0,0,3/4);

\draw[simplex] (h) -- (m) -- (l) -- (h);

\begin{scope}
\clip (h) -- (m) -- (l) -- (h);

\draw (h) -- (uniform) -- (m) -- cycle;
\draw (h) -- (uniform) -- (l) -- cycle;
\draw (m) -- (uniform) -- (l) -- cycle;

\draw[dashed, color=blue, opacity = 0.6] (h) -- (uniform) -- (m) -- (uniform) -- (l);

\end{scope}

\node (134center) at (.45 * .75, 0.45 * .75, 0.1 * 0.75) {};
\node (134label) at (3/8, 3/8, -1/5) {{\tiny $134$}};
\draw[-{Latex[length=1mm, width=1mm]}, red] (134label) -- (134center.center);

\node (234center) at (.1 * .75, 0.45 * .75, 0.45 * 0.75) {};
\node (234label) at (-1/4, 3/8, 3/8) {{\tiny $234$}};
\draw[-{Latex[length=1mm, width=1mm]}, red] (234label) -- (234center.center);

\node (124center) at (.45 * .75, 0.1 * .75, 0.45 * 0.75) {};
\node (124label) at (3/8, -1/4,3/8) {{\tiny $124$}};
\draw[-{Latex[length=1mm, width=1mm]}, red] (124label) -- (124center.center);

\end{tikzpicture}

 \\ 
	\end{tabular}
	\caption{
		Visualizations of the properties elicited by the losses (embedded by) $\Li{2}$, $\Li{3}$, $\Li{4}$ studied by \citet{yang2018consistency} and \citet{finocchiaro2022consistenttopk}, and $L^k$ in eq.~\eqref{eq:topk-embedding}.
    We take $n=4$ and $k \in \{2,3\}$, and visualize in 2 dimensions by fixing $p_4 = 1/4$. 
		The blue-filled regions are cells of the surrogate property which cross the dashed blue lines of the target property, exhibiting inconsistency (see Figure~\ref{fig:abstain-intuition-inconsistent}).
    Intuitively, the inconsistency arises from ambiguity in the top-$k$ elements of the optimal surrogate report.  
	}
	\label{tab:loss-slices}
\end{table*}

For example, one of the surrogates is $\Li {4}(u)_y = \left(1 - u_y + \frac 1 k \sum_{i=1}^k (u_{\setminus y})_{[i]}\right)_+$, where $u_{[i]}$ denotes the $i$th largest element of $u \in \reals^n$.
The authors show that $\Li{4}$ embeds $\ell^{(4)}(T)_y = \tfrac {k+1} {k+1-|T|} \ones\{y\notin T\}$, where $T$ is a set of at most $k$ labels.
These embedded losses may therefore be useful in top-$k$ settings where choosing smaller sets may have some benefit, such as a search engine that can use unused space for advertisements.
Using the losses each proposed surrogate embeds, using the same technique from Figure~\ref{fig:abstain-intuition-inconsistent}, the authors go on to derive constraints on the label distributions under which the proposed surrogates are actually consistent for top-$k$ classification; these constraints are tighter than previous constraints~\citep{yang2018consistency}.

Beyond analyzing the previously proposed surrogates, \citeauthor{finocchiaro2022consistenttopk} also use our framework to derive the first consistent polyhedral surrogate for $\elltopk$,
\begin{align}\label{eq:topk-embedding}
L^k(u)_y &= \max \left(u_{[1]}, \max_{m \in \{k+1, \ldots, n\}} \left[ 1 - \frac k m + \frac 1 m \sum_{i=1}^m u_{[i]}\right] \right)- u_y~.
\end{align}
That is, they show that a hinge-like surrogate does exist which is both convex and consistent.
In light of our framework, this fact is unsurprising: Theorems~\ref{thm:embed-poly-main} and~\ref{thm:link-main} imply that \emph{every} discrete loss has a consistent polyhedral surrogate.
This new surrogate $L^k$ is given directly by the construction from the proof of Theorem~\ref{thm:discrete-loss-poly-embeddable} and applying Theorem~\ref{thm:link-main} to obtain consistency.
While Theorem~\ref{thm:link-main} guarantees the existence of some consistent link function, the authors further ask whether the canonical argmax link function $\psi^k$, which returns the $k$ largest elements of $u$, is calibrated.
They indeed confirm its consistency using our framework, showing that $\psi^k$ is $\epsilon$-separated for $L^k$ and $\elltopk$, for any $\epsilon \leq \frac 1 {2n}$ \citep[Theorem 4.4]{finocchiaro2022consistenttopk}.

\section{Additional Structure of Embeddings}
\label{sec:min-rep-sets}

We have shown in \S~\ref{sec:poly-loss-embed} a close connection between embeddings and polyhedral losses.
Here we go beyond polyhedral losses, showing a more general necessary condition for an embedding: a surrogate embeds a discrete loss if and only if it has a polyhedral Bayes risk, or equivalently, a finite representative set (Lemma~\ref{lem:X}).
This result implies that the embedding condition simplifies to matching Bayes risks (Proposition~\ref{prop:embed-bayes-risks}).
We also use this result to understand deeper structure of embeddings, and the geometry of the underlying properties. 
In particular, we study a natural notion of a ``trimmed'' loss function (Definition~\ref{def:trim-loss}), and connect this notion to tight embeddings, and to non-redundancy from property elicitation (Proposition~\ref{prop:embed-iff-trims-equal}).

\subsection{Structure of polyhedral Bayes risks}

While we have focused on polyhedral losses thus far, many of our results extend to losses with polyhedral Bayes risks, a strictly weaker condition.
(We say a concave function is polyhedral if its negation is a polyhedral convex function.)
To see that every polyhedral loss has a polyhedral Bayes risk, recall that Theorem~\ref{thm:poly-embeds-discrete} constructs a finite representative set $\Sc$ for any polyhedral loss $L$, and thus $\risk{L} = \risk{L|_\Sc}$ by Lemma~\ref{lem:loss-restrict}, which is polyhedral.
Conversely, however, a Bayes risk may be polyhedral even if the loss itself is not.
For example, a modified hinge loss $L(r)_y = \max(r^2-1,1-ry)$
as shown in Figure~\ref{fig:modified-hinge}, which matches hinge loss on the interval $[-1,1]$ but is strictly convex outside the interval $[-2,2]$, still embeds twice 0-1 loss.

Much of our embedding framework relies on the existence of finite representative sets.
Our main structural result is that a minimizable loss has a finite representative sets if and only if its Bayes risk is polyhedral.
The proof looks at the facets (full-dimensional faces) of the Bayes risk, and argues that each facet is generated by the loss at a particular report, and the (finite) set of these reports is representative.
Along the way, we identify several other useful facts deriving from this same geometry; for example, a discrete loss tightly embedded by a loss are unique up to relabeling, any set-wise minimal representative set must be minimum in cardinality, and the level sets of the corresponding property are unique and full-dimensional.
Together, these facts form Lemma~\ref{lem:X}, which we use throughout this section.
See \S~\ref{app:polyhedra} for ommitted proofs.

\begin{figure}[t]
	\begin{minipage}{0.47\linewidth}
		\centering
		\includegraphics[width=0.95\linewidth]{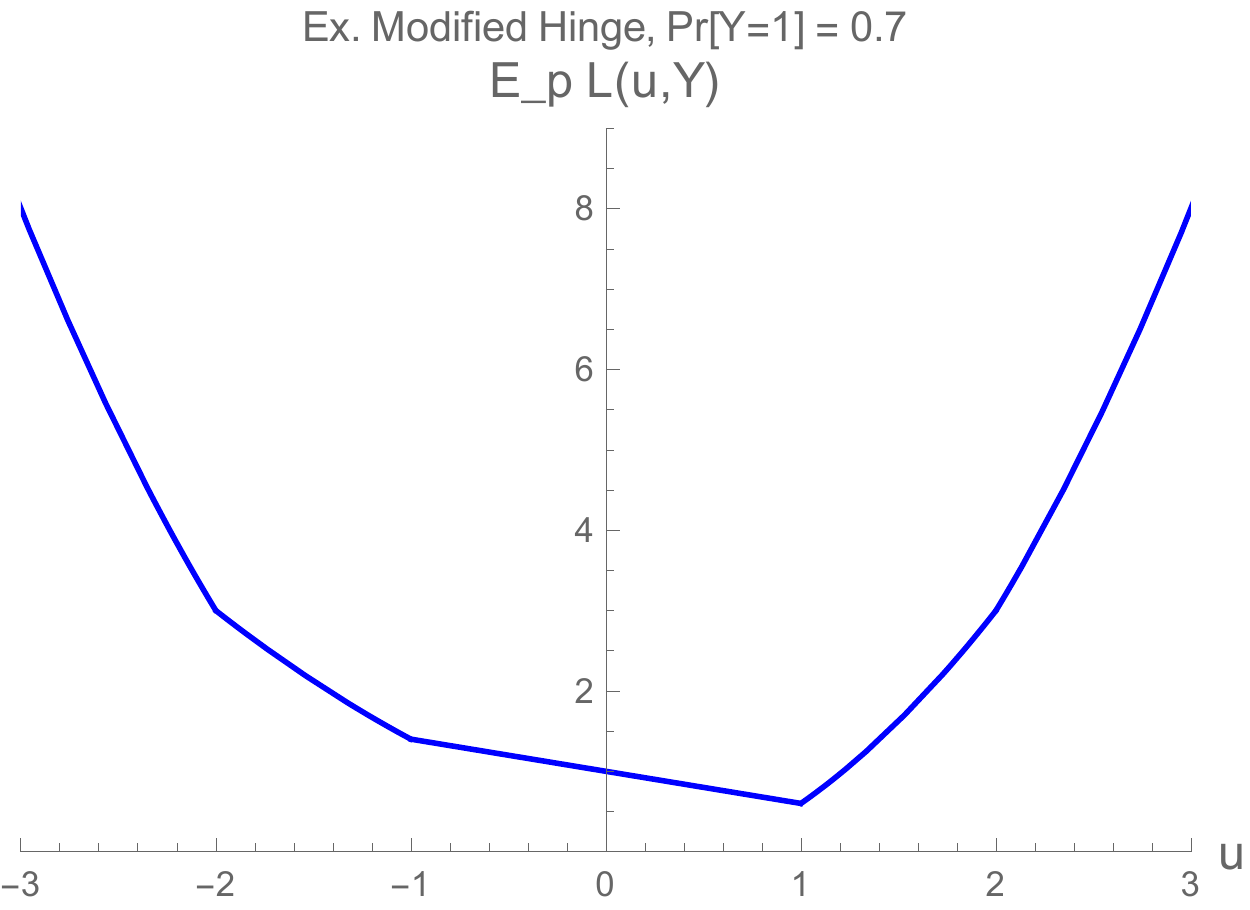}
	\end{minipage}
	\hfill
	\begin{minipage}{0.47\linewidth}
		\centering		\includegraphics[width=0.95\linewidth]{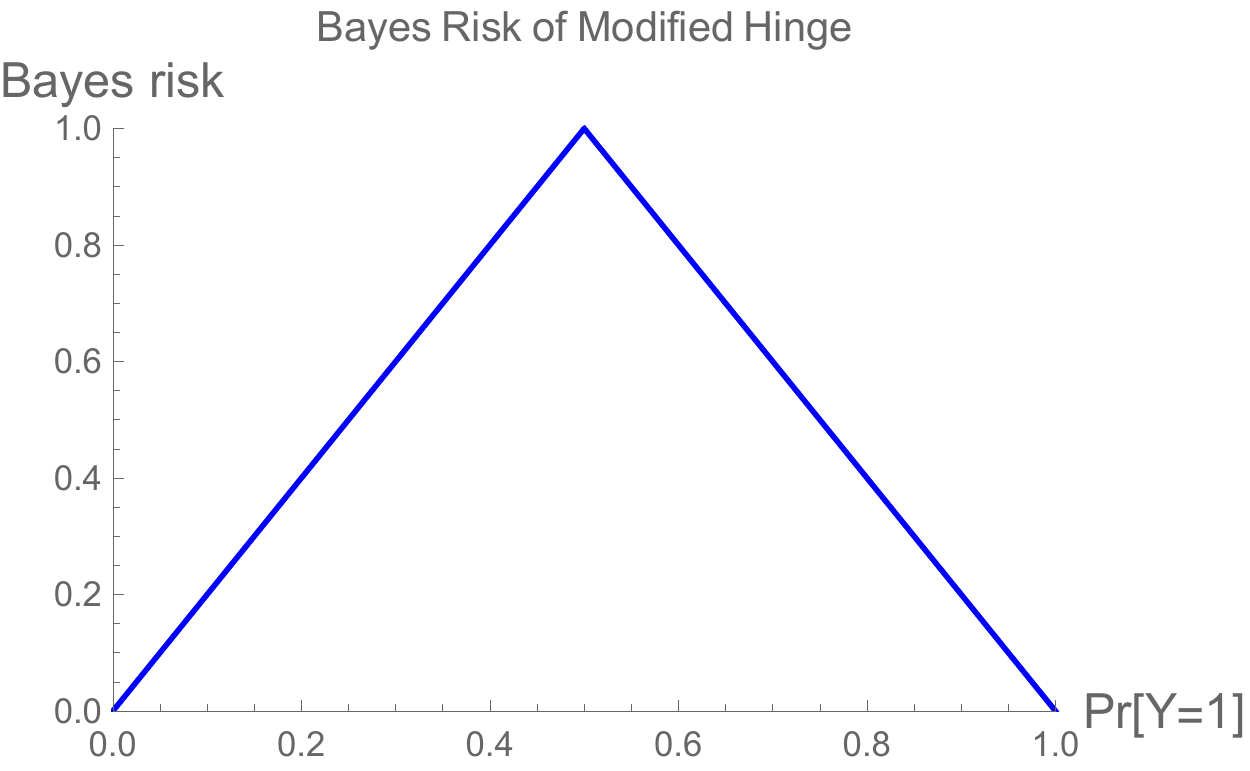}
	\end{minipage}
	\caption{(L) Expected modified hinge loss for fixed distribution; (R) Bayes risk of modified hinge still matches the Bayes risk of hinge.}
	\label{fig:modified-hinge}
\end{figure}

\begin{restatable}{lemma}{lemmaX}\label{lem:X}
  Let $L: \R \to \reals^\Y_+$ be a minimizable loss with a polyhedral Bayes risk $\risk L$.
  Then $L$ has a finite representative set.
  Furthermore, letting $\Gamma = \prop{L}$, there exist finite sets
  $\V \subseteq \reals^\Y_+$ and
  $\Theta = \{\theta_v \subseteq \simplex \mid v\in\V\}$,
  both uniquely determined by $\risk{L}$ alone,
  such that
  \begin{enumerate}
  \item A set $\R'\subseteq\R$ is representative if and only if $\V \subseteq L(\R')$.\label{item:X-rep-V}
  \item A set $\R'\subseteq\R$ is minimum representative if and only if $L(\R') = \V$.\label{item:X-min-V}
  \item A set $\R'\subseteq\R$ is representative if and only if $\Theta \subseteq \{\Gamma_r \mid r \in \R'\}$.\label{item:X-rep-Theta}
  \item A set $\R'\subseteq\R$ is minimum representative if and only if $\{\Gamma_r \mid r \in \R'\} = \Theta$.\label{item:X-min-Theta}
  \item Every representative set for $L$ contains a minimum representative set for $L$.\label{item:X-rep-contain-min}
  \item The set of full-dimensional level sets of $\Gamma$ is exactly $\Theta$.\label{item:X-full-dim}
  \item For any $r \in \R$, there exists $\theta \in \Theta$ such that $\Gamma_r \subseteq \theta$.\label{item:X-redundant}
  \item $L$ tightly embeds $\ell:\R'\to\reals^\Y_+$ if and only if $\ell$ is injective and $\ell(\R') = \V$.\label{item:X-tight-embed}
  \end{enumerate}
\end{restatable}

As a finite representative set implies a polyhedral Bayes risk by Lemma~\ref{lem:loss-restrict}, Lemma~\ref{lem:X} shows that polyhedral Bayes risks are equivalent to having finite representative sets, which in turn gives an embedding by
Proposition~\ref{prop:representative-embeds-restriction}.
\begin{corollary}\label{cor:poly-risk-fin-rep}
  The following are equivalent for any minimizable loss $L:\R\to\reals^\Y_+$.
  \begin{enumerate}
  \item $\risk{L}$ is polyhedral.
  \item $L$ has a finite representative set.
  \item $L$ embeds a discrete loss.
  \end{enumerate}
\end{corollary}
From Corollary~\ref{cor:poly-risk-fin-rep}, $L$ having a finite representative set is an equivalent condition to $L$ being minimizable and $\risk{L}$ being polyhedral.
(Recall that having a finite representative set already implies minimizability.)
As it is also a more succinct condition, we will use the former in the sequel.
In particular, the implications of Lemma~\ref{lem:X} follow whenever $L$ has a finite representative set.

\subsection{Equivalent condition: matching Bayes risks}\label{subsec:match-BR}

Lemma~\ref{lem:X} leads to another appealing equivalent condition to an embedding: a surrogate embeds a discrete loss if and only if their Bayes risks match.
The proof follows by mapping the two conditions of an embedding onto the geometric structure revealed by Lemma~\ref{lem:X}: (ii) if the properties have the same level sets, the Bayes risks have the same projections onto $\simplex$, and (i) if the loss values match, then the slopes of the Bayes risk must be identical as well.

\begin{proposition}\label{prop:embed-bayes-risks}
  Let discrete loss $\ell$ and minimizable loss $L$ be given.
  Then $L$ embeds $\ell$ if and only if $\risk{L}=\risk{\ell}$.
\end{proposition}
\begin{proof}
  Define $\Gamma = \prop{L}$ and $\gamma = \prop{\ell}$.
  
  Suppose $L$ embeds $\ell$, so we have some $\Sc\subseteq \R$ which is representative for $\ell$ and an embedding $\varphi:\Sc\to\reals^d$; take $\U := \varphi(\Sc)$.
  Since $\Sc$ is representative for $\ell$, by embedding condition (ii) we have $\{\gamma_s \mid s\in\Sc\} = \{\Gamma_u \mid u\in\U\}$, so $\U$ is representative for $L$.
  By Lemma~\ref{lem:loss-restrict}, we have $\risk{\ell} = \risk{\ell|_{\Sc}}$ and $\risk{L} = \risk{L|_{\U}}$.
  As $L(\varphi(\cdot)) = \ell(\cdot)$ by embedding condition (i), for all $p\in\simplex$ we have
  \begin{equation*}
    \risk{\ell}(p) = \risk{\ell|_\Sc}(p) = \min_{r \in \Sc}\inprod{p}{\ell(r)} = \min_{r \in \Sc}\inprod{p}{L(\varphi(r))} = \min_{u \in \U}\inprod{p}{L(u)} = \risk{L|_\U}(p) = \risk{L}(p)~.
  \end{equation*}
  
	For the reverse implication, assume $\risk{L} = \risk{\ell}$, which are polyhedral functions as $\ell$ is discrete.
  From Lemma~\ref{lem:X}(\ref{item:X-min-V}), we have some set $\V\subseteq\reals^\Y_+$ and minimum representative sets $\R^* \subseteq \R$ and $\U^* \subseteq \U$, for $\ell$ and $L$ respectively, such that $\ell(\R^*) = \V = L(\U^*)$.
  As $\R^*$ and $\U^*$ are miniumum, they cannot repeat loss vectors, and thus $|\R^*|=|\ell(\R^*)|$ and $|L(\U^*)|=|\U^*|$.
  We conclude that $\R^*$ and $\U^*$ are both in bijection with $\V$.
  The map $\varphi :\R^* \to \reals^d$, given by $\varphi(r) = u \in \U^*$ where $\ell(r) = L(u)$, is therefore well-defined.
  Condition (i) of an embedding is immediate.
  From Proposition~\ref{prop:representative-embeds-restriction}, $\ell$ embeds $\ell|_{\R^*}$ and $L$ embeds $L|_{\U^*}$, both via the identity embedding.
  Using condition (ii) from both embeddings, for all $p\in\simplex$ and $r\in\R^*$, we have
  \begin{equation*}
    r \in \gamma(p) \iff r \in \prop{\ell|_{\R^*}}(p) \iff \varphi(r) \in \prop{L|_{\U^*}}(p)
    \iff \varphi(r) \in \prop{L}(p)~,
  \end{equation*}
  giving condition (ii).
\end{proof}

We use this fact in the proof of Theorem~\ref{thm:discrete-loss-poly-embeddable} to show that every discrete loss is embedded by some polyhedral surrogate.
See Figure~\ref{fig:bayes-risks-01} for an illustration.

\subsection{Trimming a loss}
\label{sec:trim}

Central to the structural results in Lemma~\ref{lem:X} is the existence of a canonical set of loss vectors $\V$ which match the loss vectors of any minimum representative set.
This fact may seem surprising when one considers that losses may have many mimimum representative sets.
For example, consider hinge loss with a spurious extra dimension, i.e., $L:\reals^2\to\reals^\Y$, $L((r_1, r_2))_y = \max(0,1-r_1y)$ for $\Y = \{-1,+1\}$.
Here the minimum representative sets are exactly the two-element sets of the form $\{(-1,a),(1,b)\}$ for any $a,b\in\reals$. 
Lemma~\ref{lem:X}(\ref{item:X-min-V}) states that, while the minimum representative set is not unique, its loss vectors are.

Motivated by this observation, let us define the ``trim'' of a loss to be this unique set $\V$ of loss vectors induced by any minimum representative set, which again is well-defined by Lemma~\ref{lem:X}(\ref{item:X-min-V}).
\begin{definition}[Trim]\label{def:trim-loss}
  Given a loss $L:\R \to \reals_+^\Y$ with a finite representative set, we define $\trimcover(L) = \{L(r) \mid r \in \R^*\}$ given any minimum representative set $\R^*$ for $L$.
\end{definition}

Using this notion of trimming a loss, we can again recast our embedding condition: a loss embeds another if and only if they induce the same loss vectors, or have the same $\trimcover$.

\begin{proposition}\label{prop:embed-iff-trims-equal}
  Let $L:\reals^d\to\reals^\Y_+$ have a finite representative set, and let $\ell:\R\to\reals^\Y_+$ be a discrete loss.
  Then $L$ embeds $\ell$ if and only if $\trimcover(L) = \trimcover(\ell)$.
  Furthermore, $L$ tightly embeds $\ell$ if and only if $\ell$ is injective and $\trimcover(L) = \ell(\R)$.
\end{proposition}
\begin{proof}
  As $L$ has a finite representative set, it is minimizable.
  Proposition~\ref{prop:embed-bayes-risks} gives $L$ embeds $\ell$ if and only if $\risk L = \risk \ell$.
  If $\risk L = \risk \ell$, Lemma~\ref{lem:X}(\ref{item:X-min-V}) gives $\trim(L) = \trim(\ell)$.
  For the converse, suppose $\trim(L) = \trim(\ell) =: \V$.
  Define the discrete loss $\ell_\trim : \V \to \V, v\mapsto v$.
  Then $\ell_\trim$ is injective and $\ell_\trim(\V) = \V$, so from Lemma~\ref{lem:X}(\ref{item:X-tight-embed}), both $L$ and $\ell$ tightly embed $\ell_\trim$.
  We conclude $\risk L = \risk{\ell_\trim} = \risk \ell$ from Proposition~\ref{prop:embed-bayes-risks}.
  The second statement also follows directly from Lemma~\ref{lem:X}(\ref{item:X-tight-embed}).
\end{proof}

In a strong sonse, the trim operation reduces a loss to its core: the unique minimal set of loss vectors that drive its statistical behavior.
One can therefore think of designing consistent convex surrogates as trying to ``fill out'' this minimal set with additional loss vectors so that one attains convexity while keeping trim the same.

\subsection{Minimum representative sets and non-redundancy}
\label{sec:min-rep-sets-trim}

The condition that a representative set be minimum implies that one has identified exactly the ``active'' reports of a loss, in some sense.
We now relate this condition to another natural notion from the property elicitation literature: non-redundancy~\cite{frongillo2014general,lambert2018elicitation}.
Intuitively, a loss is non-redundant if no report is weakly dominated by another report.

\begin{definition}[Non-redundancy]\label{def:nonredundant}
  A loss $L : \R \to \reals^\Y_+$ eliciting $\Gamma:\simplex \toto \R$ is \emph{redundant} if there are reports $r, r' \in \R$ with $r \neq r'$ such that $\Gamma_r \subseteq \Gamma_{r'}$, and \emph{non-redundant} otherwise.
\end{definition}

From the structural result of Lemma~\ref{lem:X}, we can see that in fact these two notions are equivalent when $L$ has a polyhedral Bayes risk.
\begin{proposition}\label{prop:tfae-min-rep-nonredundant}
  Let $L:\R\to\reals^\Y_+$ have a finite representative set $\R'$.
  Then $\R'$ is a minimum representative set for $L$ if and only if $L|_{\R'}$ is non-redundant.
\end{proposition}
\begin{proof}
  Let $\Gamma = \prop{L}$.
  Suppose first that $L|_{\R'}$ is redundant.
  Then there exist $r,r' \in \R'$ such that $\Gamma_r \subseteq \Gamma_{r'}$.
  Thus, for all $p \in \Gamma_r$, we have $\{r, r'\} \subseteq \Gamma(p)$.
  Therefore $\R' \setminus \{r\}$ still a representative set, so $\R'$ is not minimum.

  Now suppose $L|_{\R'}$ is non-redundant.
  As $\R'$ is a representative set, Lemma~\ref{lem:X}(\ref{item:X-rep-contain-min}) gives some minimum representative set $\Sc \subseteq \R'$.
  Suppose we had some $r \in \R' \setminus \Sc$.
  Now Lemma~\ref{lem:X}(\ref{item:X-min-Theta},\ref{item:X-redundant}) gives some $s\in\Sc$ such that $\Gamma_r \subseteq \Gamma_s$, which contradicts $L|_{\R'}$ being non-redundant.
  We conclude $L(\Sc)=L(\R')$, meaning $\R'$ is a minimum representative set.
\end{proof}

\begin{corollary}\label{cor:tight-embed-min-rep}
  Let loss $L:\R\to\reals^\Y_+$ with finite representative set $\R'$ be given.
  Then $L$ tightly embeds $L|_{\R'}$ if and only if $L|_{\R'}$ is non-redundant.
\end{corollary}

In fact, we can show something stronger: the reports in minimum representative sets are precisely those which are not strictly redundant.
To formalize this statement, given $\Gamma : \simplex \toto \R$, let $\red(\Gamma) := \{r\in\R \mid \exists r'\in\R,\; \Gamma_r \subsetneq \Gamma_{r'}\}$ be the set of strictly redundant reports.
Similarly, for minimizable $L$, let $\red(L) := \red(\prop L)$.

\begin{proposition}
  Let $L : \R \to \reals^\Y_+$ have a finite representative set.
  Let $\R'$ be the union of all minimum representative sets for $L$.
  Then $\R' = \R \setminus \red(L)$.
\end{proposition}

\begin{proof}
  Let $\Gamma = \prop L$.
  Let $\Sc$ be a minimum representative set for $L$, and let $s\in\Sc$.
  Suppose for a contradiction that $s\in\red(\Gamma)$.
  Then we have some $r\in\R$ with $\Gamma_s \subsetneq \Gamma_r$.
  From Lemma~\ref{lem:X}(\ref{item:X-min-Theta},\ref{item:X-redundant}) we have some $s'\in\Sc$ such that $\Gamma_r \subseteq \Gamma_{s'}$.
  But now $\Gamma_s \subsetneq \Gamma_r \subseteq \Gamma_{s'}$, contradicting $\Sc$ being minimum representative.
  Thus $\Sc \subseteq \R \setminus \red(\Gamma)$.

  For the reverse inclusion, let $r\in\R\setminus\red(\Gamma)$.
  Let $\Sc$ again be a minimum representative set for $L$.
  From Lemma~\ref{lem:X}(\ref{item:X-min-Theta},\ref{item:X-redundant}), we have some $s\in\Sc$ such that $\Gamma_r \subseteq \Gamma_s$.
  By definition of $\red(L)$, we conclude $\Gamma_r = \Gamma_s$.
  Now take $\Sc' = (\Sc \setminus \{s\}) \cup \{r\}$, that is, the same set of reports with $r$ replacing $s$.
  We have $\{\Gamma_s \mid s\in\Sc\} = \{\Gamma_{s'} \mid s'\in\Sc'\}$, and thus $\Sc'$ is a minimum representative for $L$ by Lemma~\ref{lem:X}(\ref{item:X-min-Theta}).
  As $r\in\Sc'$, we have $r \in \R'$ and we are done.
\end{proof}

As a corollary, we can state another characterization of $\trim$ in terms of redundant reports.
The result follows immediately from the definition of $\trim$.

\begin{corollary}\label{cor:trim-loss-red}
  Let $L : \R \to \reals^\Y_+$ have a finite representative set.
  Then $\trimcover(L) = L(\R \setminus \red(L))$.
\end{corollary}

This result motivates the analogous definition for properties, $\trimred(\Gamma) := \{\Gamma_r \mid r \in \R\setminus\red(\Gamma)\}$.
We leverage this definition next, to study embeddings at the property level.

\subsection{A property elicitation perspective on trimmed losses}
\label{sec:prop-trim}

We conclude this section with a structural result similar to Lemma~\ref{lem:X}, but for properties.
To do so, we must first generalize the definition of embeddeding to properties.
We say a property $\Gamma:\simplex\toto\reals^d$ embeds a finite property $\gamma:\simplex\toto\R$ if condition (ii) of Definition~\ref{def:loss-embed} holds.
In other words, $\Gamma$ embeds $\gamma$ if we have some representative set $\Sc\subseteq\R$ for $\gamma$ and embedding $\varphi:\Sc\to\reals^d$ such that for all $s\in\Sc$ we have $\gamma_s = \Gamma_{\varphi(s)}$.

Roughly, our result is as follows.
First, if $\Gamma$ embeds $\gamma$, the level sets of $\Gamma$ must all be redundant relative to $\gamma$.
In other words, $\Gamma$ is exactly the property $\gamma$ up to relabelling reports, but potentially with other reports ``filling in the gaps'' between the embedded reports of $\gamma$.
When working with convex surrogates, extra reports often arise in the convex hull of the embedded reports.
In this sense, we can regard embedding as only a slight departure from direct elicitation: if a loss $L$ directly elicits $\Gamma$ which embeds $\gamma$, we can almost think of $L$ as eliciting $\gamma$ itself.
Finally, we have an important converse: if $\Gamma$ has finitely many full-dimensional level sets, or equivalently, if $\trimred(\Gamma)$ is finite, then $\Gamma$ must embed some finite elicitable property with the same full-dimensional level sets.

The proof relies heavily on Lemma~\ref{lem:X}.
The statements about level sets use the following corollary of Proposition~\ref{prop:embed-iff-trims-equal} for properties.
\begin{corollary}\label{cor:trim-prop-red}
  Let $\Gamma : \simplex \toto \R$ be an elicitable property with a finite representative set.
  Then $\trimred(\Gamma)$ is the set of full-dimensional level sets of $\Gamma$.
\end{corollary}
\begin{proof}
  Let $L$ elicit $\Gamma$.
  From Lemma~\ref{lem:X}(\ref{item:X-min-Theta},\ref{item:X-full-dim}), for any finite minumum representative set $\Sc\subseteq\R$, the set $\{\Gamma_s\mid s\in\Sc\}$ is exactly the set of full-dimensional level sets $\Theta$ of $\Gamma$.
  From Proposition~\ref{prop:tfae-min-rep-nonredundant}, we have $r \in \R\setminus \red(\Gamma)$ if and only if $r$ is an element of some minimum representative set.
  As $\Gamma$ has at least one minimum representative set, we conclude $\trimred(\Gamma) = \{\Gamma_r \mid r\in \R\setminus\red(\Gamma)\} = \Theta$.  
\end{proof}

\begin{proposition}\label{prop:embed-trim}
  Let $\Gamma:\simplex\toto\reals^d$ be an elicitable property.
  The following are equivalent:
  \begin{enumerate}\setlength{\itemsep}{0pt}
  \item $\Gamma$ embeds a elicitable finite property $\gamma:\simplex \toto \R$.
  \item $\trimred(\Gamma)$ is a finite set.
  \item There is a finite minimum representative set $\U$ for $\Gamma$.
  \item There is a finite set of full-dimensional level sets $\hat\Theta$ of $\Gamma$, and $\cup\,\hat\Theta = \simplex$.
  \end{enumerate}
  Moreover, when any of the above hold, $\trimred(\gamma) = \trimred(\Gamma) = \{\Gamma_u \mid u\in\U\} = \hat\Theta$.
\end{proposition}

\begin{proof}
  Let $L$ be a fixed loss eliciting $\Gamma$, so that in particular $\risk L$ is fixed.
  By definition of elicitable properties, $L$ is minimizable.
  In each case, we will show that $\risk L$ is polyhedral (or equivalently, that $L$ has a finite representative set), and thus Lemma~\ref{lem:X} will give us the set $\Theta$ of full-dimensional level sets of $\Gamma$, uniquely determined by $\risk L$.
  We will prove $1 \Rightarrow 2 \Rightarrow 3 \Rightarrow 4 \Rightarrow 1$, and in each case show that the relevant set of level sets is equal to $\Theta$, giving the result.

  $1 \Rightarrow 2$:
  Let $\Sc$ be the representative set for $\gamma$ and $\varphi:\Sc\to\reals^d$ the embedding.
  Since $\Sc$ is finite, $\varphi(\Sc)$ is a finite representative set for $\Gamma$ (and $L$; thus, $\risk L$ is polyhedral).
  Corollary~\ref{cor:trim-prop-red} now gives $\trimred(\Gamma) = \Theta$, which is finite, showing Case 2.

  $2 \Rightarrow 3$:
  If $\trimred(\Gamma)$ is finite, then in particular we have a finite set of reports $\Sc \subseteq \reals^d$ such that $\trimred(\Gamma) = \{\Gamma_s \mid s\in\Sc\}$.
  As $\Gamma$ is elicitable, $\reals^d$ is representative for $\Gamma$.
  By definition of $\trimred$, we have $\simplex = \cup_{r\in\reals^d} \Gamma_r = \cup \trimred(\Gamma) = \cup_{s\in\Sc} \Gamma_s$, and therefore $\Sc$ is representative for $\Gamma$ and for $L$.
  As $\Sc$ is finite, we have $\risk L$ polyhedral.
  From Lemma~\ref{lem:X}(\ref{item:X-rep-contain-min}), we have some minimum representative set $\U\subseteq\Sc$ for $L$ and $\Gamma$, implying statement 3.
  Moreover, Lemma~\ref{lem:X}(\ref{item:X-min-Theta},\ref{item:X-full-dim}) gives $\{\Gamma_u \mid u\in\U\} = \Theta$.

  $3 \Rightarrow 4$:
  Let $\U$ be a finite minimum representative set for $\Gamma$.
  Then $\risk L = \risk{L|_\U}$ is polyhedral.
  Lemma~\ref{lem:X}(\ref{item:X-min-Theta},\ref{item:X-full-dim}) once again gives $\{\Gamma_u \mid u\in\U\} = \Theta$.
  We simply let $\hat\Theta = \Theta$, giving statement 4 as $\U$ is representative.

  $4 \Rightarrow 1$:
  Let $\Sc\subseteq\R$ such that $\{\Gamma_s \mid s\in\Sc\} = \hat \Theta$.
  Then $\Sc$ is representative for $\Gamma$ and $L$, as $\cup\hat\Theta = \simplex$.
  Again, this yields a finite representative set for $L$.
  Lemma~\ref{lem:loss-restrict} now states that $L$ embeds $L|_\Sc$, so $\Gamma$ embeds $\gamma := \Gamma|_\Sc$, giving Case 1.
  Finally, Corollary~\ref{cor:trim-prop-red} gives $\trimred(\gamma) = \Theta$.
\end{proof}

As a final observation, recall that a property $\Gamma$ elicited by a polyhedral loss has a finite range, in the sense that there are only finitely many optimal sets $\Gamma(p)$ for $p\in\simplex$ (Lemma~\ref{lem:polyhedral-range-gamma}).
Proposition~\ref{prop:embed-trim} shows a complementary statement: there are only finitely many level sets $\Gamma_u$ for $u\in\reals^d$.
In other words, both $\Gamma$ and $\Gamma^{-1}$ have a finite range as multivalued maps.

\section{Polyhedral Indirect Elicitation Implies Consistency}
\label{sec:poly-ie-consistency}

As we have observed, consistency, and therefore calibration, implies indirect elicitation (\S~\ref{subsec:calibration-links}).
In general, indirect elicitation is simpler and weaker than calibration, since it only depends on the loss through the property it elicits, i.e., its exact minimizers.
Surprisingly, for polyhedral surrogates, we show the converse: indirect elicitation implies calibration, and therefore consistency.

\begin{theorem}\label{thm:poly-ie-implies-consistent}
	Let $L:\reals^d \to \reals^\Y_+$ be a polyhedral loss which indirectly elicits a finite property $\gamma$.
  For any loss $\ell$ eliciting $\gamma$, there exists a link $\psi$ such that $(L, \psi)$ is calibrated with respect to $\ell$.
\end{theorem}

One technical detail is that the link function may have to change.
That is, we will show that if $(L,\psi)$ indirectly elicits $\prop{\ell}$, then there exists some potentially different link $\psi'$ such that $(L,\psi')$ is calibrated with respect to $\ell$.
To see why this change may be necessary, consider again the example from \S~\ref{subsec:calibration-links}:
hinge loss with the link $\psi(u) = -1$ for $u < 1$ and $\psi(u) = 1$ for $u\geq 1$.
Here indirect elicitation is achieved, since we have $\psi((-\infty,-1]) = \{-1\}$ and $\psi([1,\infty)) = \{1\}$, but the link is not $\epsilon$-separated for any $\epsilon>0$.
In general, it is not clear whether one can always adjust the link in this case to achieve separation, and therefore calibration.
Fortunately, for polyhedral surrogates, one can always ``thicken'' a given link to achieve separation.

We give two proofs of Theorem~\ref{thm:poly-ie-implies-consistent}.
The first is direct: we show, as foreshadowed in \S~\ref{sec:calibration}, that our thickened link construction can be generalized for indirect elicitation.
In fact, we will further prove that our general construction recovers every possible calibrated link function.
The second proof highlights the central role that embeddings play when reasoning about polyhedral surrogates.
Specifically, we will show that if a polyhedral surrogate indirectly elicits a finite property, the link function must ``pass through'' an embedding, giving calibration through Construction~\ref{const:eps-thick-link}.

\subsection{Generalizing the thickened link construction}

Given that Construction~\ref{const:eps-thick-link} uses the embedding $\varphi:\Sc\to\reals^d$ in a crucial role, it is not immediately clear how to generalize the construction beyond embeddings.
Specifically, this crucial role is in the definition of $R_U$, the set of target reports which must be optimal whenever a given surrogate report set $U\subseteq\reals^d$ is optimal.
Using the embedding, we can simply define $R_U = \{ r\in\Sc \mid \varphi(r) \in U \}$, since the definition of embedding means that those are exactly the target reports (among the representative set $\Sc$) which are optimal when $U$ is.

Now suppose we merely know that $L$ indirectly elicits some finite property $\gamma:\simplex\toto\R$.
In \S~\ref{app:sep-link-exists}, we give Construction~\ref{const:general-eps-thick-link}, which is the same as Construction~\ref{const:eps-thick-link} but with the following modification of $R_U$ to $\hat R_U$.
Let $\Gamma = \prop L$ and $\U = \{\Gamma(p) \mid p \in \simplex\}$ as before.
Then for all $U\in\U$, we define $\hat R_U := \{r\in\R \mid \Gamma_U \subseteq \gamma_r\}$, where $\Gamma_U := \{p\in\simplex \mid U = \Gamma(p)\}$ is the set of distributions for which $U$ is the surrogate optimal set.
In words, $\hat R_U$ is the set of reports $r$ which may be linked to from points in $U$, in the sense that $U$ being $L$-optimal implies $r$ is $\ell$-optimal.
For the special case where $L$ embeds $\ell$, it is straightforward to verify that $R_U = \hat R_U \cap \Sc$.
As a result, Construction~\ref{const:eps-thick-link} is the special case of Construction~\ref{const:general-eps-thick-link} where one is given an embedding and restricts to the representative set $\Sc$ (Lemma~\ref{lem:general-construction-special-case}).

The main result of \S~\ref{app:sep-link-exists} is that, if a polyhedral surrogate indirectly elicits a finite property, then for small enough $\epsilon>0$, Construction~\ref{const:general-eps-thick-link} always produces a link (Proposition~\ref{prop:general-eps-thick-produce}).
Combined with the fact that, by design, the construction and our choice of $\hat R_U$ enforce separation, we have the following.
\begin{proposition}\label{prop:general-eps-thick-separated}
  Let $L$ be a polyhedral surrogate which indirectly elicits a finite property $\gamma$.
  Then there exists $\epsilon_0>0$ such that for all $0 < \epsilon \leq \epsilon_0$, Construction \ref{const:general-eps-thick-link} for $L,\gamma,\epsilon,\|\cdot\|_{\infty}$ produces a separated link from $\prop{L}$ to $\gamma$.
\end{proposition}
\noindent
Since separation is equivalent to calibration for polyhedral surrogates (Theorem~\ref{thm:calibrated-separated}), we now have Theorem~\ref{thm:poly-ie-implies-consistent}: indirect elicitation implies calibration for polyhedral surrogates.

In fact, we can show something stronger: $\hat R_U$ enforces separation exactly, and therefore every possible calibrated link must arise from Construction~\ref{const:general-eps-thick-link}.

\begin{theorem}\label{thm:link-char}
  A link $\psi$ is calibrated for a given polyhedral surrogate $L$ and discrete target $\ell$ if and only if there exists $\epsilon>0$ such that $\psi$ is produced by Construction \ref{const:general-eps-thick-link} for $L,\prop\ell,\epsilon,\|\cdot\|$.
\end{theorem}

\subsection{Centrality of embeddings}

To derive another proof of Theorem~\ref{thm:poly-ie-implies-consistent}, we now show that, for polyhedral surrogates, indirect elicitation must always pass through an embedding.
That is, if $L$ indirectly elicits $\gamma$, then there is some loss $\ell$ which $L$ embeds, such that $\ell$ indirectly elicits $\gamma$.
This result holds more generally whenever $L$ has a finite representative set, as in \S~\ref{sec:min-rep-sets}.
\begin{lemma}\label{lem:ie-iff-embeds-refinement}
  Let $L:\reals^d\to\reals^\Y_+$ be polyhedral.
  Then $L$ indirectly elicits a property $\gamma$ if and only if $L$ tightly embeds a discrete loss $\ell$ that indirectly elicits $\gamma$.
\end{lemma}
\begin{proof}
  Let $\Gamma = \prop L$.
From Lemma~\ref{lem:X}(\ref{item:X-tight-embed}), $L$ tightly embeds a discrete loss.
  Furthermore, Lemma~\ref{lem:X}(\ref{item:X-min-Theta},\ref{item:X-redundant},\ref{item:X-tight-embed}) implies that $L$ indirectly elicits $\hat\gamma := \prop \ell$ for any discrete loss $\ell$ that $L$ tightly embeds.
  The link is any function $\hat\psi: u \mapsto r$ such that $\Gamma_u \subseteq \hat\gamma_r$ for all $u \in \reals^d$.

  We will prove the stronger statement that, for any property $\gamma$, and any loss $\ell:\R\to\reals^\Y_+$ that $L$ tightly embeds, $L$ indirectly elicits $\gamma$ if and only if $\ell$ indirectly elicits $\gamma$.
  If $\ell$ indirectly elicits $\gamma$ via the link $\psi$, then $L$ indirectly elicits $\gamma$ by transitivity of subset inclusion, as $\Gamma_u \subseteq \hat\gamma_{\hat\psi(u)} \subseteq \gamma_{\psi \circ \hat\psi(u)}$ for all $u \in \reals^d$.
  Conversely, suppose $L$ indirectly elicits $\gamma$ via the link $\psi$.
  As $L$ tightly embeds $\ell$, from Lemma~\ref{lem:X}(\ref{item:X-min-Theta},\ref{item:X-tight-embed}), the level sets of $\hat \gamma$ are contained in the set $\{\Gamma_u \mid u\in\reals^d\}$.
  Letting the map $\hat\psi:\R\to\reals^d$ exhibit this contaiment, we have $\hat\gamma_r = \Gamma_{\hat\psi(r)} \subseteq \gamma_{\psi \circ \hat\psi (r)}$ for all $r\in\R$.
\end{proof}

\begin{proof}[{Alternate proof of Theorem~\ref{thm:poly-ie-implies-consistent}}]
  Let $\R$ be the range of $\gamma$, so that $\gamma: \simplex \toto \R$, and let $\ell$ elicit $\gamma$.
  By Lemma~\ref{lem:ie-iff-embeds-refinement}, $L$ tightly embeds a discrete loss $\hat\ell:\hat\R\to\reals^\Y_+$ such that $\hat\ell$ indirectly elicits $\gamma$; let $\psi^\R: \hat\R \to \R$ be the corresponding link function.
  Let $\hat\gamma := \prop{\hat\ell}$ be the property that $\hat\ell$ directly elicits.
  Then for all $r\in\hat\R$ and $p\in\simplex$ we have $r\in \hat\gamma(p) \implies \psi^\R(r) \in \gamma(p)$. 
  Moreover, Construction~\ref{const:eps-thick-link} gives a link function $\hat \psi : \reals^d \to \hat\R$ such that $(L,\hat\psi)$ is calibrated with respect to $\hat\ell$.

  Consider $\psi := \psi^\R \circ \hat\psi$ and fix $p\in\simplex$.
	For any $u\in\reals^d$, if $\hat\psi(u) \in \hat\gamma(p)$, then $\psi(u) = \psi^\R\circ \hat\psi(u) \in \gamma(p)$ by definition of $\hat\psi$ and $\psi^\R$.
  Contrapositively,
  $\psi(u) \notin \gamma(p) \implies \hat\psi(u) \notin \hat\gamma(p)$.
  Thus, we have
  \begin{equation*}
    \label{eq:link-set-inclusion}
    \{u\in\reals^d \mid \psi(u) \not \in \gamma(p) \} \subseteq \{u\in\reals^d \mid \hat\psi(u) \not \in \hat\gamma(p) \}~.
  \end{equation*}
  Combined with the fact that $(L,\hat\psi)$ is calibrated with respect to $\hat\ell$, we have
	\begin{align*}
\inf_{u\in\reals^d : \psi(u) \not \in \gamma(p)} \inprod{p}{L(u)} \geq	\inf_{u\in\reals^d : \hat\psi(u) \not \in \hat\gamma(p)} \inprod{p}{L(u)} > \inf_{u\in\reals^d}\inprod{p}{L(u)}~,
	\end{align*}
  showing calibration of $\psi$.
\end{proof}

\section{Conclusion} \label{sec:conclusion}

In this work, we introduce an embedding framework to design and analyze consistent, convex surrogates for discrete prediction tasks.
Our results are constructive; as we outline in \S~\ref{sec:applications}, they can be fruitfully applied to a range of tasks, from designing new surrogates and link functions to understanding the consistency or inconsistency of existing surrogates.
Beyond these tools, our results shed light on fundamental questions about the design of consistent surrogates.

Perhaps the most pressing open direction is simply to apply our framework to prediction problems of interest.
We hope that the discussion in \S~\ref{sec:apply-embedd-fram}, and the detailed examples in subsequent works, serve as useful guidelines for doing so.
A particularly promising domain to apply our framework is structured prediction, where relatively few consistent surrogates are known.
Indeed, our framework has already been applied to submodular structured problems (\S~\ref{sec:lovasz-hinge}) and to max-margin losses~\cite{nowak2022consistency}.

Beyond applying our fromework, we see several interesting directions for theoretical research.
Below we outline several such directions.

\paragraph{Prediction dimension}

It can be important for applications to understand the minimum prediction dimension $d$ of a consistent convex surrogate $L:\reals^d\to \reals^\Y_+$ for a given target problem, also called its convex elicitation complexity~\cite{frongillo2021elicitation}.
Theorem~\ref{thm:discrete-loss-poly-embeddable} constructs a consistent surrogate for any discrete loss, with prediction dimension $d = n := |\Y|$.
In some settings, such as structured prediction and information retrieval, a prediction dimension of $d=n$ can be prohibitively large.
For example, in \S~\ref{sec:lovasz-hinge} we discuss structured problems which decompose as $k$ simple subproblems, like pixel classification for image segmentation.
The Lov\'asz hinge has prediction dimension $k$ for this problem, whereas our construction would give one with $d = n = 2^k$, an impractical number even for relatively small images.
While one could achieve $d = n-1$ with a simple modification to our construction,\footnote{One can always reduce to $d=n-1$ in Theorem~\ref{thm:discrete-loss-poly-embeddable} via a linear transformation from $\reals^n$ to $\reals^{n-1}$ which is injective on $\simplex$; redefining the surrogate appropriately, the Bayes risks will still match.}
it is unclear when and how the prediction dimension could be further lowered.

Beyond studying convex elicitation complexity directly~\citep{ramaswamy2016convex,finocchiaro2021unifying,frongillo2021elicitation},
one promising approach to this question is to first understand the minimum $d$ for which a polyhedral surrogate $L:\reals^d\to \reals^\Y_+$ embeds $\ell$, called the \emph{embedding dimension} of $\ell$, and then relate this dimension to polyhedral, or general convex, elicitation complexity.
One reason this approach may be fruitful is that embeddings have much more structure than general convex losses, such as the fact that calibrated links arise automatically (\S~\ref{thm:thickened-separated}).
Yet from \S~\ref{sec:poly-ie-consistency} and similar observations, it may well be that the lowest possible prediction dimension is achieved by an embedding.

In previous work, we introduce and present some bounds for embedding dimension
based on optimality conditions \citep{finocchiaro2020embedding}.
We show in particular that a target loss has embedding dimension $1$ if and only if it hase convex elicitation complexity $1$, underscoring the possibility that these quantities may be the same for all discrete losses.
It is unclear if these bounds are tight or if they can be improved by leveraging information about adjacent level sets of an embedded property.
Moreover, beyond the fact that the embedding dimension upper bounds convex elicitation complexity, it remains to understand the relationship between these two quantities in dimensions greater than~$1$.

\paragraph{Indirect elicitation as a condition for consistency}

It is well-known that statistical consistency is equivalent to calibration (Definition~\ref{def:calibrated}) for discrete target problems.
As calibration is a much easier condition to work with, and in particular, only involves the conditional label distributions, the bulk of work on consistency of learning algorithms for discrete target problems uses calibration.
It is easy to verify that calibration in turn implies indirect elicitation, meaning that exact minimizers of the surrogate loss are linked to exact minimizers of the target.
In \S~\ref{sec:poly-ie-consistency}, we show that indirect elicitation is actually \emph{equivalent} to calibration, and therefore consistency, when restricting to the class of polyhedral surrogates.
As indirect elicitation is even simpler of a condition than calibration, an important line of future work is to identify other classes of surrogates for which this equivalence holds.

\paragraph{Polyhedral vs.\ smooth surrogates}
The literature on convex surrogates focuses mainly on smooth surrogate losses~\citep{crammer2001algorithmic,bartlett2006convexity,bartlett2008classification, duchi2018multiclass, williamson2016composite, reid2010composite,menon2019multilabel,zhang2020convex,bao2020calibrated}.
In practice, minimizing such surrogates often implicitly fits a model to the full conditional label distributions.
On the other hand, \citet[Section 1.2]{ramaswamy2018consistent} contend that optimizing nonsmooth losses may enable reduction of the prediction dimension while maintaining consistency relative to smooth losses, improving downstream efficiency of the learning algorithm.
While generalization rates may suffer for nonsmooth losses, polyhedral surrogates achieve linear regret transfer bounds (\S~\ref{subsec:regret-bounds}), so the target generalization rates may remain the same; see also~\citet{frongillo2021surrogate}.
Even further, \citet{lapin2016loss} suggest that optimizing a nonsmooth loss that directly captures the target problem of interest, rather than a smooth one that implicitly fits to the full conditional label distributions, can improve performance in limited data settings.
We would like to verify this intuition, with specific cases or broad results comparing smooth and polyhedral losses.

\paragraph{Superprediction sets}

An interesting direction is to understand consistent surrogates by studying their superprediction sets, as has been done for proper losses~\citep{williamson2014geometry}.
The superprediction set of a loss is the set of loss vectors weakly dominated by the range of the loss: $\{v \in \reals^\Y \mid \exists r\; L(r) \leq v\}$, where the inequality holds pointwise.
One appealing aspect of the superprediction set is that it ignores the surrogate reports and focuses directly on the set of loss vectors, in a similar fashion to the trim operation in \S~\ref{sec:trim}.
In particular, it may be that questions about the required prediction dimension (see above) could be more readily answered by trying to find low-dimensional structures in the superprediction set of the target loss.

\paragraph{Convex envelope}

Finally, recall that we motivated the idea of an embedding as a way to ``convexify'' a discrete loss.
It is not clear, however, how embeddings relate to the convex envelope operation, which is perhaps the most direct way to perform this convexification given the map $\varphi$.
For example, suppose $L:\reals^d\to\reals^\Y_+$ embeds $\ell:\R\to\reals^\Y_+$ via the embedding $\varphi:\R\to\reals^d$, and consider the (polyhedral) surrogate $L':\reals^d\to\reals^\Y_+$ given by $L'_y = (L_y + \ones_{\varphi(\R)})^{**}$, where here $\ones$ denotes the convex indicator and $(\cdot)^{**}$ the biconjugate.
(One might also consider similar operations that keep $L'$ finite-valued.)
When is it the case that $L'$ also embeds $\ell$?
Conversely, we would like to know when the construction in Theorem~\ref{thm:discrete-loss-poly-embeddable} can be viewed as a convex envelope.

\subsection*{Acknowledgements}
We thank Arpit Agarwal and Peter Bartlett for many early discussions and insights,
Stephen Becker for a reference to Hoffman constants,
and Nishant Mehta, Enrique Nueve, and Anish Thilagar for other suggestions.
This material is based upon work supported by the National Science Foundation under Grant Nos.\ CCF-1657598, IIS-2045347, and DGE-1650115.

\newpage
\bibliographystyle{plainnat}
\bibliography{diss,extra}

\appendix

\newpage
\section{Power diagrams}\label{app:power-diagrams}
We begin with several definitions from Aurenhammer~\cite{aurenhammer1987power}.
\begin{definition}\label{def:cell-complex}
  A \emph{cell complex} in $\reals^d$ is a set $C$ of faces (of dimension $0,\ldots,d$) which (i) union to $\reals^d$, (ii) have pairwise disjoint relative interiors, and (iii) any nonempty intersection of faces $F,F'$ in $C$ is a face of $F$ and $F'$ and an element of $C$.
\end{definition}

\begin{definition}\label{def:power-diagram}
  Given sites $s_1,\ldots,s_k\in\reals^d$ and weights $w_1,\ldots,w_k \geq 0$, the corresponding \emph{power diagram} is the cell complex given by
  \begin{equation}
    \label{eq:pd}
    \cell(s_i) = \{ x \in\reals^d : \forall j \in \{1,\ldots,k\} \, \|x - s_i\|^2 - w_i \leq \|x - s_j\|^2 - w_j\}~.
  \end{equation}
\end{definition}

\begin{definition}\label{def:affine-equiv}
  A cell complex $C$ in $\reals^d$ is \emph{affinely equivalent} to a (convex) polyhedron $P \subseteq \reals^{d+1}$ if $C$ is a (linear) projection of the faces of $P$.
\end{definition}

Some of the convex polyhedra we study are the (negative) Bayes risks of loss functions, whose projections onto $\simplex$ form the level sets of the property they elicit.
The following result from~\citet{aurenhammer1987power} therefore immediately applies to elicitable properties.
We also make use of the same structure for the loss itself; in particular, one can consider the epigraph of a polyhedral convex function on $\reals^d$ and the projection down to $\reals^d$.
In either case, we refer to the resulting power diagram as being \emph{induced} by the convex function.

\begin{theorem}[Aurenhammer~\cite{aurenhammer1987power}]\label{thm:aurenhammer}
	A cell complex is affinely equivalent to a convex polyhedron if and only if it is a power diagram.
\end{theorem}

We extend Theorem~\ref{thm:aurenhammer} to a weighted sum of convex functions, showing that the induced power diagram is the same for any choice of strictly positive weights.

\begin{lemma}\label{lem:polyhedral-pd-same}
	Let $f_1,\ldots,f_m:\reals^d\to\reals$ be polyhedral convex functions.
	The power diagram induced by $\sum_{i=1}^m p_i f_i$ is the same for all $p \in \inter(\simplex)$.
\end{lemma}
\begin{proof}
	For any polyhedral convex function $g$ with epigraph $P$, the proof of~\citet[Theorem 4]{aurenhammer1987power} shows that the power diagram induced by $g$ is determined by the facets of $P$.
	Let $F$ be a facet of $P$, and $F'$ its projection down to $\reals^d$.
	It follows that $g|_{F'}$ is affine, and thus $g$ is differentiable on $\inter(F')$ with constant derivative $d\in\reals^d$.
	Conversely, for any subgradient $d'$ of $g$, the set of points $\{x\in\reals^d : d'\in\partial g(x)\}$ is the projection of a face of $P$; we conclude that $F = \{(x,g(x))\in\reals^{d+1} : d\in\partial g(x)\}$ and $F' = \{x\in\reals^d : d\in\partial g(x)\}$.
	
	Now let $f := \sum_{i=1}^k f_i$ with epigraph $P$, and $f' := \sum_{i=1}^k p_i f_i$ with epigraph $P'$.
	By Rockafellar~\cite{rockafellar1997convex}, $f,f'$ are polyhedral.
	We now show that $f$ is differentiable whenever $f'$ is differentiable:
	\begin{align*}
	\partial f(x) = \{d\}
	&\iff \sum_{i=1}^k \partial f_i(x) = \{d\} \\
	&\iff \forall i\in\{1,\ldots,k\}, \; \partial f_i(x) = \{d_i\} \\
	&\iff \forall i\in\{1,\ldots,k\}, \; \partial p_i f_i(x) = \{p_id_i\} \\
	&\iff \sum_{i=1}^k \partial p_if_i(x) = \left\{\sum_{i=1}^k p_id_i\right\} \\
	&\iff \partial f'(x) = \left\{\sum_{i=1}^k p_id_i\right\}~.
	\end{align*}
	From the above observations, every facet of $P$ is determined by the derivative of $f$ at any point in the interior of its projection, and vice versa.
	Letting $x$ be such a point in the interior, we now see that the facet of $P'$ containing $(x,f'(x))$ has the same projection, namely $\{x'\in\reals^d : \nabla f(x) \in \partial f(x')\} = \{x'\in\reals^d : \nabla f'(x) \in \partial f'(x')\}$.
	Thus, the power diagrams induced by $f$ and $f'$ are the same.
	The conclusion follows from the observation that the above held for any strictly positive weights $p$, and $f$ was fixed.
\end{proof}

We now include the full proof of Lemma~\ref{lem:polyhedral-range-gamma}.

\polyhedralrangegamma*
\begin{proof}
	First, observe that $L: \reals^d \to \reals^\Y_+$ is finite and bounded from below (by $0$), and thus its infimum is finite. 
	Therefore, we can apply \citet[Corollary 19.3.1]{rockafellar1997convex} to conclude that its infimum is attained for all $p \in \simplex$ and is therefore minimizable.
  Thus, $L$ elicits a property.
	
	For all $p$, let $P(p)$ be the epigraph of the convex function $u\mapsto \inprod{p}{L(u)}$.
	From Lemma~\ref{lem:polyhedral-pd-same}, we have that the power diagram $D_\Y$ induced by the projection of $P(p)$ onto $\reals^d$ is the same for any $p\in\inter(\simplex)$.
	Let $\F_\Y$ be the set of faces of $D_\Y$, which by the above are the set of faces of $P(p)$ projected onto $\reals^d$ for any $p\in\inter(\simplex)$.
	
	We claim for all $p\in\inter(\simplex)$, that $\Gamma(p) \in \F_\Y$.
	To see this, let $u \in \Gamma(p)$, and $u' = (u,\inprod{p}{L(u)}) \in P(p)$.
	The optimality of $u$ is equivalent to $u'$ being contained in the face $F$ of $P(p)$ exposed by the normal $(0,\ldots,0,-1)\in\reals^{d+1}$.
	Thus, $\Gamma(p) = \argmin_{u\in\reals^d} \inprod{p}{L(u)}$ is a projection of $F$ onto $\reals^d$, which is an element of $\F_\Y$.
	
	Now for $p \not \in \inter(\simplex)$, consider $\Y'\subsetneq \Y$, $\Y'\neq\emptyset$.
	Applying the above argument, we have a similar guarantee: a finite set $\F_{\Y'}$ such that $\Gamma(p) \in \F_{\Y'}$ for all $p$ with support exactly $\Y'$.
	Taking $\F = \bigcup\{\F_{\Y'} | \Y'\subseteq\Y, \Y'\neq\emptyset\}$, we have for all $p\in\simplex$ that $\Gamma(p) \in \F$, giving $\U \subseteq \F$.
	As $\F$ is finite, so is $\U$, and the elements of $\U$ are closed polyhedra as faces of $D_{\Y'}$ for some $\Y'\subseteq\Y$.
\end{proof}

\section{Equivalence of Separation and Calibration for Polyhedral Surrogates}
\label{sec:equiv-sep-calib}

We recall that Theorem \ref{thm:link-main} states that, if a polyhedral $L$ embeds a discrete $\ell$, then there exists a calibrated link $\psi$.
Theorem \ref{thm:link-main} is directly implied by the combination of Theorem \ref{thm:calibrated-separated}, that calibration is equivalent to separation (Definition \ref{def:sep-link}); and Theorem \ref{thm:thickened-separated}, existence of a separated link.
Theorem \ref{thm:calibrated-separated} is proven in this section and Theorem \ref{thm:thickened-separated} is proven in Appendix \ref{app:sep-link-exists}.

Throughout we will work with the two \emph{regret} functions:
the \emph{surrogate regret} $R_L(u,p) = \inprod{p}{L(u)} - \risk{L}(p)$, and similarly the \emph{target regret} $R_{\ell}(r,p) = \inprod{p}{\ell(r)} - \risk{\ell}(p)$.
We will use these functions again when we prove surrogate regret bounds (\S~\ref{app:regret-bounds}).

We first show one direction: any calibrated link from a polyhedral surrogate to a discrete target must be $\epsilon$-separated.
The proof follows a similar argument to that of~\citet[Lemma 6]{tewari2007consistency}.
\begin{lemma}\label{lem:calibrated-eps-sep}
  Let polyhedral surrogate $L:\reals^d \to \reals^\Y_+$, discrete loss $\ell:\R\to\reals^\Y_+$, and link $\psi:\reals^d\to\R$ be given such that $(L,\psi)$ is calibrated with respect to $\ell$.
  Then there exists $\epsilon>0$ such that $\psi$ is $\epsilon$-separated with respect to   $\prop{L}$ and $\prop{\ell}$.
\end{lemma}
\begin{proof}
  Let $\Gamma := \prop{L}$ and $\gamma := \prop{\ell}$.
  Suppose that $\psi$ is not $\epsilon$-separated for any $\epsilon>0$.
  Then letting $\epsilon_i := 1/i$ we have sequences $\{p_i\}_i \subset \simplex$ and  $\{u_i\}_i \subset \reals^d$ such that for all $i\in\mathbb N$ we have both $\psi(u_i) \notin \gamma(p_i)$ and $d_\infty(u_i,\Gamma(p_i)) < \epsilon_i$.
  First, observe that there are only finitely many values for $\gamma(p_i)$ and $\Gamma(p_i)$, as $\R$ is finite and $L$ is polyhedral (from Lemma~\ref{lem:polyhedral-range-gamma}).
  Thus, there must be some $p\in\simplex$ and some infinite subsequence indexed by $j\in J \subseteq \mathbb N$ where
  for all $j\in J$, we have $\psi(u_j) \notin \gamma(p)$ and $\Gamma(p_j) = \Gamma(p)$.

  Next, observe that, as $L$ is polyhedral, the expected loss $\inprod{p}{L(u)}$ is $\beta$-Lipschitz in $\|\cdot\|_\infty$ for some $\beta>0$.
  Thus, for all $j\in J$, we have
  \begin{align*}
    d_\infty(u_i,\Gamma(p)) < \epsilon_j
    &\implies \exists u^*\in\Gamma(p)\; \|u_j-u^*\|_\infty < \epsilon_j
    \\
    &\implies \left| \inprod{p}{L(u_j)} - \inprod{p}{L(u^*)} \right| < \beta\epsilon_j
    \\
    &\implies \left| \inprod{p}{L(u_j)} - \risk{L}(p) \right| < \beta\epsilon_j~.
  \end{align*}
  Finally, for this $p$, we have
  \begin{align*}
    \inf_{u:\psi(u)\notin\gamma(p)} \inprod{p}{L(u)}
    \leq
    \inf_{j\in J} \inprod{p}{L(u_j)}
    =
    \risk{L}(p)~,
  \end{align*}
  contradicting the calibration of $\psi$.
\end{proof}

For the other direction, we will make use of Hoffman constants for systems of linear inequalities.
See \citet{zalinescu2003sharp} for a modern treatment.
\begin{theorem}[Hoffman constant \cite{hoffman1952approximate}]
  \label{thm:hoffman}
  Given a matrix $A\in\reals^{m\times n}$, there exists some smallest $H(A)\geq 0$, called the \emph{Hoffman constant} (with respect to $\|\cdot\|_\infty$), such that for all $b\in\reals^m$ and all $x\in\reals^n$,
  \begin{equation}
    \label{eq:hoffman}
    d_\infty(x,S(A,b)) \leq H(A) \|(A x - b)_+\|_\infty~,
  \end{equation}
  where $S(A,b) = \{x\in\reals^n \mid A x \leq b\}$ and $(u)_+ := \max(u,0)$ component-wise.
\end{theorem}

\begin{lemma}\label{lem:hoffman-polyhedral}
  Let $L: \reals^d \to \reals_+^{\Y}$ be a polyhedral loss with $\Gamma = \prop{L}$.
  Then for any fixed $p$, there exists some smallest constant $H_{L,p} \geq 0$ such that $d_{\infty}(u,\Gamma(p)) \leq H_{L,p} R_L(u,p)$ for all $u \in \reals^d$.
\end{lemma}
\begin{proof}
  Since $L$ is polyhedral, there exist $a_1,\ldots,a_m \in \reals^d$ and $c\in\reals^m$ such that we may write $\inprod{p}{L(u)} = \max_{1\leq j\leq m} a_j \cdot u + c_j$.
  Let $A \in \reals^{m\times d}$ be the matrix with rows $a_j$, and let $b = \risk{L}(p)\ones - c$, where $\ones\in\reals^m$ is the all-ones vector.
  Then we have
  \begin{align*}
    S(A,b)
    &:= \{u\in\reals^d \mid A u \leq b\}
    \\
    &= \{u\in\reals^d \mid A u + c \leq \risk{L}(p)\ones\}
    \\
    &= \{u\in\reals^d \mid \forall i\, (A u + c)_i \leq \risk{L}(p)\}
    \\
    &= \{u\in\reals^d \mid \max_i \;(A u + c)_i \leq \risk{L}(p)\}
    \\
    &= \{u\in\reals^d \mid \inprod{p}{L(u)} \leq \risk{L}(p)\}
    \\
    & = \Gamma(p)~.
  \end{align*}
  Similarly, we have $\max_i\; (A u - b)_i = \inprod{p}{L(u)} - \risk{L}(p) = \regret{L}{u}{p} \geq 0$.
  Thus,
  \begin{align*}
    \|(Au - b)_+\|_\infty
    &= \max_i\; ((Au - b)_+)_i
    \\
    &= \max((Au - b)_1,\ldots,(Au - b)_m, 0)
    \\
    &= \max(\max_i\; (Au - b)_i, \, 0)
    \\
    &= \max_i\; (Au - b)_i
    \\
    &= \regret{L}{u}{p}~.
  \end{align*}
  Now applying Theorem~\ref{thm:hoffman}, we have
  \begin{align*}
    d_\infty(u,\Gamma(p))
    &=    d_\infty(u,S(A,b))
    \\
    &\leq H(A) \|(Au-b)_+\|_\infty
    \\
    &= H(A) \regret{L}{u}{p}~.\qedhere
  \end{align*}
\end{proof}

Given discrete loss $\ell: \R \to \reals_+^{\Y}$, define the constant $C_{\ell} = \max_{r,r' \in \R, y \in \Y} \ell(r)_y - \ell(r')_y$.
We are now ready to prove Theorem \ref{thm:calibrated-separated}.
\calibratedseparated*
\begin{proof}
  Let $\gamma=\prop{\ell}$ and $\Gamma=\prop{L}$.
  From Lemma~\ref{lem:calibrated-eps-sep}, calibration implies $\epsilon$-separation.
  For the converse, suppose $\psi$ is $\epsilon$-separated with respect to $L$ and $\ell$.
  Fix $p\in\simplex$.
  To show calibration, it suffices to find a positive lower bound for $R_L(u,p)$ that holds for all $u\in\reals^d$ with $\psi(u) \notin \gamma(p)$.
  
  Applying the definition of $\epsilon$-separated and Lemma~\ref{lem:hoffman-polyhedral}, $\psi(u) \notin \gamma(p)$ implies
  \begin{align*}
    \epsilon &\leq    d_{\infty}(u,\Gamma(p)) \leq H_{L,p} R_L(u,p) \implies 1 \leq \frac{H_{L,p}}{\epsilon} R_L(u,p)~.
  \end{align*}
  Let $C_{\ell} = \max_{r,p} R_{\ell}(r,p)$.
  Then $R_{\ell}(\psi(u),p) \leq C_{\ell} \leq \frac{C_{\ell} H_{L,p}}{\epsilon} R_L(u,p)$.

  If $H_{L,p} = 0$, then for all $u\in\reals^d$ we have $R_\ell(\psi(u),p) = 0$, so calibration for this $p$ is trivial.
  Similarly, if $C_\ell = 0$, then $R_\ell(r,p) = 0$ for all $r\in\R$, so again $R_\ell(\psi(u),p) = 0$ for all $u\in\reals^d$.

  Now assume $C_\ell > 0$ and $H_{L,p} > 0$.
  Let $C'_{\ell,p} \doteq \min_{r \notin \gamma(p)} R_\ell(r,p) > 0$.
  (As we assume $C_\ell > 0$, we must have $\gamma(p) \neq \R$, so the minimum is attained.)
  Then for all $u$ such that $\psi(u) \notin \gamma(p)$, we have $R_\ell(\psi(u),p) \geq C'_{\ell,p}$.
  Rearranging, we have
  \[ \psi(u) \notin \gamma(p) \implies R_L(u,p) \geq \frac{C'_{\ell,p} \epsilon}{C_\ell H_{L,p}} > 0~.\]
  Thus, $\inf_{u : \psi(u) \notin \gamma(p)} \inprod{L(u)}{p} > \risk{L}(p)$.
  Since the above holds for all $p\in\simplex$, $\psi$ is calibrated.
\end{proof}

\section{Surrogate Regret Bounds} \label{app:regret-bounds}

\subsection{Proof of Theorem~\ref{thm:linear-regret-bound}}

\begin{lemma}\label{lem:linear-on-levelset}
  Suppose $(L,\psi)$ indirectly elicits $\ell$ and let $\Gamma = \prop{L}$.
  Then for any fixed $u,u^* \in \reals^d$ and $r \in \R$, the functions $R_L(u,\cdot)$ and $R_{\ell}(r,\cdot)$ are linear in their second arguments on $\Gamma_{u^*}$.
\end{lemma}
\begin{proof}
  Let $u^* \in \reals^d$ and $p \in \Gamma_{u^*}$.
  By definition, for all $p \in \Gamma_{u^*}$, $\risk{L}(p) = \inprod{p}{L(u^*)}$.
  So for fixed $u$,
    \[ R_L(u,p) = \inprod{p}{L(u)} - \inprod{p}{L(u^*)} = \inprod{p}{L(u) - L(u^*)} , \]
  a linear function of $p$ on $\Gamma_{u^*}$.
  Next, by the definition of indirect elicitation, there exists $r^*$ such that $\Gamma_{u^*} \subseteq \gamma_{r^*}$.
  By the same argument, for fixed $r$, $R_{\ell}(r,p) = \inprod{p}{\ell(r) - \ell(r^*)}$, a linear function of $p$ on $\gamma_{r^*}$ and thus on $\Gamma_{u^*}$.
\end{proof}

Recall the definitions of $C_\ell$ and $H_{L,p}$ from \S~\ref{sec:equiv-sep-calib}.

\begin{lemma}\label{lem:separated-constant-p}
  Let $\ell: \R \to \reals_+^{\Y}$ be a discrete target loss, $L: \reals^d \to \reals_+^{\Y}$ be a polyhedral surrogate loss, and $\psi: \reals^d \to \R$ a link function.
  If $(L,\psi)$ indirectly elicit $\ell$ and $\psi$ is $\epsilon$-separated, then for all $u$ and $p$,
    \[ R_{\ell}(\psi(u),p) \leq \frac{C_{\ell} H_{L,p}}{\epsilon} R_L(u,p) . \]
\end{lemma}
\begin{proof}
  If $\psi(u) \in \gamma(p)$, then $R_{\ell}(u,p) = 0$ and we are done.
  Otherwise, applying the definition of $\epsilon$-separated and Lemma~\ref{lem:hoffman-polyhedral},
  \begin{align*}
    \epsilon &<    d_{\infty}(u,\Gamma(p))  \\
             &\leq H_{L,p} R_L(u,p) .
  \end{align*}
  So $R_{\ell}(\psi(u),p) \leq C_{\ell} \leq \frac{C_{\ell} H_{L,p}}{\epsilon} R_L(u,p)$.
\end{proof}

We can now restate and prove Theorem~\ref{thm:linear-regret-bound}.
\begin{theorem*}[Theorem~\ref{thm:linear-regret-bound}]
  Let $\ell: \R \to \reals_+^{\Y}$ be discrete, $L: \reals^d \to \reals_+^{\Y}$ polyhedral, and $\psi: \reals^d \to \R$.
  If $(L,\psi)$ are consistent for $\ell$, then there exists constants $\epsilon_\psi, H_L > 0$ such that
    \[ (\forall h,\D) \quad R_{\ell}(\psi \circ h ; \D) \leq \frac{C_{\ell} H_L}{\epsilon_{\psi}} R_L(h ; \D) ~. \]
\end{theorem*}
\begin{proof}
  Let $\Gamma = \prop{L}$.
  From Lemma~\ref{lem:X}, there is a finite set $U \subset \reals^d$ of predictions such that (a) for each $u \in U$, the level set $\Gamma_u$ is a polytope (see e.g.\ Lemma~\ref{lem:level-set-is-projected-face}), and (b) $\cup_{u \in U} \Gamma_u = \simplex$.
  For each $u\in U$ let $\mathcal{Q}_u \subset \simplex$ be the finite set of vertices of the polytope $\Gamma_u$, and define the finite set $\mathcal{Q} = \cup_{u \in U} \mathcal{Q}_u$.
  Let $H_L := \max_{q\in\mathcal{Q}} H_{L,q}$.
  
  By Lemma~\ref{lem:calibrated-eps-sep}, $\psi$ is $\epsilon$-separated for some $\epsilon>0$; let $\epsilon_{\psi} = \epsilon$.
  By Lemma~\ref{lem:separated-constant-p}, for each $q \in \mathcal{Q}$,
  \[ R_{\ell}(\psi(u),q) \leq \frac{C_{\ell} H_{L,q}}{\epsilon_{\psi}} R_L(u,q) \leq \frac{C_{\ell} H_L}{\epsilon_{\psi}} R_L(u,q)\]
  for all $u\in\reals^d$.
  Now consider a general $p \in \simplex$, which is in some full-dimensional polytope level set $\Gamma_u$.
  Write $p = \sum_{q \in \mathcal{Q}_u} \beta(q) q$ for some convex combination $\beta \in \Delta_{\mathcal Q_u}$.
  By Lemma~\ref{lem:linear-on-levelset}, $R_L$ and $R_{\ell}$ are linear in the second argument on $\Gamma_u$, so for any $u'\in\reals^d$,
  \begin{align*}
    R_{\ell}(\psi(u'),p)
    &=    \sum_{q \in \mathcal{Q}_u} \beta(q) R_{\ell}(\psi(u'), q)  \\
    &\leq \sum_{q \in \mathcal{Q}_u} \beta(q) \frac{C_{\ell} H_{L}}{\epsilon_{\psi}} R_L(u', q) \\
    &\leq \frac{C_{\ell} H_L}{\epsilon_{\psi}} \sum_{q \in \mathcal{Q}_u} \beta(q) R_L(u', q)  \\
    &= \frac{C_{\ell} H_L}{\epsilon_{\psi}} R_L(u', p) .
  \end{align*}
  The result for $\D$ now holds by linearity of expectation over $\D$.
\end{proof}

\subsection{Tighter bounds}
\label{sec:regret-tighter-bounds}

Our goal in proving Theorem~\ref{thm:linear-regret-bound} is to show a broad result that consistent polyhedral losses always yield linear regret bounds.
As one may expect given the generality of the result, however, the specific constant we derive may be loose in some cases.
We now discuss some techniques to further tighten the constant.

Let us consider the tightest possible constant $c^*$ for which $R_{\ell}(\psi \circ h;\D) \leq c^* R_L(h;\D)$ for all $h$ and $\D$.
In general, for a fixed $p$, there is some smallest $c_p^*$ such that $R_{\ell}(\psi(u),p) \leq c_p^* R_L(u,p)$ for all $u$.
It therefore follows from our results that $c^* = \max_{p \in \mathcal{Q}} c_p^*$ for the finite set $\mathcal{Q}$ used in the proof, i.e., the vertices of the full-dimensional level sets of $\Gamma = \prop{L}$.

Above, we bounded $c_p^* \leq \frac{C_{\ell} H_{L,p}}{\epsilon_{\psi}}$.
The intuition is that some $u$ at distance $\geq \epsilon_{\psi}$ from $\Gamma(p)$, the optimal set, may link to a ``bad'' report $r = \psi(u) \not\in \gamma(p)$.
The rate at which $L$ grows is at least $H_{L,p}$, so the surrogate loss at $u$ may be as small as $\frac{\epsilon_{\psi}}{H_{L,p}}$, while the target regret may be as high as $C_{\ell} = \max_{r',p'} R_{\ell}(r',p')$.
The ratio of regrets is therefore bounded by $\frac{H_{L,p} C_{\ell}}{\epsilon_{\psi}}$.

The tightest possible bound, on the other hand, is $c_p^* = \sup_{u: \psi(u) \not\in \gamma(p)} \frac{R_{\ell}(\psi(u),p)}{R_L(u,p)}$.
This bound can be smaller if the values of numerator and denominator are correlated across $u$.
For example, $u$ may only be $\epsilon_{\psi}$-close to the optimal set when it links to reports $\psi(u)$ with lower target regret; or $L$ may have a smaller slope in the direction where the link's separation is larger than $\epsilon$.

To illustrate with a concrete example, consider the \emph{binary encoded predictions (BEP) surrogate} of \citet{ramaswamy2018consistent}, which we discuss in \S~\ref{sec:abstain}.
The target loss here is the abstain loss, $\ell(r,y) = \frac{1}{2}$ if $r = \bot$, otherwise $\ell(r,y) = \ones\{r \neq y\}$.
Letting $d = \ceil{\log_2 |\Y|}$, the BEP surrogate $L : \reals^d \to \reals^\Y_+$ is given by
$L(u)_y = \max_{j \in [d]} \left(1 - \varphi(y)_j u_j\right)_+$,
where $\varphi:\Y\to\{-1,1\}^d$ is an injection.
The associated link is $\psi(u) = \bot$ if $\|u\|_\infty \leq \tfrac{1}{2}$, otherwise $\psi(u) = \argmin_{y \in \Y} \|B(y) - u\|_{\infty}$.

One can show for $p = \delta_y$, the distribution with full support on some $y \in \Y$, that $L(u)_y = d_{\infty}(u,\Gamma(p))$ exactly, giving $H_{L,p} = 1$.
It is almost immediate that $\epsilon_{\psi} = \tfrac{1}{2}$.
Meanwhile, $R_{\ell}(r,p) \leq 1$, giving us an upper bound $c^*_p \leq \frac{(1)(1)}{1/2} = 2$.
The exact constant as given by \citeauthor{ramaswamy2018consistent}, however, is $c^* = 1$.
The looseness stems from the fact that for $p = \delta_y$, the closest reports $u$ to the optimal set, i.e., at distance only $\epsilon_{\psi} = \tfrac{1}{2}$ away, do not link to reports maximizing target regret; they link to the abstain report $\bot$, which has regret only $\tfrac{1}{2}$.
With this correction, and an observation that all $u$ linking to reports $y' \neq y$ are at distance at least $\tfrac{3}{2}$ from $\Gamma(p)$, we restore the tight bound $c^*_p \leq 1$.
A similar but slightly more involved calculation can be carried out for the other vertices $p \in \mathcal{Q}$, which turn out to be all vertices of the form $\tfrac{1}{2} \delta_y + \tfrac{1}{2} \delta_{y'}$.

Finally, while we use $\|\cdot\|_{\infty}$ to define the minimum-slope $H_L$ and the separation $\epsilon_\psi$, in principle one could use another norm.
One reason for restricting to $\|\cdot\|_\infty$ is that it is more compatible with Hoffman constants.
However, all definitions hold for other norms and so does the main upper bound, as existence of an $H_L$ and $\epsilon_{\psi}$ in $\|\cdot\|_{\infty}$ imply existence of constants for other norms.
These constants may change for different norms, and in particular, the optimal overall constant may arise from a norm other than $\|\cdot\|_\infty$.

\section{Existence of a Separated Link} \label{app:sep-link-exists}
In this section, we prove Theorem \ref{thm:thickened-separated} from \S~\ref{sec:calibration}, as discussed at the beginning of Appendix \ref{sec:equiv-sep-calib}: embeddings give rise to separated links.
The crux of the proof is showing that embeddings imply eq.~\eqref{eq:intersection-property-embedding}, the intersection condition on optimal sets, and that this condition is sufficient for the construction to produce a link.
Calibration then follows by the fact every link produced by Construction~\ref{const:eps-thick-link} is separated, and therefore calibrated.

In fact, we will show that the approach outlined above also suffices for the more general case where the given polyhedral surrogate indirectly elicits a given finite property.
This more general setting will allow us to prove the results from \S~\ref{sec:poly-ie-consistency}, and in particular, that indirect elicitation implies calibration for polyhedral surrogates.

To relate back to embeddings, we split the first phase into two: (i) an embedding is a special case of indirect elicitation (Lemma~\ref{lem:embedding-implies-indirect-elic}), and (ii) indirect elicitation is equivalent to the intersection condition (Lemma~\ref{lem:intersection-equiv-indirect-elic}).
We will then reason instead about Construction~\ref{const:general-eps-thick-link}, a generalization of Construction~\ref{const:eps-thick-link} for indirect elicitation.

\subsection{A more general construction}

As described in \S~\ref{sec:poly-ie-consistency}, the task of generalizing Construction~\ref{const:eps-thick-link} beyond embeddings reduces to carefully generalizing the definition of $R_U$.
Informally, $R_U$ in Construction~\ref{const:eps-thick-link} is the set of target reports which must be $\ell$-optimal whenever $U$ is $L$-optimal.
There we define $R_U$ simply as $\{r\in\Sc \mid \varphi(r) \in U\}$ where $\Sc$ is the given representative set.
For the more general case, we can define $R_U$ as follows.
\begin{definition}\label{def:general-link-defs}
  For polyhedral loss $L:\reals^d\to\reals^\Y_+$, and finite propert $\gamma:\simplex\toto\R$, we will define
  \begin{itemize}
  \item $\Gamma = \prop L$,
  \item $\U = \{\Gamma(p) \mid p \in \simplex\}$,
  \item $\Gamma_U := \{p\in\simplex \mid U = \Gamma(p)\}$ for all $U \in \U$,
  \item $R_U := \{r\in\R \mid \Gamma_U \subseteq \gamma_r\}$ for all $U \in \U$.
  \end{itemize}
\end{definition}

Construction~\ref{const:general-eps-thick-link} is essentially the same as Construction~\ref{const:eps-thick-link} but with the definition of $R_U$ above.

\begin{construction}[General $\epsilon$-thickened link] \label{const:general-eps-thick-link}
  Let $L:\reals^d\to\reals^\Y_+$, $\gamma:\simplex\toto\R$, $\epsilon > 0$, and a norm $\|\cdot\|$ be given, such that $L$ is polyhedral and indirectly elicits $\gamma$.
  Let $\U$ and $R_U$ be defined as in Definition~\ref{def:general-link-defs}.
  The \emph{$\epsilon$-thickened link} $\psi$ is constructed as follows.
  First, initialize the \emph{link envelope} $\Psi: \reals^d \to 2^{\R}$ by setting $\Psi(u) = \R$ for all $u$.
  Then for each $U \in \U$, for all points $u$ such that $\inf_{u^* \in U} \|u^*-u\| < \epsilon$, update $\Psi(u) = \Psi(u) \cap R_U$.
  If we have $\Psi(u)\neq\emptyset$ for all $u\in\reals^d$, then the construction \emph{produces a link} $\psi \in \Psi$ pointwise, breaking ties arbitrarily.
\end{construction}

It is straightforward to show that Construction~\ref{const:eps-thick-link} is a special case of Construction~\ref{const:general-eps-thick-link}.

\begin{lemma}\label{lem:general-construction-special-case}
  Let $L:\reals^d\to\reals^\Y_+$ be a polyhedral surrogate which embeds $\ell:\R\to\reals^\Y_+$ via the representative set $\Sc\subseteq\R$ and embedding $\varphi:\Sc\to\reals^d$.
  Then for any $\epsilon>0$ and norm $\|\cdot\|$,
  Construction~\ref{const:eps-thick-link} for $L,\ell,\epsilon,\|\cdot\|$ is equivalent to
  Construction~\ref{const:general-eps-thick-link} for $L,\prop{\ell|_\Sc},\epsilon,\|\cdot\|$.
\end{lemma}
\begin{proof}
  Let $\gamma = \prop{\ell|_\Sc}$, which is also given by $\gamma: p \mapsto \prop{\ell}(p) \cap \Sc$.
  Let $U\in\U$.
  Then for $r\in\Sc$, we have $r \in R_U \iff \Gamma_U \subseteq \gamma_r \iff \Gamma_U \subseteq \Gamma_{\varphi(r)} \iff (\Gamma(p) = U \implies \varphi(r) \in \Gamma(p)) \iff \varphi(r) \in U$.
  As we have $\R = \Sc$ in Construction~\ref{const:general-eps-thick-link} for $L,\prop{\ell|_\Sc},\epsilon,\|\cdot\|$, we conclude
  $R_U = \{r\in\Sc \mid \varphi(r) \in U\}$, exactly as in Construction~\ref{const:eps-thick-link} for $L,\ell,\epsilon,\|\cdot\|$.
  The equivalence of the two constructions follows.
\end{proof}

\subsection{Indirect elicitation and optimal set intersection}

We first show (i), that embedding is a special case of indirect elicitation.
\begin{lemma}\label{lem:embedding-implies-indirect-elic}
  If $L$ embeds $\ell$, then $L$ indirectly elicits $\prop\ell$.
\end{lemma}
\begin{proof}
  Let $\Gamma = \prop L$ and $\gamma = \prop\ell$.
  From Proposition \ref{prop:embed-trim}, we have $\trim(\gamma) = \trim(\Gamma) =: \Theta$.
  By definition of trim, for any $u\in\reals^d$, we have some $\theta \in \Theta$ such that $\Gamma_u \subseteq \theta$.
  Since $\trim(\gamma) = \Theta$, we have some $r\in\R$ such that $\theta = \gamma_r$, giving $\Gamma_u \subseteq \gamma_r$.
\end{proof}

We next show (ii), the equivalence of indirect elicitation and the following intersection condition.

\begin{lemma}\label{lem:intersection-equiv-indirect-elic}
  Let $L:\reals^d\to\reals^\Y_+$ be polyhedral and $\gamma:\simplex\toto\R$ be a finite property.
  Let $\U$ and $R_U$ be defined as in Definition~\ref{def:general-link-defs}.
  Then $L$ indirectly elicits $\gamma$
  if and only if the following condition holds
  \begin{align}
    \label{eq:general-intersection-condition}
    \forall \U'\subseteq\U,\; \cap_{U\in\U'} U \neq \emptyset \implies \cap_{U\in\U'} R_U \neq \emptyset~.
  \end{align}
\end{lemma}
\begin{proof}
  First assume $L$ indirectly elicits $\gamma$.
  As $\U$ is the range of $\Gamma$, we have for all $u\in\reals^d$ that $\Gamma_u = \cup\{\Gamma_U \mid U\in\U, u\in U\}$.
  Suppose $\cap_{U\in\U'} U \neq \emptyset$; let $u \in \cap_{U\in\U'} U$.
  As $u\in U$ for all $U\in\U'$, we have $\Gamma_U \subseteq \Gamma_u$ for all $U\in\U'$.
  By indirect elicitation, there exists some $r\in\R$ such that $\Gamma_u \subseteq \gamma_r$.
  Thus, for all $U\in\U'$, we have $\Gamma_U \subseteq \gamma_r$ and thus $r\in R_U$.
  We conclude $\cap_{U\in\U'} R_U \neq \emptyset$.

  For the converse, let $u\in\reals^d$.
  If $\Gamma_u = \emptyset$, then $\Gamma_u \subseteq \gamma_r$ for any $r\in\R$.
  Otherwise, $\Gamma_u \neq \emptyset$, and the set $\U'_u = \{U\in\U \mid u\in U\}$ is nonempty.
  Moreover, $\cap \U'_u \neq \emptyset$ as $u\in\cap \U'_u$.
  Eq.~\eqref{eq:general-intersection-condition} now gives some $r\in\cap\{R_U \mid U\in\U'_u\}$.
  By definition of $R_U$, for all $U\in\U'_u$ we have $\Gamma_U \emptyset \gamma_r$.
  Thus $\Gamma_u = \cup\{\Gamma_U \mid U\in\U'_u\} \subseteq \gamma_r$, showing indirect elicitation.
\end{proof}

\subsection{Convex geometry for separation}

Let some norm $\|\cdot\|$ on $\reals^d$ be given.
Given a set $T\subseteq\reals^d$ and a point $u\in\reals^d$, let $d(T,u) = \inf_{t \in T} \|t-u\|$.
Given two sets $T,T'\subseteq\reals^d$, let $d(T,T') = \inf_{t\in T, t' \in T'} \|t-t'\|$.
Finally, for $T\subseteq \reals^d$ and $\epsilon > 0$, let the ``thickening'' $B(T,\epsilon)$ be defined as
\[ B(T,\epsilon) = \{u \in \R' : d(T,u) < \epsilon \} . \]

The goal of this subsection is to prove the first part of step (iii): for small enough $\epsilon>0$, if any set of $\epsilon$-thickened optimal sets intersect, then the optimal sets themselves must intersect.
We will conclude that, for small enough $\epsilon$, the link envelope $\Psi$ is non-empty everywhere, meaning there will be legal choices left over for the link.
\begin{lemma}~\label{lem:thick-intersect}
  Let $\U$ be defined as in Definition~\ref{def:general-link-defs}.
  There exists $\epsilon_0 > 0$ such that, for any $0 < \epsilon \leq \epsilon_0$, for any subset $\{U_j : j \in \mathcal{J}\}$ of $\U$, if $\cap_j U_j = \emptyset$, then $\cap_j B(U_j,\epsilon) = \emptyset$.
\end{lemma}

The next few geometric results build to Lemma~\ref{lem:thick-intersect}.

\begin{lemma} \label{lem:enclose-halfspaces}
  Let $D$ be a closed, convex polyhedron in $\reals^d$.
  For any $\epsilon > 0$, there exists an \emph{open}, convex set $D'$, the intersection of a finite number of open halfspaces, such that
    \[ D \subseteq D' \subseteq B(D,\epsilon) . \]
\end{lemma}
\begin{proof}
  Let $S$ be the standard open $\epsilon$-ball $B(\{\vec{0}\},\epsilon)$.
  Note that $B(D,\epsilon) = D + S$ where $+$ is the Minkowski sum.
  Now let $S' = \{u : \|u\|_1 \leq \delta\}$ be the closed $\delta$ ball in $L_1$ norm.
  By equivalence of norms in Euclidean space~\cite[Appendix A.1.4]{boyd2004convex}, we can take $\delta$ small enough yet positive such that $S' \subseteq S$.
  By standard results, the Minkowski sum of two closed, convex polyhedra, $D'' = D + S'$ is a closed polyhedron, i.e. the intersection of a finite number of closed halfspaces. (A proof: we can form the higher-dimensional polyhedron $\{(x,y,z) : x \in D, y \in S', z = x+y\}$, then project onto the $z$ coordinates.)

  Now, if $T' \subseteq T$, then the Minkowksi sum satisfies $D + T' \subseteq D + T$.
  In particular, because $\emptyset \subseteq S' \subseteq S$, we have
    \[ D \subseteq D'' \subseteq B(D,\epsilon) . \]
  Now let $D'$ be the interior of $D''$, i.e. if $D'' = \{x : Ax \leq b\}$, then we let $D' = \{x: Ax < b\}$.
  We retain $D' \subseteq B(D,\epsilon)$.
  Further, we retain $D \subseteq D'$, because $D$ is contained in the interior of $D'' = D + S'$.
  (Proof: if $x \in D$, then for some $\gamma$, $x + B(\{\vec{0}\},\gamma) = B(x,\gamma)$ is contained in $D + S'$.)
  This proves the lemma.
\end{proof}

\begin{lemma} \label{lem:thick-nonempty}
  Let $\{U_j : j \in \mathcal{J}\}$ be a finite collection of closed, convex sets with $\cap_{j\in\mathcal{J}} U_j \neq \emptyset$.
  Let $\delta > 0$ be given.
  Then there exists  $\epsilon_0 > 0$ such that, for all $0 < \epsilon \leq \epsilon_0$, $\cap_j B(U_j,\epsilon) \subseteq B(\cap_j U_j, \delta)$.
\end{lemma}
\begin{proof}
  We induct on $|\mathcal{J}|$.
  If $|\mathcal{J}|=1$, set $\epsilon = \delta$.
  If $|\mathcal{J}|>1$, let $j\in\mathcal{J}$ be arbitrary, let $U' = \cap_{j'\neq j} U_{j'}$, and let $C(\epsilon) = \cap_{j' \neq j} B(U_{j'},\epsilon)$.
  Let $D = U_j \cap U'$.
  We must show that $B(U_j,\epsilon) \cap C(\epsilon) \subseteq B(D,\delta)$.
  By Lemma \ref{lem:enclose-halfspaces}, we can enclose $D$ strictly within a polyhedron $D'$, the intersection of a finite number of open halfspaces, which is itself strictly enclosed in $B(D,\delta)$.
  (For example, if $D$ is a point, then enclose it in a hypercube, which is enclosed in the ball $B(D,\delta)$.)
  We will prove that, for all small enough $\epsilon$, $B(U_j,\epsilon) \cap C(\epsilon)$ is contained in $D'$.
  This implies that it is contained in $B(D,\delta)$.

  For each halfspace defining $D'$, consider its complement $F$, a closed halfspace.
  We prove that $F \cap B(U_j,\epsilon) \cap C(\epsilon) = \emptyset$.
  Consider the intersections of $F$ with $U$ and $U'$, call them $G$ and $G'$.
  These are closed, convex sets that do not intersect (because $D$ in contained in the complement of $F$).
  So $G$ and $G'$ are separated by a nonzero distance, so $B(G,\gamma) \cap B(G',\gamma) = \emptyset$ for all small enough $\gamma$.
  And $B(G,\gamma) = F \cap B(U_j,\gamma)$ while $B(G',\gamma) = F \cap B(U',\gamma)$.
  This proves that $F \cap B(U_j,\gamma) \cap B(U',\gamma) = \emptyset$.
  By inductive assumption, $C(\epsilon) \subseteq B(U',\gamma)$ for small enough $\epsilon = \epsilon_F$.
  So $F \cap B(U_j,\gamma) \cap C(\epsilon) = \emptyset$.
  We now let $\epsilon_0$ be the minimum over these finitely many $\epsilon_F$ (one per halfspace).
\end{proof}

\begin{figure}
\caption{Illustration of a special case of the proof of Lemma \ref{lem:thick-nonempty} where there are two sets $U_1,U_2$ and their intersection $D$ is a point. We build the polyhedron $D'$ inside $B(D,\delta)$. By considering each halfspace that defines $D'$, we then show that for small enough $\epsilon$, $B(U_1,\epsilon)$ and $B(U_2,\epsilon)$ do not intersect outside $D'$. So the intersection is contained in $D'$, so it is contained in $B(D,\delta)$.}
\includegraphics[width=0.24\textwidth]{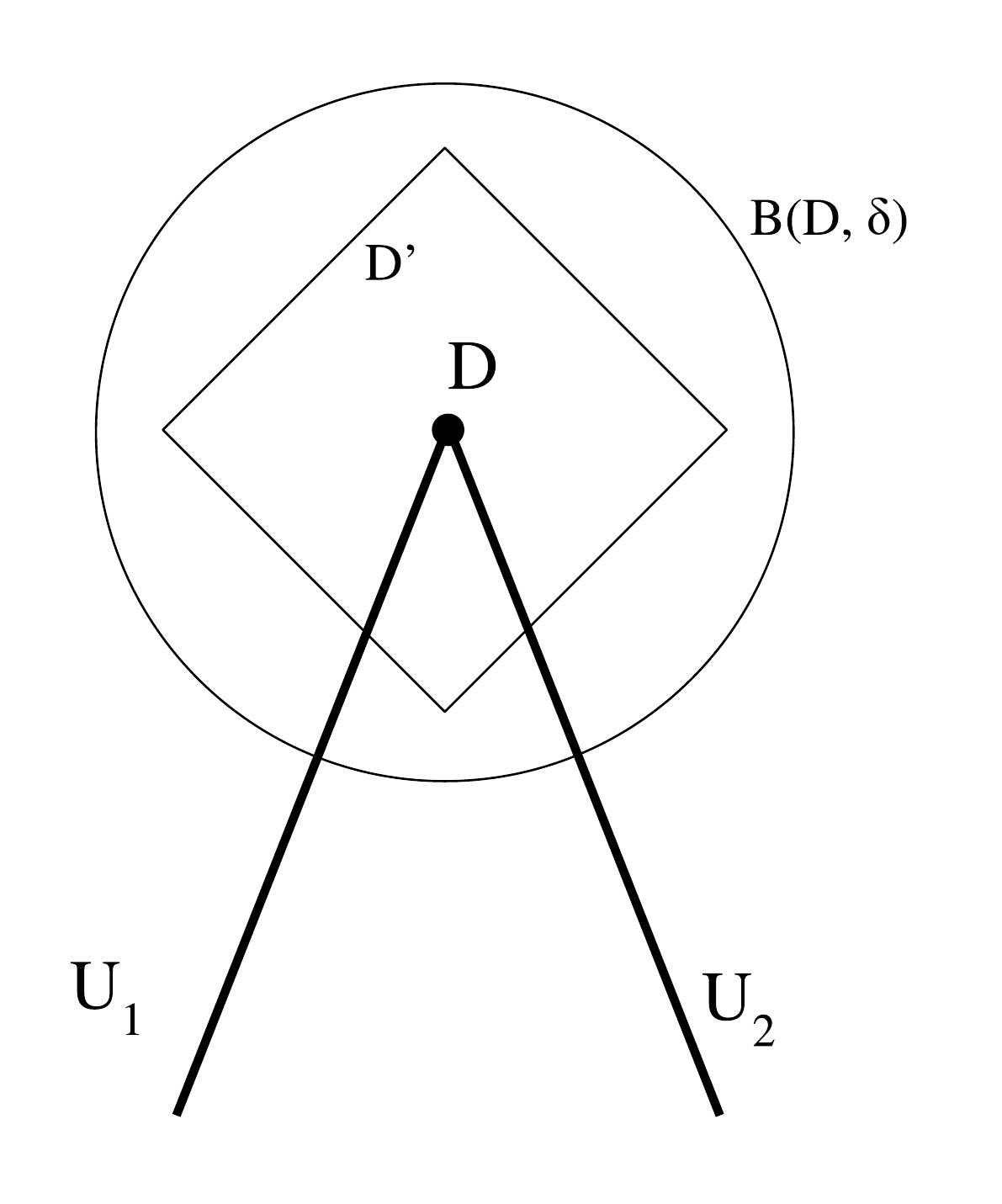} \hfill
\includegraphics[width=0.24\textwidth]{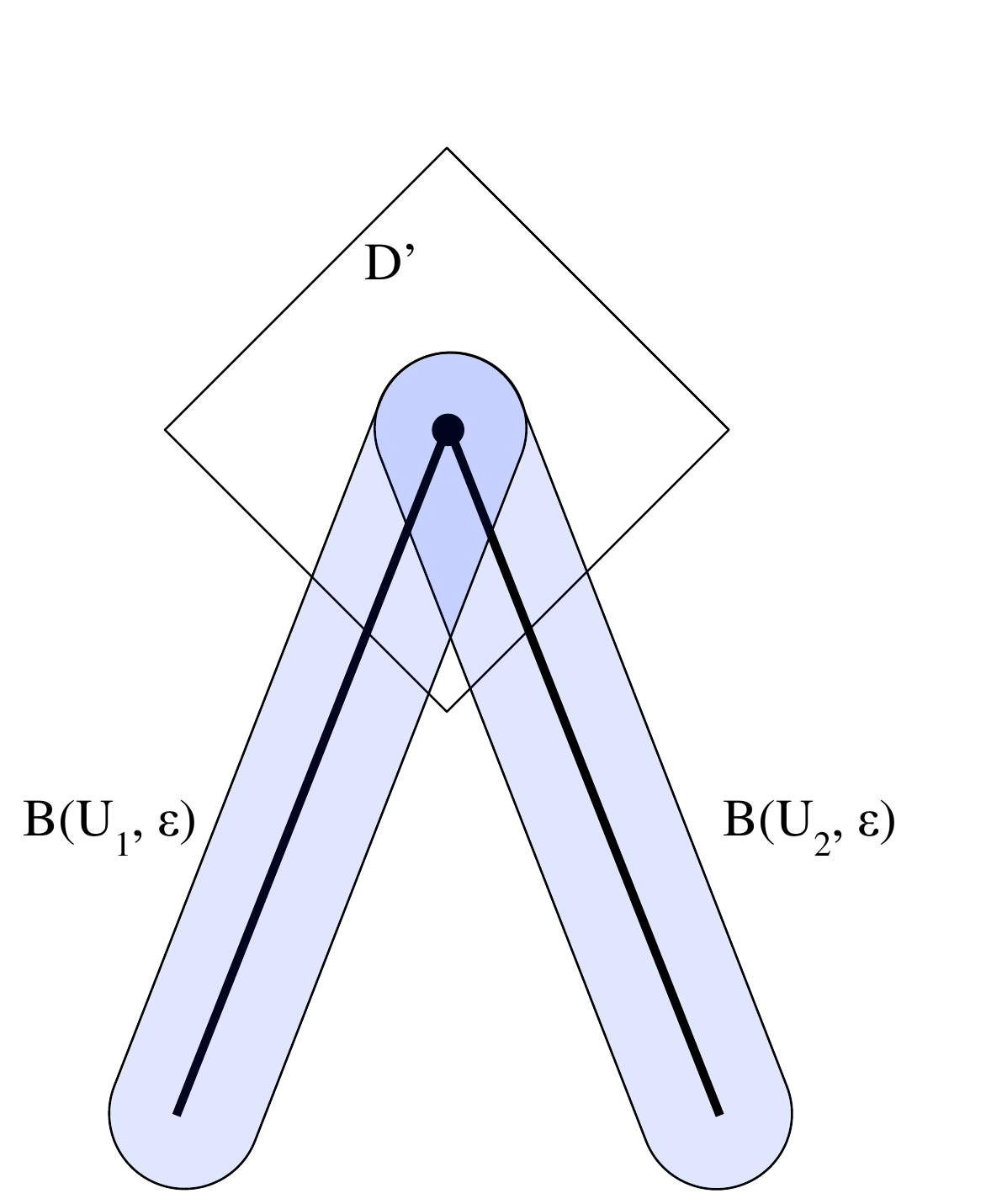} \hfill
\includegraphics[width=0.24\textwidth]{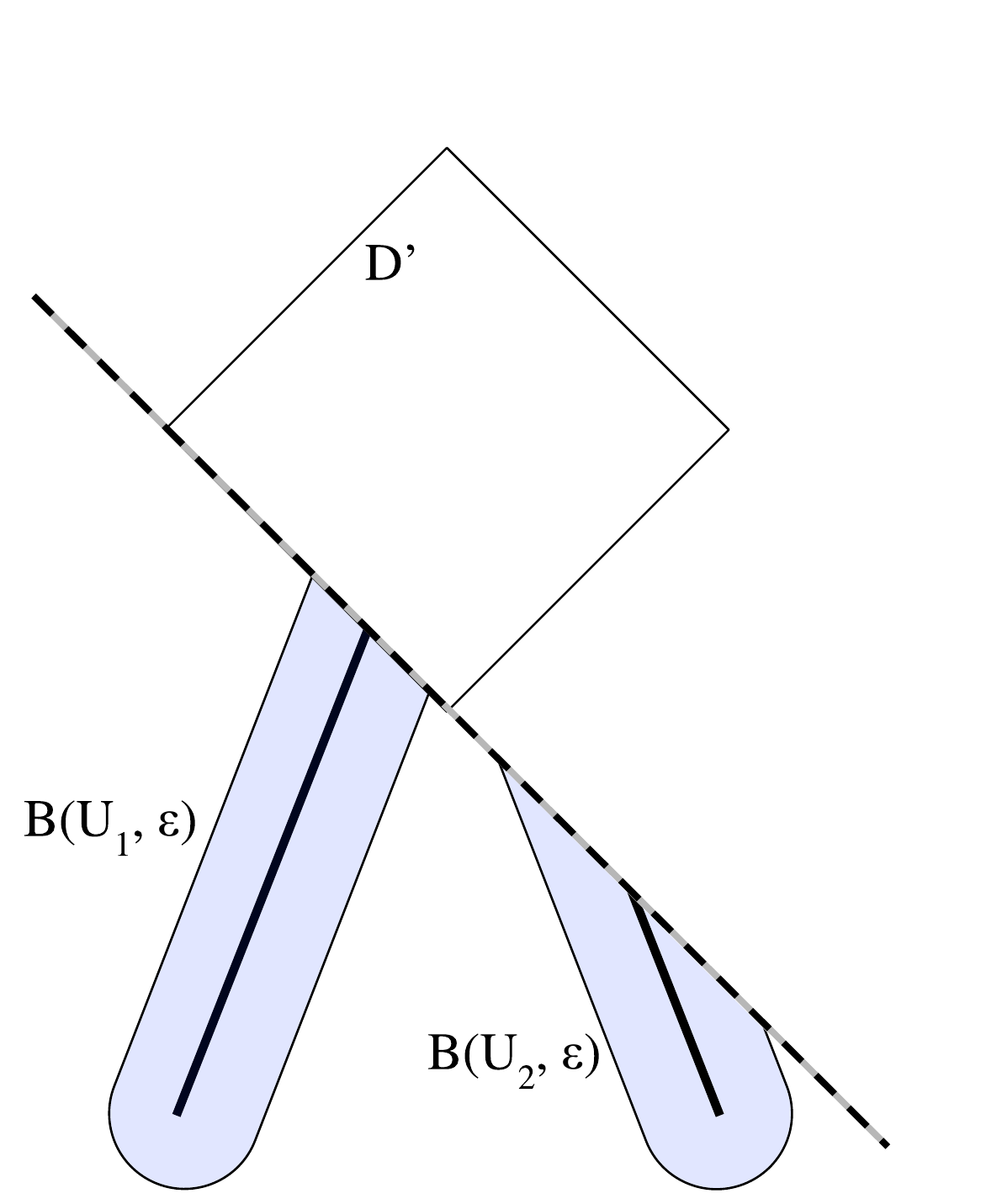} \hfill
\includegraphics[width=0.24\textwidth]{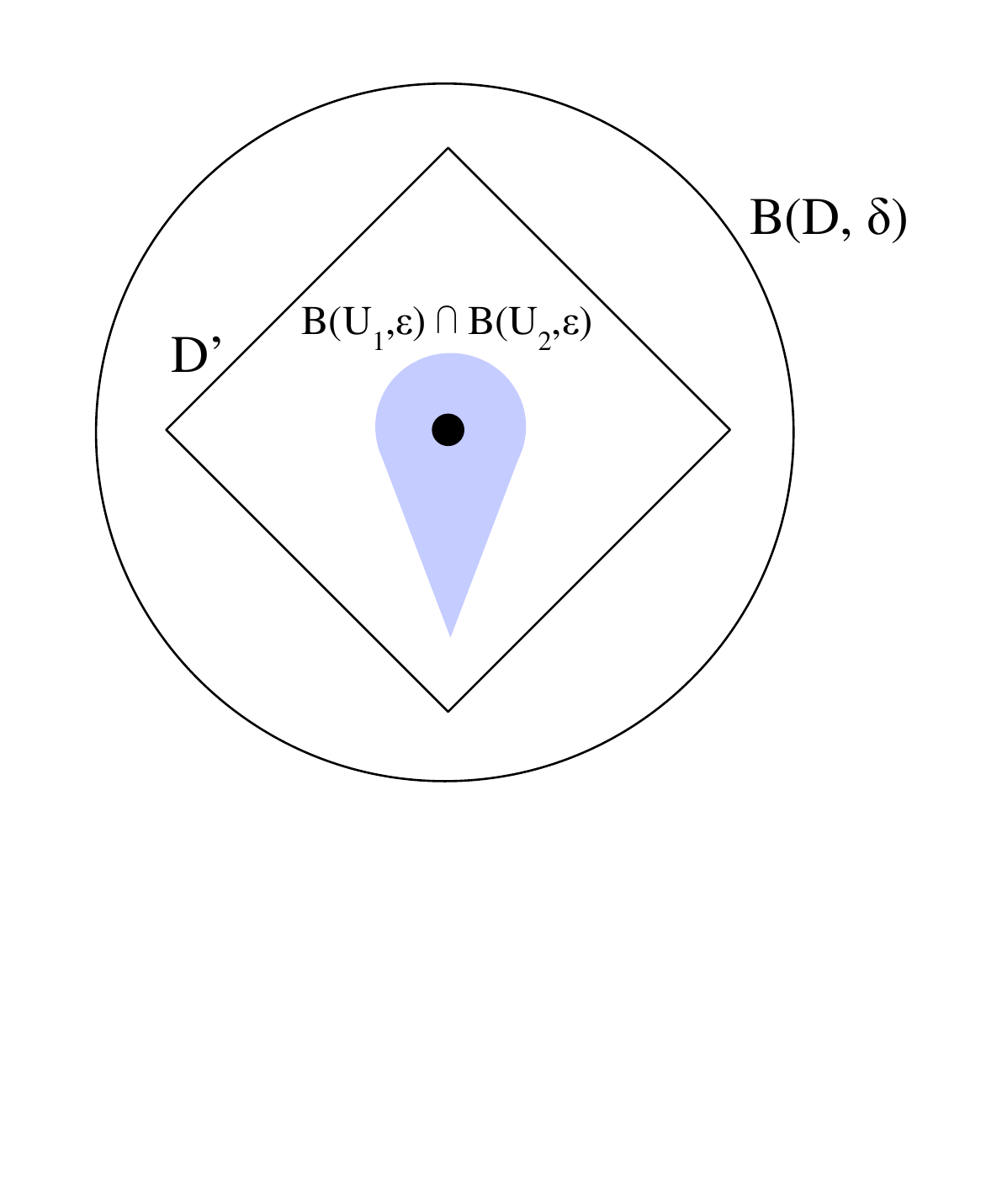}
\end{figure}

\begin{lemma} \label{lem:thick-empty}
  Let $\{U_j : j \in \mathcal{J}\}$ be a finite collection of nonempty closed, convex sets with $\cap_{j\in\mathcal{J}} U_j = \emptyset$.
  Then there exists  $\epsilon_0 > 0$ such that, for all $0 < \epsilon \leq \epsilon_0$, $\cap_{j\in\mathcal{J}} B(U_j,\epsilon) = \emptyset$.
\end{lemma}
\begin{proof}
  By induction on the size of the family.
  Note that the family must have size at least two.
  Let $U_j$ be any set in the family and let $U' = \cap_{j' \neq j} U_{j'}$.
  There are two possibilities.

  The first possibility, which includes the base case where the size of the family is two, is the case $U'$ is nonempty.
  Because $U_j$ and $U'$ are non-intersecting closed convex sets, they are separated by some distance $\delta$.
  So $B(U_j, \delta/3) \cap B(U', \delta/3) = \emptyset$.
  By Lemma \ref{lem:thick-nonempty}, there exists $\epsilon'_0 > 0$ such that $\cap_{j'\neq j} B(U_{j'},\epsilon) \subseteq B(U', \delta/3)$ for all $0 < \epsilon \leq \epsilon'_0$.
  Pick $\epsilon_0 = \min\{\epsilon'_0,\delta/3\}$.
  Then for all $0 < \epsilon \leq \epsilon_0$, the intersection of $\epsilon$-thickenings is contained in the $(\delta/3)$-thickening of the intersection, which is disjoint from the $(\delta/3)$-thickening of $U_j$, which contains the $\epsilon$-thickening of $U_j$.

  The second possibility is that $U'$ is empty.
  This implies we are not in the base case, as the family must have three or more sets.
  By inductive assumption, for all small enough $\epsilon$ we have $\cap_{j' \neq j} B(U_{j'},\epsilon) = \emptyset$, which proves this case.
\end{proof}

The proof of Lemma~\ref{lem:thick-intersect} now follows: for each $\U'\subseteq\U$, Lemma \ref{lem:thick-empty} gives an $\epsilon_0(\U') > 0$; we take the minimum of $\epsilon_0(\U')$ over the finitely many choices of $\U'$.

\subsection{Separation of the general construction}

We now prove the main results for link construction from \S~\ref{sec:calibration} and \S~\ref{sec:poly-ie-consistency}.
Specifically, we show that indirect elicitation implies that Construction~\ref{const:general-eps-thick-link} produces a link, and moreover, a link is produced if and only if it is separated.
As we have established above that Construction~\ref{const:eps-thick-link} is a special case, the results specific to embeddings will follow.

\begin{proposition}\label{prop:general-eps-thick-produce}
  Let $L:\reals^d\to\reals^\Y_+$ be polyhedral and $\gamma:\simplex\toto\R$ be finite.
  If $L$ indirectly elicits $\gamma$, then there exists $\epsilon_0 > 0$ such that, for all $0 < \epsilon \leq \epsilon_0$, Construction~\ref{const:general-eps-thick-link} for $L$, $\gamma$, $\epsilon$, $\|\cdot\|_\infty$, produces a link.
\end{proposition}
\begin{proof}
  Fix a small enough $\epsilon_0$ as promised by Lemma~\ref{lem:thick-intersect}.
  Let $u\in\reals^d$ and $\U'_u = \{U\in\U \mid d_\infty(u,U) < \epsilon\}$.
  From Construction \ref{const:general-eps-thick-link}, we have $\Psi(u) = \cap \{R_U \mid U\in\U'_u\}$.
  Since $u\in\cap\{B(U,\epsilon) \mid U\in\U'_u\} $ by definition, Lemma~\ref{lem:thick-intersect} and our choice of $\epsilon$ give $\cap\U'_u \neq \emptyset$.
  By Lemma~\ref{lem:intersection-equiv-indirect-elic}, we have $\Psi(u) = \cap \{R_U \mid U\in\U'_u\} \neq \emptyset$.
\end{proof}

Perhaps surprisingly, one can also show that \emph{every} calibrated link from a polyhedral surrogate to a discrete loss is produced by Construction \ref{const:eps-thick-link}.
This result follows from the fact, stated now, that Construction~\ref{const:eps-thick-link} with $\|\cdot\|_\infty$ is exactly enforcing $\epsilon$-separation.
From Theorem~\ref{thm:calibrated-separated}, every calibrated link $\psi$ is therefore output by the construction for some sufficiently small $\epsilon$, in the sense that $\psi$ is one of the valid choices within the link envelope.

\begin{proposition}\label{prop:link-converse}
  Let polyhedral surrogate $L:\reals^d \to \reals^\Y_+$, finite property $\gamma:\simplex\toto\R$, and $\epsilon>0$ be given.
  Then Construction~\ref{const:general-eps-thick-link} for $L,\gamma,\epsilon,\|\cdot\|_\infty$ produces a link $\psi$ if and only if $\psi$ is $\epsilon$-separated.
\end{proposition}
\begin{proof}
  Let $\U'_u = \{U\in\U \mid d_{\infty}(u,U) < \epsilon\}$.
  Then we have the following chain of equivalences.  
  \begin{align*}
    &\phantom{\iff}~\forall u\in\reals^d,\;\; \psi(u)\in\Psi(u)
    \\
    &\iff \forall u\in\reals^d,\, U\in\U'_u,\;\; \psi(u)\in R_U
    \\
    &\iff \forall u\in\reals^d,\, U\in\U'_u,\;\; \Gamma_U \subseteq \gamma_{\psi(u)}
    \\
    &\iff \forall u\in\reals^d,\, p\in\simplex \text{ s.t. } \Gamma(p)\in\U'_u,\;\; p \in \gamma_{\psi(u)}
    \\
    &\iff \forall u\in\reals^d,\, p\in\simplex \text{ s.t. } d_\infty(u,\Gamma(p)) < \epsilon,\;\; \psi(u) \in \gamma(p)
    \\
    &\iff \text{$\psi$ is $\epsilon$-separated}~.
      \qedhere
  \end{align*}
\end{proof}

From the equivalence of calibration and separation for polyhedral surrogates (Theorem~\ref{thm:calibrated-separated}), we now have Theorem~\ref{thm:link-char}: the construction produces exactly the set of calibrated links.
(The move from $\|\cdot\|_\infty$ to a general norm follows from norm equivalence in finite-dimensional vector spaces.)
Combined with Proposition~\ref{prop:general-eps-thick-produce}, we also have Proposition~\ref{prop:general-eps-thick-separated}: if $L$ indirectly elicits $\gamma$, the construction produces a calibrated link.

Returning to embeddings, recall that Construction~\ref{const:eps-thick-link} is a special case of Construction~\ref{const:general-eps-thick-link} by Lemma~\ref{lem:general-construction-special-case}.
Moreover, embeddings are a special case of indirect elicitation (Lemma~\ref{lem:embedding-implies-indirect-elic}).
As Proposition~\ref{prop:general-eps-thick-produce} guarantees that Construction~\ref{const:general-eps-thick-link} produces a link, and Proposition~\ref{prop:link-converse} that every link produced is separated, we now have
Theorem \ref{thm:thickened-separated} which we restate.
(Note that the converse need not hold for Construction~\ref{const:eps-thick-link}, and indeed, that construction may miss some separated links which make use of reports outside $\Sc$.)
\thickenedseparated*

\section{Proving Lemma~\ref{lem:X}}
\label{app:polyhedra}

In this section, we give a careful treatment of the results on convex polyhedra needed to prove Lemma~\ref{lem:X}.

\subsection{General definitions for polyhedra}
\label{app:polyhedra:defs}

We begin with general definitions of polyhedra.
See also \citet{ziegler1993lectures} and \citet{gallier2008notes}.

\begin{definition}[Closed halfspace]
  A closed halfspace is a set of the form $H_{(w,b)}^+ := \{ x \in \reals^d \mid \inprod{x}{w} \geq b\}$ for some $(w,b) \in \reals^d \times \reals$.
\end{definition}
\begin{definition}[Hyperplane]
  A hyperplane is a set of the form $H_{(w,b)} := \{ x \in \reals^d \mid \inprod{x}{w} = b\}$ for some $(w,b)\in\reals^d \times \reals$.
\end{definition}

Observe that $H_{(w,b)} = \partial H^+_{(w,b)}$, meaning the hyperplane $H_{(w,b)}$ is the boundary of $H^+_{(w,b)}$.
Thus, for any halfspace $H^+$, we have that $H^+$ is one of the two closed halfspaces corresponding to the hyperplane $\partial H^+ = H$.

\begin{definition}[Polyhedron halfspace representation~\citep{ziegler1993lectures}]
	A \emph{polyhedron} $P$ is an intersection of a finite set of closed halfspaces $\H$
  presented in the form $P = \cap \H$.
  Here, we say $\H$ is a halfspace representation for $P$.
\end{definition}

\noindent
Observe that by the halfspace representation, a polyhedron need not be bounded.

\begin{definition}[Supports]
	A hyperplane $H$ \emph{supports} the polyhedron $P$ if 
	(i) $P \subseteq H^+$ for a halfspace $H^+$ with $H = \partial H^+$, and
	(ii) $H \cap \partial P \neq \emptyset$.
	Moreover, $H$ supports $P$ at $x$ if $x \in H \cap \partial P$.
\end{definition}

\begin{definition}[Face, facet]\label{def:face}
  Let $P \subseteq \reals^d$ be a convex polyhedron.
  A \emph{face} $F$ of the polytope $P$ is any set of the form
  \begin{equation*}
    F = P \cap H~,
  \end{equation*}
  for a hyperplane $H$ supporting $P$.
  The dimension of a face $F$ is the dimension of its affine hull $\dim(F) := \dim(\affhull(F))$.
  A face $F$ with $\dim(F) = \dim(\affhull(P)) - 1$ is called a facet.
\end{definition}

While one traditionally considers $P$ to be a trivial face of itself, we exclude this case throughout.

It is often useful to understand polyhedra in terms of their halfspace representations and the set of hyperplanes generating facets of $P$.
To find this set, we must first establish when a halfspace representation is irredundant for a given polyhedron, as this irredundant set corresponds to the facets of a polyhedron in a natural way.

\begin{definition}[Irredundant; adapted from {\citet{gallier2008notes}}]\label{def:irredundant}
	Let $P = \cap \H$ for a finite set of closed halfspaces $\H$ be a polyhedron.
	We say that $\cap \H$ is an \emph{irredundant decomposition} for $P$ (and $\H$ is irredundant for $P$) if $P$ cannot be expressed as $P = \cap \H'$ for some set of closed halfspaces $\H'$ such that $|\H'| < |\H|$, and redundant otherwise.
	Moreover, we call $\H$ irredundant for $P$ if $\cap \H$ is an irredundant decomposition of $P$.
\end{definition}

\citet{gallier2008notes} shows that every $d$-dimensional polyhedron $P \subseteq \reals^d$ has a unique and irredundant halfspace representation $\H^*$, and each $H^+ \in \H^*$ generates a facet of $P$.

\begin{theorem}[{\citet{gallier2008notes}}]\label{thm:polyhedron-uniquely-gen-facets}
  Given a $d$-dimensional polyhedron $P \subseteq \reals^d$, 
  (i) there is a unique irredundant and finite set of closed halfspaces $\H^*$ such that $P = \cap \H^*$, 
  (ii) $\{H \cap P \mid H^+ \in \H^*, H = \partial H^+\}$ is the set of facets of $P$, and
  (iii) for all finite sets of closed halfspaces $\H$ such that $P = \cap \H$, we have $\H^* \subseteq \H$.
\end{theorem}
\begin{proof}
	Since $P$ is $d$-dimensional in $\reals^d$, it therefore has nonempty interior.
	As $P$ has a finite halfspace representation, it must have a smallest halfspace representation $\H^*$.
	That is, $|\H^*| = \min \{|\H| : P = \cap \H, \H$ finite$\}$.
	As a smallest halfspace representation, $\H^*$ is irredundant by definition.
	\citet[Proposition 4.5(i)]{gallier2008notes} then states that $\H^*$ is unique, giving (i).
	Additionally, (ii) is shown by~\citep[Proposition 4.5(ii)]{gallier2008notes}.

	It remains to show (iii).
	Let $\H$ be a finite set of closed halfspaces such that $P = \cap \H$.
  As noted in the last sentence of the proof of \citep[Proposition 4.5]{gallier2008notes},
  the hyperplanes defining facets are unique:
  if $F$ is a facet of $P$ and $H,H'$ are hyperplanes with $F = H \cap P = H' \cap P$, then it must be the case that $H = H'$.
	It therefore suffices to show that, for each facet $F$ of $P$, there is an $H^+ \in \H$ such that $F = H\cap P$.
  \citet[Proposition 3.17]{gallier2008notes} observes that, for all $x \in \partial P$, there exists some hyperplane $\hat H$ such that $\hat H$ supports $P$ at $x$.
	Since $x \in \relint(F)$ is in exactly one face of $P$, namely $F$, there must be a unique $H^+ \in \H$ such that $F = H \cap P$.
\end{proof}

\subsection{Specializing to $\reals^\Y \times \reals$}\label{appsubsec:notation}
Within this appendix, we use some self-contained notation to work in the function graph space $\reals^\Y \times \reals$.
We will later consider losses over a finite set of outcomes $\Y$; to make notation consistent, we use $\reals^\Y_+$ throughout as shorthand for $\reals^{|\Y|}_+$, and let $d := |\Y|+1$.

Given any $v \in \reals^\Y_+$, define $H^+_v := H^+_{(v, -1)} = \{(x, c) \in \reals^{\Y}_+ \times \reals \mid \inprod{v}{x} \geq c\}$.
Similarly, we denote $H_y^+ := H_{(e_y, 0)}^+$ for any $y \in \Y$; the latter will help us restrict a constructed polyhedron to the nonnegative orthant. 
Extending to hyperplanes, we construct $H_v := H_{(v,-1)}$ and observe that $H_v = \partial H^+_v$ for $v \in \reals^\Y_+$ and define $H_y := H_{(e_y, 0)}$ so that $H_y = \partial H^+_y$.
Given a set $\V \subseteq \reals^d$, we let $\H_{\V} = \{H_v^+ \mid v\in\V\}$ denote the set of halfspaces generated by $\V$.
Similarly, let $\H_\Y = \{H_y^+ \mid y\in\Y\}$.

For any $S \subseteq \reals^k$, let $\delta(\cdot \mid S):\reals^k \to \reals \cup \{\infty\}$ be the convex indicator function, given by $\delta(x \mid S) = 0$ if $x\in S$ and $\infty$ otherwise.
Throughout, we will work with a concave function $g_\V$ generated by a set $\V \subseteq \reals^\Y_+$ of the following form.

\begin{definition}\label{def:g-finite}
  Given a set $\V \subseteq \reals^\Y_+$, define the function $g_\V : \reals^\Y \to \reals_+ \cup \{-\infty\}$ by
  \begin{align*}
    g_\V(x) = \inf_{v\in\V} \inprod{x}{v} - \delta(x \mid \reals_+^\Y)~. 
  \end{align*}
\end{definition}
\noindent
We denote the hypograph of a function $g:\reals^\Y\to\reals\cup\{-\infty\}$ by $\hyp(g) = \{(x,c) \mid c \leq g(x)\} \subseteq \reals^\Y \times \reals$.

A first observation is that the region generated by the intersection of the $H^+_y$ halfspaces restricts the hypograph $g_\V$ to be finite only on the nonnegative orthant for any $\V \subseteq \reals^\Y_+$.
\begin{lemma}\label{lem:x-nonneg-orthant-iff-intersection-HY}
  $\cap \H_\Y = \reals^\Y_+ \times \reals$.
\end{lemma}
\begin{proof}
  The result follows if we show $x \in \reals^\Y_+ \iff (x,c) \in \cap \H_\Y$ for all $c \in \reals$.

  $\implies$
  Fix any $c \in \reals$.
  $x \in \reals^\Y_+ \iff x_y \geq 0$ for all $y \in \Y$.
  This means that for any $y \in \Y$, $(x,c) \in \{(x,c) \mid x_y \geq 0\} = H^+_y$.
  As $y$ and $c$ were arbitrary, this shows the forward direction.
  
  $\impliedby$
  $(x,c) \in \cap \H_\Y$ implies $x_y \geq 0$ for all $y \in \Y$, and therefore $x \in \reals^d_+$.	
\end{proof}

\subsection{Hypographs of extended Bayes risks}\label{appsubsec:phase2}
\noindent
We now apply this polyhedral perspective to the Bayes risk of a loss function, extended to ``unnormalized distributions'', i.e., all of $\reals^\Y_+$.
Given a minimizable loss function $L : \R \to \reals^\Y_+$, define the $1$-homogeneous extension of its Bayes risk as
$\risk{L}_+:\reals^\Y\to\reals\cup\{\infty\}$, $x \mapsto \inf_{r\in\R} \inprod{x}{L(r)} - \delta(x \mid \reals^\Y_+)$.
In other words, letting $L(\R) := \{L(r) \mid r\in\R\} \subseteq \reals^\Y_+$, we have $\risk L_+ = g_{L(\R)}$.
Observe that $L(\R)$ and $\H_{L(\R)}$ may be infinite sets.  
\begin{claim}\label{claim:gV-equals-riskL}
	Suppose we are given a minimizable $L : \R \to \reals^\Y_+$ with polyhedral extended risk $\risk L_+$.
	Then $\hyp(g_{L(\R)}) = \cap(\H_\Y \cup \H_{L(\R)})$.
\end{claim}
\begin{proof}
	Observe that $x \in \reals^\Y_+ \iff (x, c) \in \cap \H_\Y$.
	Let $x \in \reals^\Y_+$.
	\begin{align*}
	(x,c) \in \hyp(g_{L(\R)})
	&\iff g_{L(\R)}(x) \geq c & \text{definition of hypograph}\\
	&\iff \risk L_+(x) \geq c & \text{$g_{L(\R)}=\risk L_+$}\\
    &\iff \inprod{v}{x} \geq c \,\, \forall v \in {L(\R)} & \text{definition of $\risk L_+$} \\
	&\iff (x,c) \in H^+_{v} \,\, \forall v \in {L(\R)} & \text{definition of $H^+_v$}\\
	&\iff (x,c) \in \cap \H_{L(\R)} & 
	\end{align*}
	Combining the two equalities, we have $\hyp(g_{L(\R)}) = \cap (\H_\Y \cup \H_{L(\R)})$.

\end{proof}

\subsection{Finding the unique smallest subset of loss vectors}

\begin{lemma}\label{lem:G-unique-facets-Hstar}
  Consider a loss $L : \reals^d \to \reals^\Y_+$ with polyhedral extended risk $\risk L_+$.
  There is a unique irredundant set $\H^*$ of closed halfspaces such that $\hyp(g_{L(\R)}) = \cap \H^*$. 
  Moreover, for each $H^+ \in\H^*$ and $H$ such that $H = \partial H^+$, the face $\hyp(g_{L(\R)}) \cap H$ is a facet. 
  Moreover, $\H_\Y \subseteq \H^*$. 
\end{lemma}
\begin{proof}
	As $g_{L(\R)}$ is nonnegative on $\reals_+^\Y$, the set $\hyp(g_{L(\R)})$ therefore contains $\{(x,c) \mid x\in\reals^\Y_+, c \leq 0\}$, which is $(|\Y| + 1)$-dimensional.
  Therefore $\hyp(g_{L(\R)})$ is full-dimensional.
	Take $\H^*$ to be the unique irredundant and finite set of closed halfspaces such that $\hyp(g_{L(\R)}) = \cap \H^*$ from Theorem~\ref{thm:polyhedron-uniquely-gen-facets}(i).
  Now, $\hyp(g_{L(\R)}) \cap H$ for any $H \in \H^*$ being a facet follows immediately from Theorem~\ref{thm:polyhedron-uniquely-gen-facets}(ii) and (iii).
  
  To show that $\H_\Y \subseteq \H^*$, it suffices from Theorem~\ref{thm:polyhedron-uniquely-gen-facets}(ii) to show that each $F_y := H_y \cap \hyp(g_{L(\R)})$ for $y\in\Y$ is a facet.
  From Claim~\ref{claim:gV-equals-riskL}, since $H_y^+ \in \H_\Y$, we have $\hyp(g_{L(\R)}) \subseteq H_y^+$.
  We have $F_y = H_{y} \cap \hyp(g_{L(\R)}) = \{(x,c) \mid x\in\reals^\Y_+, x_y=0, c \leq g_{L(\R)}(x)\} \supseteq \{(x,c) \mid x\in\reals^\Y_+, x_y=0, c \leq 0\}$ as $g_{L(\R)}(x) \geq 0$.
  Thus, $F_y$ is a nonempty face of $\hyp(g_{L(\R)})$, and contains a $|\Y|$-dimensional set, so must be a facet.
\end{proof}

We now use the above result about the unique irredundant halfspace decomposition of $\hyp(g_\V)$ to observe a unique finite set of loss vectors generating these halfspaces.

\begin{corollary}\label{cor:unique-set-loss-vectors-defining-facets}
  Given a minimizable loss $L: \reals^d \to \reals^\Y_+$ with polyhedral extended risk, consider the unique irredundant set $\H^*$ given by Lemma~\ref{lem:G-unique-facets-Hstar}. 
  Then
  (i) there is a unique finite set $\V^* \subseteq \reals^\Y_+$ such that $\H^* = \H_\Y \cup \H_{\V^*}$.
  (ii) $F_v := H_v \cap \hyp(g_{L(\R)})$ is a facet of $\hyp(g_{L(\R)})$ for each $v\in\V^*$.
  (iii) $g_{L(\R)}(x) = \min_{v \in \V^*}\inprod{v}{x} - \delta(x \mid \reals_+^\Y) = g_{\V^*}(x)$.
\end{corollary}
\begin{proof}
  (i) Since $\hyp(g_{L(\R)})$ is full-dimensional, the facets of $\hyp(g_{L(\R)})$ are uniquely determined by the hyperplanes $H$ such that $H = \partial H^+$ and $H^+ \in \H^*$ by Lemma~\ref{lem:G-unique-facets-Hstar}.
  Any facet must then be some intersection of an $H_y \cap \hyp(g_{L(\R)})$ or $H_v \cap \hyp(g_{L(\R)})$.
  Since $\H_\Y \subseteq \H^*$ by Lemma~\ref{lem:G-unique-facets-Hstar}, take $\H_{\V^*} := \H^* \setminus \H_{\Y}$, and $\V^*$ the unique set generating $\H_{\V^*}$.
  (ii) Moreover, $H_v \in \H_{\V^*} \subseteq \H^*$ generates the facet $F_v = H_v \cap \hyp(g_{L(\R)})$ of $\hyp(g_{L(\R)})$ by Lemma~\ref{lem:G-unique-facets-Hstar}.
  (iii)   The result holds if $\hyp(g_{L(\R)}) = \hyp(g_{\V^*})$.
  By construction, $\hyp(g_{L(\R)}) = \cap \H^* = \cap (\H_\Y \cup \H_{\V^*}) = \{(x,c) \in \reals_+^\Y \times \reals \mid \inprod{v^*}{x} \geq c$ for all $v^* \in \V^* \} = \hyp(g_{\V^*})$.
  The result follows from equality of the hypographs.
\end{proof}

The above establishes our primary assumption for the rest of this section.

\begin{assumption}\label{assum:V-star-exists-infinite}
	$L : \R \to \reals^\Y_+$ is a minimizable loss function such that $\risk L_+ = g_{L(\R)}$ is polyhedral.
	$\V^* \subseteq L(\R)$ is the unique finite set such that $g_{L(\R)} = g_{\V^*}$ and $\hyp(g_{L(\R)}) = \hyp(g_{\V^*}) = \cap(\H_{\V^*} \cup \H_\Y)$, the last of which is irredundant.
\end{assumption}

\begin{proposition}\label{prop:poly-risk-assumption-satisfied}
	Given a minimizable loss $L : \R \to \reals^\Y_+$ such that $\risk L_+$ is polyhedral, there exists a finite $\V^*$ satisfying Assumption~\ref{assum:V-star-exists-infinite}.
\end{proposition}
\begin{proof}
	Consider the unique $\V^* \subseteq L(\R)$ that is given in Corollary~\ref{cor:unique-set-loss-vectors-defining-facets} such that $g_{L(\R)} = g_{\V^*}$.
	Finally, by construction in Lemma~\ref{lem:G-unique-facets-Hstar}, we additionally have $\H^* = \H_\Y \cup \H_{\V^*}$, and $\cap \H^*$ is irredundant.
\end{proof}

\subsection{Extended Bayes risk to extended level sets: projecting from $\reals^d_+$ to $\reals^\Y_+$}\label{subsec:project-pi}

When the loss $L$ is clear, we denote the faces of $\hyp(g_{L(\R)})$ by $F_v := H_v \cap \hyp(g_{L(\R)})$.
We now define the projection $\pi:\reals^\Y\times \reals \to \reals^\Y, (x,c) \mapsto x$.

In this subsection, we will establish results about the dimensionality of extended level sets, and conditions under which subsets of these extended level sets cover the nonnegative orthant $\reals^\Y_+$.
As a first step, the extended level sets generated by $\V^*$ cover the nonnegative orthant.

\begin{corollary}\label{cor:projected-Vstar-faces-cover-pos-orthant}
	Consider $L, \V^*$ satisfying Assumption~\ref{assum:V-star-exists-infinite}.
	Then $\cup_{v\in\V^*} \pi(F_v) = \reals^\Y_+$.
\end{corollary}

Moreover, the projection $\pi$ preserves dimension of faces.

\begin{claim}\label{claim:pi-preserves-dim}
	Consider $L, \V^*$ satisfying Assumption~\ref{assum:V-star-exists-infinite}. 
	Then $\dim(F_v) = \dim(\pi(F_v))$.
\end{claim}
\begin{proof}
	Recall from Definition~\ref{def:face} that the dimension of a polytope to be the dimension of its affine hull.
	Suppose we are given $|\Y|+1$ affinely independent vectors $z_i$ in $F_v$. 
	We will show that their projections $\{\pi(z_i)\}_i$ are affinely independent, giving the result.

	Let $a_1 + \ldots + a_{|\Y|+1} = 0$, such that $\sum_i a_i \pi(z_i) = 0$.	
	As $z_i \in F_v \subseteq H_v$, we have some $\{x_i\}_i$ such that $z_i = (x_i, \inprod{v}{x_i})$ for all $i$.
  By assumption, then, $0 = \sum_i a_i \pi(z_i) = \sum_i a_i x_i$.
	Thus, $\sum_i a_i z_i = \sum_i a_i (x_i, \inprod{v}{x_i}) = (\sum_i a_i x_i, \inprod{v}{\sum_i a_i x_i}) = (0,0) = 0 \in \reals^{|\Y|+1}$.
  By affine independence of the $z_i$, we conclude $a_i = 0$ for all $i$.
  Thus, the set $\{\pi(z_i)\}_i$ is affinely independent.
\end{proof}

Since we preserve the dimension of these projected spaces, we can now study equivalence of projected faces of the hypograph and regions of support of $g_{\V^*}$ for any $v \in L(\R)$.

\begin{lemma}\label{lem:projected-faces-iff-support-iff-argmin}
Given $L, \V^*$ satisfying Assumption~\ref{assum:V-star-exists-infinite}, fix $x \in \reals^\Y_+$.
	For any $v \in L(\R)$, the following are equivalent:

	(1) $(x,g_{\V^*}(x)) \in F_v$; 
	
	(2) $\inprod{v}{x}= g_{\V^*}(x)$;
	
	(3) $v \in \argmin_{v' \in L(\R)} \inprod{v'}{x}$; and
	
	(4) $x \in \pi(F_v)$~.
\end{lemma}
\begin{proof}
	\begin{align*}
	(1) \quad \quad (x,g_{\V^*}(x)) \in F_v
	&\iff (x,g_{\V^*}(x)) \in \{(x',c) \in \hyp(g_{\V^*}) \mid \inprod{v}{x'} = c\} & \\
	&\iff  \inprod{v}{x} = g_{\V^*}(x) & (2)\\
	&\iff \inprod{v}{x} = \min_{v' \in L(\R)}\inprod{v'}{x} & \\
	&\iff v \in \argmin_{v' \in L(\R)}\inprod{v'}{x}~. & (3)
	\end{align*}
	This covers $1 \iff 2 \iff 3$.
	
	For $1 \iff 4$, the forward implication follows trivially by applying the definition of the projection $\pi$.
	For the reverse implication, consider some $x \in \pi(F_v)$.
	There must be a $c \in \reals$ so that $(x,c) \in F_v$.
	Expanding, this is actually saying $(x,c) \in \{(x',c') \in \hyp(g_{\V^*}) \mid \inprod{v}{x'} = c\}$.
	In particular, this is true when $c = \inprod{v}{x}$, which defines a face of $\hyp(g_{\V^*})$ at $x$ if any only if $\inprod{v}{x} = g_{\V^*}(x)$.
	Therefore, we have $(x, g_{\V^*}(x)) \in F_v$.
\end{proof}

\begin{claim}\label{claim:projected-Vstar-faces-full-dim}
	Consider $L, \V^*$ satisfying Assumption~\ref{assum:V-star-exists-infinite}.
	For all $v \in \V^*$, $\pi(F_v)$ is full dimensional in $\reals_+^\Y$.
\end{claim}
\begin{proof}
	By Corollary~\ref{cor:unique-set-loss-vectors-defining-facets}, $F_v$ is a facet of $\hyp(g_{L(\R)})$ in $\reals^d_+$, meaning it is $(d - 1)$-dimensional.
	Moreover, Claim~\ref{claim:pi-preserves-dim} states that the dimension of $F_v$ is preserved for each $v \in \V^*$.
	Thus, $d-1 = |\Y| = \dim(F_v) = \dim(\pi(F_v))$.
\end{proof}

Now we can observe a set of normals $\V'$ generates faces of $\hyp(g_{\V^*})$ whose projections cover $\reals^\Y_+$ if and only if the set contains $\V^*$.
This fact will translate to a set of reports being representative for a loss if and only if it contains a finite minimum representative set.

\begin{claim}\label{claim:Vprime-projected-faces-cover-iff-Vstar-subset-Vprime}
	Consider $L, \V^*$ satisfying Assumption~\ref{assum:V-star-exists-infinite}.
	For $\V' \subseteq L(\R)$, we have
	$\cup_{v\in\V'} \pi(F_v) = \reals^\Y_+ \iff \V^* \subseteq \V'$.
\end{claim}
\begin{proof}

	($\implies$) 
	For contraposition, suppose $\V^* \not \subseteq \V'$.
	Then $\exists v \in \V^* \setminus \V'$.
	Observe that $\V^*$ is unique, $\H^*$ is irredundant by assumption, and $\pi(F_v)$ is full-dimensional in $\reals^\Y_+$ by Claim~\ref{claim:projected-Vstar-faces-full-dim}.
	Moreover, $\pi(F_v) \not \in \cup_{v' \in \V'} \pi(F_v)$, which implies $\cup_{v' \in \V'} \pi(F_v) \neq \cup_{v^* \in \V^*} \pi(F_{v^*}) = \reals^\Y_+$.

	($\impliedby$)
	Since $\V^* \subseteq \V'$, we immediately have $\cup_{v \in \V^*} \pi(F_v) \subseteq \cup_{v' \in \V'} \pi(F_{v'})$.
	Moreover, $\cup_{v \in \V^*} \pi(F_v) = \reals^\Y_+$ by Corollary~\ref{cor:projected-Vstar-faces-cover-pos-orthant}, so $\reals^\Y_+ \subseteq \cup_{v' \in \V'} \pi(F_{v'})$.
	As $g_{\V^*}$ is only finite on $\reals^\Y_+$ by construction, equality follows.
	
\end{proof}

We now claim that a set of projected faces $\{F_v\}_{v \in \V'}$ for some $\V'$ will cover $\reals^\Y_+$ if and only if $\V^* \subseteq \V'$.
Given a finite set $\V^* \subset \reals^\Y_+$, we denote $\Lambda_{S} := \{\pi(F_v) \mid v \in S\}$ as the set of projected facets generated by $S$.

\begin{claim}\label{claim:Vprime-projected-faces-cover-iffprojected-faces-subsets}
	Consider $L, \V^*$ satisfying Assumption~\ref{assum:V-star-exists-infinite}, and $\R' \subseteq \R$ with $\V' := L(\R')$. 
	We have $\cup_{v \in \V'} \pi(F_v) = \reals_+^\Y \iff \Lambda_{\V^*}  \subseteq \Lambda_{\V'}$.
\end{claim}
\begin{proof}
	$\implies$
	The result follows if $\V^* \subseteq \V'$, which follows from the forward implication of Claim~\ref{claim:Vprime-projected-faces-cover-iff-Vstar-subset-Vprime}.
	Explicitly, for all $v \in \V^*$ we also have $v \in \V'$, so, $\pi(F_v) \in \Lambda_{\V^*} \cap \Lambda_{\V'} = \Lambda_{\V^*}$.
	
	$\impliedby$
	If $\Lambda_{\V^*} \subseteq \Lambda_{\V'}$, then $\cup_{v^* \in \V^*} \pi(F_{v^*}) \subseteq \cup_{v \in \V'} \pi(F_v)$.
	By Corollary~\ref{cor:projected-Vstar-faces-cover-pos-orthant}, we have $\cup_{v^* \in \V^*} \pi(F_{v^*}) = \reals_+^\Y$, so $\reals_+^\Y \subseteq \cup_{v \in \V'} \pi(F_v)$.
	The other direction of subset inequality following from $\hyp(g_{\V^*})$ being finite only on $\reals_+^\Y$.	
\end{proof}

Now, we can conclude that projected facets generated by $L(\R)$ contain all other projected faces of $\hyp(g_{\V^*})$.
\begin{corollary}\label{cor:exists-vstar-projected-face-subset}
	Consider $L, \V^*$ satisfying Assumption~\ref{assum:V-star-exists-infinite}.
	For any $v \in L(\R)$, there is a $v^* \in \V^*$ such that $\pi(F_v) \subseteq \pi(F_{v^*})$.
\end{corollary}
\begin{proof}
	First, observe that $F_v \subseteq F_{v^*}$ as $F_{v^*}$ is a facet by construction of $\V^*$ and Theorem~\ref{thm:polyhedron-uniquely-gen-facets} (ii).
	As each face of a polyhedron is contained in a facet, we have $F_v \subseteq F_{v^*}$.
	Moreover, $\pi(F_v) \subseteq \pi(F_{v^*})$ follows immediately as a corollary.
\end{proof}

\subsection{Translating to properties: projecting from $\reals^\Y_+$ to $\simplex$}\label{subsec:project-f}

\begin{definition}\label{def:f-polyhedral}
	For any polyhedral concave function $f_\V$ with domain on $\simplex$, we denote the function $f_\V(p) = \min_{v\in\V} \inprod{p}{v} - \delta(p \mid \simplex)$, where $\V \subset \reals^\Y_+$ is a finite set. 
\end{definition}

\noindent
We now make a few observations about the extensions of $1$-homogeneous polyhedral functions.

\begin{claim}\label{claim:extended-poly-risk-poly}
	For any minimizable function $L$ such that $\risk L$ is polyhedral, its extension $\risk L_+$ is also polyhedral.
\end{claim}
\begin{proof}
  We will think of $\risk L$ as defined $\risk L:\reals^\Y\to\reals_+\cup\{-\infty\}$  with $\dom(\risk L) = \simplex$.
For $p \in \simplex$, we know $\sum_i p_i = 1$, and can write any inner product $\inprod{p}{b} - \beta = \inprod{p}{b} - \inprod{p}{\beta \ones} = \inprod{p}{b - \beta \ones}$.
If $p \not \in \simplex$, then $\risk L(p) = -\infty$.
Moreover, since $\risk L$ is polyhedral, it is finitely generated \citep[Proposition 19.1.2]{rockafellar1997convex} and can be written
\newcommand{\B}{\mathcal{B}} 
\begin{align*}
\risk L(p) 
&= \min(\inprod{p}{b_1} - \beta_1, \ldots, \inprod{p}{b_k} - \beta_k) - \delta(p \mid \simplex)\\
&= \min(\inprod{p}{b_1 - \beta_1 \ones}, \ldots, \inprod{p}{b_k - \beta_k \ones}) - \delta(p \mid \simplex)\\
&= \min_{(b, \beta) \in \B} \inprod{p}{b - \ones \beta} - \delta(p \mid \simplex)~,
\end{align*}
where $\B = \{(b_i, \beta_i)\}_{i = 1}^k$ from the second line.
We claim that $\risk L_+(x) = \min_{(b, \beta) \in \B} \inprod{x}{b - \ones \beta} - \delta(x \mid \reals^\Y_+)$, as this form is clearly 1-homogenous and agrees with $\risk L$ on $\simplex$.
As $\B$ is a finite set, $\risk L_+$ is polyhedral.
\end{proof}

\begin{claim}\label{claim:gV-and-fV}
	Consider $L, \V^*$ satisfying Assumption~\ref{assum:V-star-exists-infinite}.
	Then $\risk L$ is polyhedral (on the simplex) and $f_{\V^*} = \risk L$.
	Moreover, $f_{\V^*}(p) = g_{\V^*}(p)$ for all $p \in \simplex$.
\end{claim}

Now, we define the function $\theta(v) = \{p \in \simplex \mid \inprod{v}{p} = f_\V(p)\}$ as the level sets of the loss vector $v \in \V$. 
\begin{claim}\label{claim:theta-is-pi-cap-simplex}
  Consider $L, \V^*$ as in Assumption~\ref{assum:V-star-exists-infinite}.
  For all $v \in \V^*$, consider the face $F_v$ of $\hyp(g_{L(\R)})$.
  Then $\theta(v) = \pi(F_v) \cap \simplex$.
\end{claim}
\begin{proof}
  Fix $p \in \simplex$.
  \begin{align*}
    p \in \theta(v)
    &\iff \inprod{v}{p} = f_{\V^*}(p) & \text{Definition of $\theta$}\\
    &\iff \inprod{v}{p} = \min_{v' \in {\V^*}}\inprod{v'}{p} & \text{$f_{\V^*}=g_{\V^*}$ on $\simplex$ (Claim~\ref{claim:gV-and-fV})}\\
    &\iff v \in \argmin_{v' \in {\V^*}}\inprod{v'}{p} &\\
    &\iff p \in \pi(F_v)~. & \text{Lemma~\ref{lem:projected-faces-iff-support-iff-argmin}}
  \end{align*}
\end{proof}

\subsection{Moving back from $\simplex$ to $\reals^\Y_+$}
Now that we have translated from $\reals^d_+$ to $\reals^\Y_+$ in \S~\ref{subsec:project-pi} and from $\reals^\Y_+$ to $\simplex$ in \S~\ref{subsec:project-f}, we take some final steps to prove Lemma~\ref{lem:X} by showing equivalences from $\simplex$ to $\reals^\Y_+$.
\begin{lemma}\label{lem:level-set-is-projected-face}
  Consider $L, \V^*$ satisfying Assumption~\ref{assum:V-star-exists-infinite}.
  For all $r \in \R$ with $v = L(r)$, $\Gamma_r = \theta(v) = \pi(F_{v}) \cap \simplex$.
\end{lemma}
\begin{proof}
  Let us rewrite
  \begin{align*}
    \Gamma_r
    &= \{p \in \simplex \mid r \in \argmin_{r' \in \R} \inprod{L(r')}{p}\}\\
    &= \{p \in \simplex \mid v \in \argmin_{v' \in L(\R)} \inprod{v'}{p}\}\\
    &= \{p \in \simplex \mid \inprod{v}{p} = \min_{v' \in L(\R)}\inprod{v'}{p}\}\\
    &= \{p \in \simplex \mid \inprod{v}{p} = f_{\V^*}(p) \}\\
    &= \theta(v)~.
  \end{align*}
  The rest of the result follows from Claim~\ref{claim:theta-is-pi-cap-simplex}.
   
\end{proof}

\begin{lemma}\label{lem:g-1-homog}
	Consider $L, \V^*$ satisfying Assumption~\ref{assum:V-star-exists-infinite}.
	Then $g_{\V^*}(x) = \min_{v \in \V^*}\inprod{v}{x}$ is (positively) $1$-homogeneous.
\end{lemma}
\begin{proof}
	If $x \not \in \reals^\Y_+$, then $c g(x) = -\infty = g(cx)$ for any $c > 0$.
	If $x \in \reals^\Y_+$, then we have $g(cx) = \min_{v \in \V^*}\inprod{v}{cx} = c \min_{v \in \V^*}\inprod{v}{x} = c g(x)$ for any $c > 0$ by linearity of the inner product.
\end{proof}

Every minimizable loss $L$ elicits a unique property $\Gamma := \prop{L}$, and we can define the \emph{extended level set} $\bar \Gamma_r := \{x \in \reals^\Y_+ \mid \inprod{L(r)}{x} = \risk L_+(x)\}$.

\begin{lemma}\label{lem:levelset-to-extended-levelset}
	Consider $L: \R \to \reals^\Y_+, \V^*$ satisfying Assumption~\ref{assum:V-star-exists-infinite}.
	For any $r \in \R$ and $c > 0$, if $p \in \Gamma_r$, then $cp \in \bar \Gamma_r$. 
\end{lemma}
\begin{proof}
	Fix $r \in \R$ and $c > 0$.
	We have
	\begin{align*}
	p \in \Gamma_r
	&= \{p' \in \simplex \mid r \in \argmin_{r' \in \R} \inprod{L(r')}{p'} \} & \text {Definition of level set} \\
	&= \{p' \in \simplex \mid v \in \argmin_{v' \in L(\R)} \inprod{v'}{p'} \} &  \\
	&= \{p' \in \simplex \mid \inprod{v}{p'} = \min_{v' \in L(\R)} \inprod{v'}{p'} \} & \text {$L$ minimizable} \\
	&= \{p' \in \simplex \mid \inprod{v}{p'} = g_{\V^*}(p') \} & \text {Assumption~\ref{assum:V-star-exists-infinite}} \\
	&= \{p' \in \simplex \mid c \inprod{v}{p'} = c g_{\V^*}(p') \} &  \\
	&= \{p' \in \simplex \mid  \inprod{v}{cp'} = g_{\V^*}(cp') \} & \text {Lemma~\ref{lem:g-1-homog}} \\
	\implies cp
	&\in \{x \in \reals^\Y_+ \mid \inprod{v}{x} = g_{\V^*}(x)\} = \bar \Gamma_r~.~
	\end{align*}
\end{proof}

\begin{lemma}\label{lem:extended-levelset-equals-projected-face}
	Consider $L : \R \to \reals^\Y_+, \V^*$ satisfying Assumption~\ref{assum:V-star-exists-infinite}.
	For any $r\in \R$ with $v = L(r)$, $\bar \Gamma_r = \pi(F_v)$.  
\end{lemma}
\begin{proof}
	\begin{align*}
	\bar \Gamma_r
	&= \{x \in \reals^\Y_+ \mid \inprod{L(r)}{x} = \risk L_+(x)\} & \text{Definition of $\bar \Gamma_r$}\\
	&= \{x \in \reals^\Y_+ \mid \inprod{L(r)}{x} = g_{L(\R)}(x)\} & \text{Assumption~\ref{assum:V-star-exists-infinite}}\\
	&= \{x \in \reals^\Y_+ \mid \inprod{v}{x} = g_{L(\R)}(x)\} & \text{$v = L(r)$}\\
	&= \pi(F_v) & \text{Since $F_v = \{(x,g_{L(\R)}(x)) \mid \inprod{v}{x} = g_{L(\R)}(x)\}$}\\
	\end{align*}
\end{proof}

\begin{claim}\label{claim:projected-faces-cover-RY-iff-representative}
	Consider $L : \R \to \reals^\Y_+, \V^*$ satisfying Assumption~\ref{assum:V-star-exists-infinite}.
	For any $v \in L(\R)$, denote the face $F_v := H_v \cap \hyp(g_{\V^*})$.
	A set $\R' \subseteq \R$ with $\V' := L(\R')$ is representative for $L$ if and only if $\cup_{v \in \V'} \pi(F_v) = \reals^\Y_+$.  
\end{claim}
\begin{proof}
	$\implies$
	This proof follows from three lemmas: first, we observe that $g_{\V^*}$ is $1$-homogeneous via Lemma~\ref{lem:g-1-homog}.
	Then we extend the notion of a level set $\Gamma_r$ to the nonnegative orthant $\bar \Gamma_r$, and show that any scalar transformation of a distribution in the level set is contained in the same (extended) level set via Lemma~\ref{lem:levelset-to-extended-levelset}.
	Finally, we show the extended level set is exactly the projection $\pi(F_v)$ in Lemma~\ref{lem:extended-levelset-equals-projected-face}.
	As a corollary, we chain the results to observe $\cup_{r \in \R'} \Gamma_r = \simplex \implies \cup_{r \in \R'} \bar \Gamma_r = \reals_+^\Y \implies \cup_{v \in L(\R')}\pi(F_v) = \reals^\Y_+$, yielding the forward implication.

	$\impliedby$
	Fix $p \in \simplex \subseteq \reals^\Y_+$.
	By the assumption, there is a $v \in \V'$ such that $p \in \pi(F_v)$.
	By Lemma~\ref{lem:level-set-is-projected-face}, we have $p \in \pi(F_v) \cap \simplex = \Gamma_r$ for the $r \in \R'$ such that $v = L(r)$.
	As this is true for all $p \in \simplex$, we have $\R'$ representative.

\end{proof}

\subsection{Proving Lemma~\ref{lem:X}}\label{subsec:lem-X-proof}
We now proceed with a few final lemmas that ultimately yield the proof of Lemma~\ref{lem:X}.

\begin{lemma}\label{lem:lemX1-rep-iff-subset-vectors}
	Consider $L : \R \to \reals^\Y_+, \V^*$ satisfying Assumption~\ref{assum:V-star-exists-infinite}.
	A finite set $\R' \subseteq \R$ with $\V' = L(\R')$ is representative if and only if $\V^* \subseteq \V'$.
\end{lemma}
\begin{proof}
	Chain Claim~\ref{claim:projected-faces-cover-RY-iff-representative} and Claim~\ref{claim:Vprime-projected-faces-cover-iff-Vstar-subset-Vprime} to yield the result.
\end{proof}

Define $\Theta_{S} := \{\theta(v) \mid v \in S\}$; it follows that $\Theta_{\V^*}$ is exactly the set of level sets of the property elicited by $L$.

\begin{lemma}\label{lem:lemX-3-rep-iff-FDLS-subsets}
	Consider $L : \R \to \reals^\Y_+, \V^*$ satisfying Assumption~\ref{assum:V-star-exists-infinite}.
	A finite set $\R' \subseteq \R$ with $\V' = L(\R')$ is representative if and only if $\Theta_{\V^*} \subseteq \Theta_{\V'}$.
\end{lemma}
\begin{proof}
	Chain Claim~\ref{claim:projected-faces-cover-RY-iff-representative} and Claim~\ref{claim:Vprime-projected-faces-cover-iffprojected-faces-subsets} to yield the result.
\end{proof}

With $L, \V^*$ satisfying Assumption~\ref{assum:V-star-exists-infinite}, we additionally let $\R^*$ be a set such that $\V^* := L(\R^*)$.
Such a set exists as $\V^* \subseteq L(\R)$ in Assumption~\ref{assum:V-star-exists-infinite}, though it is not necessarily unique.

\begin{corollary}
	Consider $L, \V^*$ satisfying Assumption~\ref{assum:V-star-exists-infinite}, and $\R^*$ such that $\V^* = L(\R^*)$. 
	Moreover, suppose $L$ elicits $\Gamma$.
	$\Theta_{\V^*} = \{\Gamma_r \mid r \in \R^*\}$.
\end{corollary}

\begin{lemma}\label{lem:fdls-exactly-theta}
	Consider $L : \R \to \reals^\Y_+, \V^*$ satisfying Assumption~\ref{assum:V-star-exists-infinite}.
	Moreover, let $\Gamma := \prop{L}$.
	$\Theta_{\V^*} = \{\Gamma_r \mid r\in\R, \dim(\Gamma_r) = |\Y|-1\}$.
\end{lemma}
\begin{proof}
	From Claim~\ref{claim:projected-Vstar-faces-full-dim}, we know $\Lambda_{\V^*}$ is exactly the set of full-dimensional level sets in $\reals^\Y_+$.
	Each element of $\Lambda_{\V^*}$ is $\pi(F_v)$ for some $v \in \V^*$.
	Take $r \in \R^*$ so that $v = L(r)$.
	By Lemma~\ref{lem:level-set-is-projected-face}, we have $\theta(v) = \Gamma_r = \pi(F_v) \cap \simplex$ is full-dimensional relative to the simplex.
	The result follows.
\end{proof}

\begin{lemma}\label{lem:any-levelset-contained-in-minlevelset}
  Consider $L : \R \to \reals^\Y_+, \V^*$ satisfying Assumption~\ref{assum:V-star-exists-infinite}, and $\R^* \subseteq \R$ the set such that $\V^* = L(\R^*)$.
  Moreover, consider $\Gamma := \prop{L}$.
  For any $r \in \R$, there exists a $v^* \in \V^*$ such that $\Gamma_r \subseteq \theta(v^*)$.
\end{lemma}
\begin{proof}
  Take $v = L(r)$.
  By Corollary~\ref{cor:exists-vstar-projected-face-subset}, there is a $v^* \in \V^* \subseteq L(\R)$ such that $\pi(F_v) \subseteq \pi(F_{v^*})$.
  Therefore, $\pi(F_v) \cap \simplex \subseteq \pi(F_{v^*})\cap \simplex$.  
  We know $\theta(v) = \pi(F_v) \cap \simplex$ and similarly for $\theta(v^*)$ by Lemma~\ref{lem:level-set-is-projected-face}.
  The result follows.
\end{proof}

We are now ready to apply the above to prove Lemma~\ref{lem:X}.
In particular, any loss $L$ satisfying the assumptions of Lemma~\ref{lem:X} has some $g_{L(\R)} = \risk{L}_+$ as in this section.

\lemmaX*
\begin{proof}
Since $\risk L$ is polyhedral, so is $\risk L_+$ by Claim~\ref{claim:extended-poly-risk-poly}.
Therefore, we have satisfied the requirements of Proposition~\ref{prop:poly-risk-assumption-satisfied}, and can conclude there is a finite set $\V^*\subseteq L(\R)$ such that $\risk L_+ = g_{\V^*}$ satisfying Assumption~\ref{assum:V-star-exists-infinite}.
Hence, there is a finite set $\R^*$ such that $\V^* = L(\R^*)$. 
For all $x\in\reals^\Y_+$, there exists $v \in \V^*$ such that $\inprod{v}{x} = g_{\V^*}(x)$, and thus some $r\in\R^*$ such that $\inprod{L(r)}{p} = g_{\V^*}(x)$.
Thus, for all $p \in \simplex$, there exists $r \in \R^*$ such that $\risk L(p) = \risk L_+(p) = g_{\V^*}(p) = \inprod{L(r)}{p}$ and therefore $r \in \prop{L}(p)$.
As this is true for all $p \in \simplex$, $\R^*$ is representative for $L$.

Now consider the itemized statements.
Lemma~\ref{lem:lemX1-rep-iff-subset-vectors} is exactly statement \eqref{item:X-rep-V}.
This immediately implies statement~\eqref{item:X-min-V}.
Moreover, Lemma~\ref{lem:lemX-3-rep-iff-FDLS-subsets} is exactly statement~\eqref{item:X-rep-Theta}, and again statement~\eqref{item:X-min-Theta} immediately follows.
Statement~\eqref{item:X-rep-contain-min} is a corollary of the existence of a finite representative set, which follows since $\V^* \subseteq L(\R)$.
Statement~\eqref{item:X-full-dim} is exactly Lemma~\ref{lem:fdls-exactly-theta}.
Statement~\eqref{item:X-redundant} is exactly Lemma~\ref{lem:any-levelset-contained-in-minlevelset}.
Finally, Statement~\eqref{item:X-tight-embed} follows as a corollary of statement~\eqref{item:X-min-V} and Corollary~\ref{cor:tight-embed-min-rep}.
\end{proof}

\end{document}